\documentclass{article}

\usepackage{microtype}
\usepackage{graphicx}
\usepackage{subcaption}
\usepackage{booktabs} % for professional tables
\usepackage{multirow}
\usepackage{tikz}
\usetikzlibrary{positioning,arrows.meta,calc,fit,backgrounds}
\usepackage{hyperref}

\usepackage[preprint]{icml2026}

\usepackage{amsmath}
\usepackage{amssymb}
\usepackage{mathtools}
\usepackage{amsthm}
\usepackage[capitalize,noabbrev]{cleveref}

\theoremstyle{plain}
\newtheorem{theorem}{Theorem}[section]
\newtheorem{proposition}[theorem]{Proposition}
\newtheorem{lemma}[theorem]{Lemma}

\theoremstyle{definition}
\newtheorem{definition}[theorem]{Definition}
\newtheorem{assumption}[theorem]{Assumption}
\theoremstyle{remark}

\usepackage[textsize=tiny]{todonotes}

\newcommand{\E}{\mathbb E}

\newcommand{\best}[1]{\textbf{#1}}
\newcommand{\second}[1]{\underline{#1}}

\usepackage{adjustbox}

\icmltitlerunning{GEPC: Group-Equivariant Posterior Consistency}

\begin{document}

\twocolumn[
  \icmltitle{GEPC: Group-Equivariant Posterior Consistency \\
    for Out-of-Distribution Detection in Diffusion Models}



  \icmlsetsymbol{equal}{*}

\begin{icmlauthorlist}
  \icmlauthor{Yadang Alexis Rouzoumka}{onera,sondra}
  \icmlauthor{Jean Pinsolle}{sondra}
  \icmlauthor{Eugénie Terreaux}{onera}
  \icmlauthor{Christèle Morisseau}{onera}
  \icmlauthor{Jean-Philippe Ovarlez}{onera,sondra}
  \icmlauthor{Chengfang Ren}{sondra}
\end{icmlauthorlist}

\icmlaffiliation{onera}{DEMR, ONERA, Université Paris-Saclay, 91120 Palaiseau, France}
\icmlaffiliation{sondra}{SONDRA, CentraleSupélec, Université Paris-Saclay, 91190 Gif-sur-Yvette, France}

% Put the real corresponding author email here:
\icmlcorrespondingauthor{Yadang Alexis Rouzoumka}{yadang-alexis.rouzoumka@centralesupelec.fr}

  \icmlkeywords{Machine Learning, ICML}

  \vskip 0.3in
]



\printAffiliationsAndNotice{}  % no special notice (required even if empty)

% \begin{abstract}
%   This document provides a basic paper template and submission guidelines.
%   Abstracts must be a single paragraph, ideally between 4--6 sentences long.
%   Gross violations will trigger corrections at the camera-ready phase.
% \end{abstract}

%=============================================================
% ABSTRACT
%=============================================================
\begin{abstract}
Diffusion models learn a time-indexed score field $\mathbf{s}_\theta(\mathbf{x}_t,t)$ that often inherits approximate equivariances  (flips, rotations, circular shifts)  from in-distribution (ID) data and convolutional backbones.
Most diffusion-based out-of-distribution (OOD) detectors exploit score \emph{magnitude} or \emph{local geometry} (energies, curvature, covariance spectra) and largely ignore equivariances.
We introduce Group-Equivariant Posterior Consistency (GEPC), a training-free probe that measures how consistently the learned score transforms under a finite group $\mathcal{G}$, detecting \emph{equivariance breaking} even when score magnitude remains unchanged.
At the population level, we propose the ideal GEPC residual, which averages an equivariance-residual functional over $\mathcal{G}$, and we derive ID upper bounds and OOD lower bounds under mild assumptions.
GEPC requires only score evaluations and produces interpretable equivariance-breaking maps.
On OOD image benchmark datasets, we show that GEPC achieves competitive or improved AUROC compared to recent diffusion-based baselines while remaining computationally lightweight. On high-resolution synthetic aperture radar imagery where OOD corresponds to targets or anomalies in clutter, GEPC  yields strong target-background separation and visually interpretable equivariance-breaking maps.
Code is available at \url{https://github.com/RouzAY/gepc-diffusion/}.
\end{abstract}

%%%%%%%%%%%%%%%%%%%%%%%%%%%%%%%%%%%%%%%%%%%%%%%%%%%%%%%%%%%%
\section{Introduction}
%%%%%%%%%%%%%%%%%%%%%%%%%%%%%%%%%%%%%%%%%%%%%%%%%%%%%%%%%%%%

Detecting out-of-distribution (OOD) inputs is a fundamental challenge for deploying reliable machine learning models.
Classic post-hoc scores for classifiers rely on confidence or energy, such as maximum softmax probability (MSP), ODIN, and energy-based scores~\cite{hendrycks2017baseline, liang2018odin, liu2020energy}, while subsequent work exploits representation geometry (e.g., $k$NN- or PCA-style feature models)~\cite{sun2022knnOod, guan2023pca}.

Diffusion models~\cite{ho2020ddpm, song2021sde, karras2022edm, yang2023diffusionSurvey} have recently emerged as strong priors for OOD and anomaly detection.
Beyond raw likelihoods, they expose a time-indexed score field and a generative trajectory, motivating diffusion OOD scores that often rely on either
(i) trajectory/energy criteria along the reverse process or probability-flow ODE~\cite{graham2023ddpmood, diffpath, shin2023spr},
or (ii) local score-field geometry such as curvature or covariance-spectrum diagnostics~\cite{scoped2025, eigenscore2025}.
These approaches primarily exploit score magnitude or local differential structure, and may require additional reverse steps or Jacobian-related computations.

In parallel, explicitly equivariant score-based and diffusion models have advanced rapidly, especially for 3D and molecular data.
E(3)-equivariant diffusion models~\cite{pmlr-v162-hoogeboom22a,NEURIPS2024_587b3f36,DBLP:journals/bmcbi/ZhangLLWG24} combine invariant noise processes with equivariant networks to guarantee that learned distributions inherit known symmetries.
Recent analyses~\cite{chen2024equivariantscorebasedgenerativemodels, tahmasebi2024sample} relate score matching to a symmetrized score term plus a deviation-from-equivariance penalty, while group-convolutional / steerable CNNs~\cite{cohen2016gcnn, cohen2017steerable} and studies of approximate shift equivariance in vanilla CNNs~\cite{zhang2019shiftinvariant, bruintjes2023equivariancefactors} show that augmentation and anti-aliasing yield only approximate equivariance in practice.

These works primarily treat equivariance as an inductive bias for training.
We take the complementary viewpoint: we do not enforce equivariance at training time; we measure its (in)consistency as a test-time statistic for OOD detection.

\paragraph{Our perspective: equivariance breaking as an OOD signal.}
We hypothesize that when the in-distribution (ID) is approximately invariant under a group $\mathcal{G}$ (e.g., flips, rotations, circular shifts) and the backbone is convolutional and trained with augmentations, the learned diffusion scores should be \emph{approximately $\mathcal{G}$-equivariant} on ID samples, but this \emph{posterior consistency} should break for OOD inputs that violate the learned symmetries or lie far from the ID manifold.
Concretely, group-transforming a noisy input $\mathbf{x}_t$ and transporting the predicted score back should preserve the score on ID; systematic violations indicate distribution shift.
Importantly, this is not a pixel-space invariance test: we probe equivariance of the \emph{learned score field} at noisy levels, hence the model's posterior geometry rather than raw image symmetries.

We operationalise this via \textbf{GEPC} (Group-Equivariant Posterior Consistency), a training-free probe of pretrained diffusion models.
For a group $\mathcal{G}$ and an operator $\mathcal{P}_g\in \mathcal{G}$ and selected timesteps, we compare $\mathcal{P}_g^\top \mathbf{s}_\theta(\mathcal{P}_g \mathbf{x}_t,t)$ and $\mathbf{s}_\theta(\mathbf{x}_t,t)$, aggregate residuals over $\mathcal{G}$ and $t$, and calibrate the resulting statistic using only ID data.
GEPC produces both a scalar OOD score and spatial heatmaps highlighting equivariance failures.

Figure~\ref{fig:gepc_overview} summarizes GEPC: we noise the input, probe score-field equivariance via group transports, aggregate residuals across timesteps, and calibrate using ID-only statistics to obtain an OOD score and equivariance-breaking maps.
\paragraph{Relation to equivariance-based conformal OOD detectors (iDECODe).}
iDECODe~\cite{kaur2022idecode} turns equivariance violations under random group actions into a conformal non-conformity score, enabling distribution-free calibrated decisions.
GEPC is complementary: rather than wrapping equivariance errors in a conformal layer, we probe pretrained diffusion score fields across timesteps and analyze the corresponding \emph{population} equivariance-breaking functional, yielding ID upper and OOD lower bounds under mild score-error assumptions.

\paragraph{Relation to diffusion OOD geometry.}
GEPC complements the dominant diffusion OOD families above.
Trajectory/energy and curvature/covariance-spectrum methods probe the \emph{local} geometry of $\mathbf{s}_\theta$ along time, and some require Jacobian-related computations.
GEPC instead targets \emph{global group consistency}: we measure how consistently the score transforms under $\mathcal{G}$ and turn deviations from equivariance into an OOD statistic, without computing any Jacobian or modifying the backbone.

At the population level, we give a equivariance-breaking characterization of the ideal GEPC residual under $\mathcal{G}$, closely related to deviation-from-equivariance analyses in equivariant score matching~\cite{chen2024equivariantscorebasedgenerativemodels}.
Under mild assumptions, we derive ID upper bounds and OOD lower bounds for the expected GEPC residual, clarifying when posterior consistency should hold or break.

% % Orthogonally, feature-space OOD methods refine matrix-induced distances or exploit neural-collapse structure in classifier representations, but they operate purely in feature space; GEPC instead probes how a generative score field responds to group actions.
% Unlike feature-space OOD scores, GEPC probes how a \emph{generative score field} responds to group actions.

\paragraph{Contributions.}
(1) We introduce \textbf{GEPC}, a training-free OOD score that tests \emph{group-consistency} of diffusion score fields across timestep and group actions.
  GEPC requires only inference access to a pretrained DDPM-style backbone (including improved diffusion), with no architectural changes, fine-tuning, or Jacobian evaluation.
\newline
(2) We provide a practical recipe combining group pooling, stability-based timestep selection, ID-only calibration (KDE or vector Mahalanobis), and stochastic subsampling of timestep and group elements.
  We characterise the computational cost and show that GEPC operates in a similar number-of-function-evaluations (NFE) regime as simple score-norm baselines while approaching the performance of more expensive trajectory and curvature-based methods.
  % \item[(3)] We connect the ideal GEPC residual to a \emph{group score-matching (Fisher) divergence} on $p_t$, derive ID upper and OOD lower bounds under symmetry and score-error assumptions, and discuss cross-backbone regimes where the diffusion model is trained on a different source distribution.
\newline
(3) We provide a population-level analysis of GEPC: we relate the ideal residual to an equivariance-breaking functional under $\mathcal{G}$, derive ID upper bounds and OOD lower bounds under mild score-error assumptions, and discuss cross-backbone regimes where the diffusion model is trained on a different source distribution.
\newline
(4) We empirically show that GEPC is competitive with and complementary to curvature, spectrum, and trajectory-based diffusion OOD scores on CIFAR-scale near/far OOD benchmarks under a shared CelebA backbone, and that in a cross-domain high-resolution setting where a $256\times256$ LSUN-trained backbone is applied to radar SAR imagery, GEPC yields strong detection performance and interpretable equivariance-breaking maps.

\begin{figure*}[t]
\centering
\resizebox{\textwidth}{!}{%
\begin{tikzpicture}[
    font=\large,
    >=Stealth,
    every node/.style={transform shape},
    box/.style={
        draw, rounded corners=2.5mm, line width=1pt, align=center,
        inner sep=6pt, minimum height=10mm, fill=#1!10, text width=5.0cm
    },
    smallbox/.style={
        draw, rounded corners=2.0mm, line width=0.9pt, align=center,
        inner sep=4pt, minimum height=8mm, fill=#1!6, text width=6.8cm
    },
    flow/.style={-Stealth, line width=1.2pt, line cap=round},
    optflow/.style={-Stealth, line width=1.0pt, line cap=round, dashed},
    lab/.style={font=\normalsize\bfseries}
]

% ---- libs needed: positioning,arrows.meta,calc,fit,backgrounds ----

% Colors
\definecolor{cIn}{RGB}{52,58,64}
\definecolor{cNoi}{RGB}{79,98,214}
\definecolor{cGrp}{RGB}{111,66,193}
\definecolor{cRes}{RGB}{214,51,132}
\definecolor{cCal}{RGB}{25,135,84}
\definecolor{cOOD}{RGB}{220,53,69}

% ==========================================================
% ROW 1 (top)
% ==========================================================
\node[box=cIn] (x0)
  {$\mathbf{x}_0$\\[0.6mm]\textcolor{cNoi!90!black}{ID or OOD input}};

\node[box=cNoi, right=18mm of x0] (noise)
  {Forward noising\\[0.6mm]
   $\mathbf{x}_t \sim q(\mathbf{x}_t\mid \mathbf{x}_0)$\\[0.2mm]
   \textcolor{cNoi!90!black}{$t\in\mathcal{T}$}};

%\node[box=cNoi, right=18mm of noise] (xt)
%  {$\mathbf{x}_t$};

\node[box=cGrp, right=18mm of noise] (branch)
  {Group transport\\[0.6mm]
   \textbf{(a)} $\mathcal{P}_g\mathbf{x}_t$,\ \textbf{(b)} $\mathbf{x}_t$\\[0.2mm]
   \textcolor{cGrp!90!black}{$g\sim \mathrm{Unif}(\mathcal{G})$}};

\node[box=cRes, right=18mm of branch] (resid)
  {Score + transport back\\[0.6mm]
   $\tilde{\mathbf{s}}_\theta=\mathcal{P}_g^{-1}\mathbf{s}_\theta(\mathcal{P}_g\mathbf{x}_t,t)$\\[0.3mm]
   $\mathbf r_t=\tilde{\mathbf{s}}_\theta - \mathbf{s}_\theta(\mathbf{x}_t,t)$};

\draw[flow] (x0) -- (noise);
%\draw[flow] (noise) -- (xt);
\draw[flow] (noise) -- (branch);
\draw[flow] (branch) -- (resid);

\node[lab, above=1mm of branch, text=cGrp!90!black]
  {Probe equivariance on $\mathbf{s}_\theta(\cdot,t)$};

% ==========================================================
% ROW 2 (bottom)  <-- aligned under row 1
% ==========================================================
\node[box=cRes, below=13mm of resid] (pool)
  {Pool \& normalise\\[0.8mm]
   $R_t=\|\mathbf r_t\|_2^2$,\quad
   $z_t=\mathbb{E}_g\!\left[\dfrac{R_t}{b_t(\mathbf{x}_0)}\right]$};

\node[box=cRes, left=18mm of pool] (agg)
  {Aggregate across time\\[0.8mm]
   $\mathrm{GEPC}_s(\mathbf{x}_0)=\sum_{t\in\mathcal{T}} w_t\,z_t$\\[0.2mm]
   \textcolor{cRes!90!black}{(keep-$K$ timesteps, weights $w_t$)}};

\node[lab, above=1mm of agg, text=cRes!90!black]
  {Patch level statistics};

\node[box=cCal, left=18mm of agg] (calib)
  {ID-only calibration\\[0.8mm]
   KDE / z-score / MVN\\[0.2mm]
   \textcolor{cCal!90!black}{fit on ID-train}};

\node[box=cCal, left=18mm of calib] (out)
  {Output\\[0.8mm]
   OOD score + map};

   \node[lab, above=1mm of out, text=cCal!90!black]
  {Decision / Thresolding};

% connect row1 -> row2 neatly (down then left, like your example)
\draw[flow] (resid.south) -- (pool.north);
\draw[flow] (pool) -- (agg);
\draw[flow] (agg) -- (calib);
\draw[flow] (calib) -- (out);

% ==========================================================
% Optional ribbon (ID-train)
% ==========================================================
\node[smallbox=cCal, above=9mm of noise] (tsel)
  {ID-train: select $\mathcal{T}$ and $w_t$ \ (\textcolor{cCal!90!black}{stability/CV})};
\draw[optflow] (tsel) -- (noise.north);

% ==========================================================
% Density inset (below calib) -- now very visible
% ==========================================================
\node[smallbox=cIn, below=13mm of calib, minimum height=26mm] (dens) {};
\node[lab, above=1mm of dens, text=cIn!90!black] {ID vs OOD score density (normalized)};

\coordinate (A) at ([xshift=6mm,yshift=6mm]dens.south west);
\coordinate (B) at ([xshift=6mm,yshift=20mm]dens.south west);
\coordinate (C) at ([xshift=62mm,yshift=6mm]dens.south west);
\draw[line width=0.9pt, draw=cIn!80] (A) -- (B);
\draw[line width=0.9pt, draw=cIn!80] (A) -- (C);
\node[font=\scriptsize, rotate=90, text=cIn!80] at ([xshift=2mm,yshift=13mm]dens.south west) {density};
\node[font=\scriptsize, text=cIn!80] at ([xshift=35mm,yshift=3mm]dens.south west) {GEPC score};

% ID curve (blue)
\path[fill=cNoi!25, draw=cNoi!90!black, line width=1.0pt]
  ([xshift=10mm,yshift=6mm]dens.south west)
  .. controls ([xshift=16mm,yshift=20mm]dens.south west) and ([xshift=25mm,yshift=20mm]dens.south west)
  .. ([xshift=30mm,yshift=11mm]dens.south west)
  .. controls ([xshift=32mm,yshift=8mm]dens.south west) and ([xshift=33mm,yshift=7mm]dens.south west)
  .. ([xshift=34mm,yshift=6mm]dens.south west)
  -- cycle;
\node[font=\scriptsize\bfseries, text=cNoi!90!black] at ([xshift=21mm,yshift=19mm]dens.south west) {ID};

% OOD curve (red)
\path[fill=cOOD!22, draw=cOOD!90!black, line width=1.0pt]
  ([xshift=28mm,yshift=6mm]dens.south west)
  .. controls ([xshift=36mm,yshift=9mm]dens.south west) and ([xshift=46mm,yshift=20mm]dens.south west)
  .. ([xshift=54mm,yshift=15mm]dens.south west)
  .. controls ([xshift=58mm,yshift=12mm]dens.south west) and ([xshift=61mm,yshift=8mm]dens.south west)
  .. ([xshift=62mm,yshift=6mm]dens.south west)
  -- cycle;
\node[font=\scriptsize\bfseries, text=cOOD!90!black] at ([xshift=52mm,yshift=20mm]dens.south west) {OOD};

% threshold
\draw[dashed, line width=0.9pt, draw=cIn!80]
  ([xshift=45mm,yshift=6mm]dens.south west) -- ([xshift=45mm,yshift=20mm]dens.south west);
\node[font=\scriptsize, text=cIn!80] at ([xshift=45mm,yshift=21.5mm]dens.south west) {$\tau$};

\draw[optflow] (calib.south) -- (dens.north);

% ==========================================================
% Map vignette (below out)
% ==========================================================
\node[smallbox=cRes, below=15mm of out, minimum height=24mm, text width=5.0cm] (map) {};
\node[lab, above=1mm of map, text=cRes!90!black] {Equivariance-breaking map};
\draw[optflow] (out.south) -- (map.north);

\begin{scope}
  \clip ([xshift=3mm,yshift=3mm]map.south west) rectangle ([xshift=-3mm,yshift=-3mm]map.north east);
  \fill[cRes!12] (map.south west) rectangle (map.north east);
  \fill[cRes!35] ([xshift=24mm,yshift=11mm]map.south west) circle [radius=5mm];
  \fill[cRes!60] ([xshift=27mm,yshift=13mm]map.south west) circle [radius=3.3mm];
  \fill[cRes!80] ([xshift=29mm,yshift=14mm]map.south west) circle [radius=2.0mm];
  \draw[draw=cRes!90!black, line width=1.2pt, rounded corners=1.2mm]
    ([xshift=18mm,yshift=7mm]map.south west) rectangle ++(18mm,13mm);
\end{scope}

% ==========================================================
% Background grouping
% ==========================================================
\begin{pgfonlayer}{background}
  \node[rounded corners=3mm, draw=cGrp!35, fill=cGrp!6, line width=1pt,
        fit=(branch)(resid), inner sep=6mm] {};
  \node[rounded corners=3mm, draw=cRes!35, fill=cRes!5, line width=1pt,
        fit=(pool)(agg), inner sep=6mm] {};
  \node[rounded corners=3mm, draw=cCal!35, fill=cCal!5, line width=1pt,
        fit=(calib)(out), inner sep=6mm] {};
\end{pgfonlayer}

\end{tikzpicture}%
}
\caption{\textbf{GEPC.} We probe group-consistency of a pretrained diffusion score field by transporting $\mathbf{x}_t$ under $g\in\mathcal{G}$, transporting scores back, and measuring $\mathbf r_t$. Residual energies are pooled, aggregated over selected timesteps, and calibrated with ID-only statistics, yielding an OOD score and equivariance-breaking maps.}
\vspace{-3mm}
\label{fig:gepc_overview}
\end{figure*}
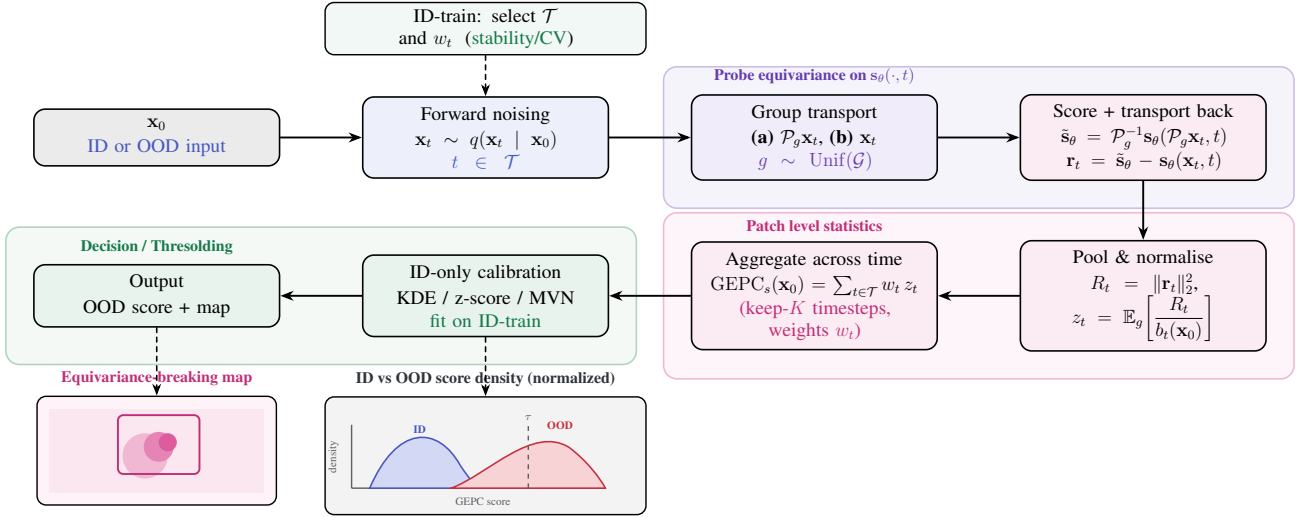

%%%%%%%%%%%%%%%%%%%%%%%%%%%%%%%%%%%%%%%%%%%%%%%%%%%%%%%%%%%%
\section{Related Work}
\label{sec:related}
%%%%%%%%%%%%%%%%%%%%%%%%%%%%%%%%%%%%%%%%%%%%%%%%%%%%%%%%%%%%

\textbf{OOD detection with discriminative models.}
Post-hoc OOD scores for classifiers are often defined on logits or penultimate features: maximum softmax probability (MSP), ODIN, and energy-based scores~\cite{hendrycks2017baseline, liang2018odin, liu2020energy}; deep $k$NN and class-aware feature decoupling further exploit representation geometry~\cite{sun2022knnOod, ling2025cadref}; gradient-based projections and PCA / kernel PCA probe feature manifolds~\cite{behpour2023gradorth, guan2023pca, fang2024kernelpca}.
A complementary line builds explicitly on \emph{matrix-induced distances} and covariance geometry: Mahalanobis-based detectors fit a Gaussian model on ID features and use the induced distance as an OOD score~\cite{lee2018mahalanobis}, while residual-space methods such as ViM and NECO weight directions in the residual subspace or exploit neural-collapse structure~\cite{wang2022vim, ammar2024neco}.
Recent work further adapts the effective covariance at test time using the current feature, shrinking directions aligned with residual activations~\cite{guo2025dcc}, and studies how controlling neural collapse via entropy regularization trades off OOD detection and OOD generalization~\cite{harun2025ncoodg}.
All these approaches operate in classifier feature space; our work is orthogonal in that we probe the \emph{score field} of a generative model through group equivariance.

\textbf{Diffusion models for OOD and anomaly detection.}
Diffusion models~\cite{ho2020ddpm, song2021sde, karras2022edm, yang2023diffusionSurvey} have been adapted to OOD via denoising- and reconstruction-based scores, trajectory energies and path discrepancies (DiffPath)~\cite{diffpath}, perturbation robustness (SPR)~\cite{shin2023spr}, and curvature- or covariance-based diagnostics (SCOPED, EigenScore)~\cite{scoped2025, eigenscore2025}.
These methods typically exploit score magnitude or local geometry along time and often require additional reverse steps or Jacobian--vector products/power iterations.
GEPC is complementary: it probes global group consistency of noised distributions via equivariance residuals, without computing Jacobian or modifying the backbone, and can be combined with curvature- or trajectory-based scores.

\textbf{Equivariance and score-based models.}
Equivariant score-based generative models combine group-equivariant parameterizations with score matching to model symmetric distributions efficiently~\cite{niu2020permutation, cohen2016gcnn, cohen2017steerable, chen2024equivariantscorebasedgenerativemodels}, while standard CNNs exhibit only approximate equivariance, degraded by subsampling and mitigated by anti-aliasing~\cite{zhang2019shiftinvariant, bruintjes2023equivariancefactors}.
GEPC takes a diagnostic angle: we treat group transports as a probe on a fixed pretrained diffusion model and interpret equivariance residuals as an empirical symmetry-breaking functional that separates ID and OOD.

\textbf{Conformal and equivariance-based OOD detection.}
iDECODe~\cite{kaur2022idecode} uses equivariance deviations as a conformal non-conformity score to obtain distribution-free calibrated decisions under random group actions.
GEPC is not a conformal method per se, but its multi-$t$ equivariance features can, in principle, be wrapped in a conformal layer when distribution-free guarantees are required.

\textbf{Equivariance as an inductive bias for OOD detection.}
Beyond score-based models, equivariance has also been used as an explicit inductive bias in discriminative unsupervised OOD detectors, e.g., via equivariant contrastive learning with soft cluster-aware semantics~\cite{11029244}.
This line is complementary to GEPC: we do not modify training or architecture, but instead use equivariance breaking of a pretrained diffusion score field as a test-time OOD signal.

% A more extended survey and additional connections to distance-based and collapse-based OOD detectors are provided in Appendix~\ref{app:extended-related}.

%%%%%%%%%%%%%%%%%%%%%%%%%%%%%%%%%%%%%%%%%%%%%%%%%%%%%%%%%%%%
\section{Background}
\label{sec-background}
%%%%%%%%%%%%%%%%%%%%%%%%%%%%%%%%%%%%%%%%%%%%%%%%%%%%%%%%%%%%

\subsection{Diffusion and score-based models}

We briefly review the foundations of DDPMs. \cite{ho2020ddpm,nichol2021improved}.
Given data $\mathbf{x}_0 \sim q(\mathbf{x}_0)$ in $\mathbb{R}^d$, we define a forward process that generates latent variables $\mathbf{x}_1$ through $\mathbf{x}_T$ by adding a white Gaussian noise of variance $\beta_t$ at time $t$ as follows:
\begin{equation}
  q(\mathbf{x}_t \mid \mathbf{x}_{t-1})
  = \mathcal{N}\bigl(\mathbf{x}_t ; \sqrt{\alpha_t}\mathbf{x}_{t-1}, \beta_t  \mathbf{I}\bigr), \, t=1,\dots,T.
\end{equation}
% with  $\alpha_t >0$ and  $\beta_t>0 $ such that $\alpha_t + \beta_t =1$.
where $\alpha_t = 1-\beta_t$ with $\beta_t\in(0,1)$.
Alternatively, we can formulate the marginal at time $t$ directly as:
\begin{equation}
  q(\mathbf{x}_t \mid \mathbf{x}_0)
  = \mathcal{N}\bigl(\mathbf{x}_t ; \sqrt{\bar\alpha_t}\,\mathbf{x}_0, (1-\bar\alpha_t)\, \mathbf{I}\bigr),
\end{equation}
with $\bar\alpha_t = \displaystyle\prod_{s=1}^t \alpha_s$. We will slightly abuse notation and refer to the forward marginal distribution of $\mathbf{x}_t$ either as $q_t(\mathbf{x}_t)$ or simply as $q_t$ when no ambiguity arises.

Equivalently, we can sample $\mathbf{x}_t$ via the reparameterization
\begin{equation}
  \mathbf{x}_t = \sqrt{\bar\alpha_t}\,\mathbf{x}_0 + \sqrt{1-\bar\alpha_t}\,\boldsymbol{\epsilon}\, ,  
\end{equation}
where $\boldsymbol{\epsilon} \sim\mathcal N(\mathbf{0},\mathbf{I})$ is independent of $\mathbf{x}_0$.

A generative model approximates the reverse conditionals $q(\mathbf{x}_{t-1}\mid \mathbf{x}_t)$ by Gaussian distributions
    $p_\theta(\mathbf{x}_{t-1} \mid \mathbf{x}_t)
      = \mathcal{N}\bigl(\mathbf{x}_{t-1}; \mu_\theta(\mathbf{x}_t,t), \tilde\beta_t \,\mathbf{I}\bigr)$, $p_\theta(\mathbf{x}_T)= \mathcal{N}(\mathbf{0},\mathbf{I})$
where $\tilde\beta_t$ is a fixed reverse variance schedule (e.g.\ the DDPM posterior variance). It is typically trained via the "simple" denoising objective:
\begin{equation}
  \mathcal{L}_{\text{simple}}(\theta)
    = \mathbb{E}_{t,\mathbf{x}_0,\boldsymbol{\epsilon}}\bigl[\|\boldsymbol{\epsilon} - \boldsymbol{\epsilon}_\theta(\mathbf{x}_t,t)\|_2^2\bigr]\, ,
\end{equation} 
where $t$ is sampled from a fixed distribution on $\{1,\dots,T\}$ (often uniform) and
$\boldsymbol{\epsilon}_\theta(\mathbf{x}_t,t)$ denotes the noise-prediction network (e.g.\ a U-Net) trained to predict the
forward noise $\boldsymbol{\epsilon}$ in $\mathbf{x}_t=\sqrt{\bar\alpha_t}\,\mathbf{x}_0+\sigma_t\boldsymbol{\epsilon}$ with $\sigma_t^2=1-\bar\alpha_t$.
Under the MSE objective, the pointwise optimum satisfies $\boldsymbol{\epsilon}_\theta(\mathbf{x}_t,t) = \E[\boldsymbol{\epsilon}\mid \mathbf{x}_t]$,
hence the associated score estimate is
\begin{equation}
\label{eq:score_estimation}
\mathbf{s}_\theta(\mathbf{x}_t,t) := -\sigma_t^{-1}\,\boldsymbol{\epsilon}_\theta(\mathbf{x}_t,t)\, .
\end{equation}
See Appendix~\ref{app:s_theta} for a detailed derivation.

\subsection{Scores and equivariance}
\label{sec-score_eq}
For any non-degenerate distribution $p$,
we denote by $\nabla_\mathbf{x}$ the (vector) gradient w.r.t.\ $\mathbf{x}\in\mathbb R^d$; thus $\nabla_\mathbf{x}\log p(\mathbf{x})\in\mathbb R^d$.
Let
\begin{equation}
\label{eq:score_exact}
  \mathbf{s}_p(\mathbf{x}) \;:=\; \nabla_\mathbf{x} \log p(\mathbf{x})\, ,
\end{equation}
denote the corresponding ideal score at time $t$. 
Thus, for any marginal $q_t(\mathbf{x}_t)$ of the forward diffusion process used to noise the data, the estimator $\mathbf{s}_\theta$ defined in equation~\eqref{eq:score_estimation} aims to predict the corresponding deterministic score $\mathbf{s}_{q_t}$, as explained in Appendix~\ref{app:s_theta}.

Let $\mathcal{G}$ be a finite group acting on $\mathbb{R}^d$ via orthogonal matrices:
for any $g \in \mathcal{G}$, we denote $\mathcal{P}_g$ the corresponding operator, with  $\mathcal{P}_g^\top \mathcal{P}_g = \mathcal{I}d$.

We say a distribution $p$ on $\mathbb{R}^d$ is \emph{$\mathcal{G}$-invariant} if
\begin{equation}
X\sim p \ \Longrightarrow\ \mathcal{P}_gX \stackrel{d}{=} X,\qquad \forall g\in \mathcal{G}\, .
\label{eq:ginv-law}
\end{equation}
% Equivalently, in pushforward notation this means $g_\#u=u$.
Since each $g$ is orthogonal, then \eqref{eq:ginv-law} is equivalent to $p(\mathcal{P}_g \mathbf{x})=p(\mathbf{x})$.
In that case, the score is $\mathcal{G}$-equivariant:
\begin{equation}
  \mathbf{s}_t(\mathcal{P}_g \mathbf{x}) \;=\; \mathcal{P}_g \mathbf{s}_t(\mathbf{x}),
  \qquad \forall \mathbf{x} ,\ \forall g\in \mathcal{G},
\end{equation}
as can be seen by differentiating $\log p(\mathcal{P}_g \mathbf{x})=\log p(\mathbf{x})$ and using $\mathcal{P}_g^\top=\mathcal{P}_g^{-1}$; see Appendix~\ref{app:inv_eq}.
%\ref{app:group-equivariance}.
If $q_0$ is approximately $\mathcal{G}$-invariant and the forward noise is isotropic, then each $q_t$ remains approximately $\mathcal{G}$-invariant, and the
corresponding scores remain approximately $\mathcal{G}$-equivariant.

In practice, approximate equivariance arises because denoising score matching fits $\mathbf{s}_\theta(\cdot,t)$ to the ideal score $\mathbf{s}_{q_t}(\cdot,t)=\nabla_{\mathbf{x}}\log q_t(\mathbf{x})$ in expectation over $\mathbf{x}\sim q_t$. Indeed, the learned score $\mathbf{s}_\theta$ appears to inherit the approximate $s_{q_t}$ equivariance in high-density regions, where the training loss is concentrated. Outside these regions, the objective provides little constraint, and equivariance may be violated arbitrarily. Architectural biases such as translation-equivariant convolutions and data augmentation can further promote such approximate symmetries. In cross-backbone settings, however, this learned equivariance is not expected to persist far from the source high-density region, which motivates the distance-to-manifold perspective in Section~\ref{sec-gepc}.

%%%%%%%%%%%%%%%%%%%%%%%%%%%%%%%%%%%%%%%%%%%%%%%%%%%%%%%%%%%%%%%%%%%%%%
\section{GEPC: Group-Equivariant Posterior Consistency}
\label{sec-gepc}
%%%%%%%%%%%%%%%%%%%%%%%%%%%%%%%%%%%%%%%%%%%%%%%%%%%%%%%%%%%%%%%%%%%%%%

For any vector field $f(\cdot,t)$ and any $g\in\mathcal{G}$ acting on $\mathbb{R}^d$ through an orthogonal matrix
$\mathcal{P}_g$ (so $\mathcal{P}_g^{-1}=\mathcal{P}_g^\top$), define the equivariance residual operator
\begin{equation}
\Delta_g f(\mathbf{x},t)\;\coloneqq\;\mathcal{P}_g^{-1} f(\mathcal{P}_g\mathbf{x},t)-f(\mathbf{x},t)\, .
\label{eq:delta-def}
\end{equation}

\begin{definition}[GEPC]
\label{def-gepc}
Let $\mathbf{s}_\theta(\cdot,t)$ denote the score field of a pretrained diffusion backbone.
Given an input $\mathbf{x}_0$, sample $\mathbf{x}_t\sim q(\mathbf{x}_t\mid \mathbf{x}_0)$ from the forward noising process.
Define the equivariance residual
\begin{equation}
R_t(\mathbf{x}_t,g)\;\coloneqq\;\|\Delta_g \mathbf{s}_\theta(\mathbf{x}_t,t)\|_2^2\, ,
\label{eq-gepc-rt}
\end{equation}
and the GEPC score
\begin{equation}
\mathrm{GEPC}(\mathbf{x}_0)\coloneqq\sum_{t\in\mathcal T} w_t
\mathbb{E}_{\mathbf{x}_t\sim q(\cdot\mid \mathbf{x}_0),\;g\sim \nu_\mathcal G}\!\big[R_t(\mathbf{x}_t,g)\big]\, ,
\label{eq-gepc}
\end{equation}
where $\nu_\mathcal G$ is uniform over the finite set $\mathcal G$, and $w_t\ge 0$ with $\sum_{t\in\mathcal T}w_t=1$.
% (For orthogonal actions, $\Delta_g \mathbf{s}_\theta(\mathbf{x},t)=\mathbf{P}_g^\top \mathbf{s}_\theta(\mathbf{P}_g\mathbf{x},t)-\mathbf{s}_\theta(\mathbf{x},t)$.)
\end{definition}

\paragraph{Why equivariance, not $\|\mathbf{s}_\theta\|$? (Gaussian mean-shift).}
Let $p=\mathcal N(\boldsymbol{\mu},\sigma^2\, \mathbf{I}_d)$, whose score is
$\mathbf{s}(\mathbf{x})=-(\mathbf{x}-\boldsymbol{\mu})/\sigma^2$.
Then $\mathbb{E}_{\mathbf{x}\sim p}\big[\|\mathbf{s}(\mathbf{x})\|_2^2\big]=d/\sigma^2$ ($d$ being the dimension of $\mathbf x$ ), independent of $\boldsymbol{\mu}$.
In contrast, the equivariance residual detects mean shifts. For $\mathcal{G}=\{\mathcal{I}_d,-\mathcal{I}_d\}$ with uniform $\nu_{\mathcal G}$,
\begin{equation}
\mathbb{E}_{g\sim\nu_\mathcal{G}} \left[\|\Delta_g \mathbf{s}(\mathbf{x})\|_2^2\right]
=\frac{2}{\sigma^4}\|\boldsymbol{\mu}\|_2^2,
\label{eq-gauss-shift}
\end{equation}
which separates $\boldsymbol{\mu}=\mathbf{0}$ (centered / invariant) from $\boldsymbol{\mu}\neq \mathbf{0}$ (non-invariant),
even though $\|\mathbf{s}(\mathbf{x})\|$ does not. Further checks are in Appendix~\ref{app:gaussian-sanity}. This intuition from the Gaussian example is confirmed in Figure~\ref{fig:hists_svhn_c100}, where GEPC shows better separation than $\|\mathbf{s}_{\theta}(\mathbf{x})\|$ on real image datasets.

\paragraph{Decomposition.}
Fix a time $t$ and let $p_t$ be any test marginal density of $\mathbf{x}_t$.
Its ideal score is $\mathbf{s}_{p_t}(\mathbf{x})\coloneqq \nabla_{\mathbf{x}}\log p_t(\mathbf{x})$, and the score approximation error is
\begin{equation}
\mathbf{e}_{p_t}(\mathbf{x},t)\;\coloneqq\;\mathbf{s}_\theta(\mathbf{x},t)-\mathbf{s}_{p_t}(\mathbf{x})\, .
\label{eq:err-def}
\end{equation}
Define the equivariance-breaking functional
\begin{equation}
\mathcal{B}^{(\mathcal{G})}(p_t)
\;\coloneqq\;
\mathbb{E}_{\mathbf{x}\sim p_t,\;g\sim\nu_\mathcal{G}}
\big[\|\Delta_g \mathbf{s}_{p_t}(\mathbf{x},t)\|_2^2\big]\, .
\label{eq:symbreak-functional}
\end{equation}
If $p_t$ is $\mathcal{G}$-invariant distribution, then $\mathcal{B}^{(\mathcal{G})}(p_t)=0$ since invariance is equivalent to score equivariance
(Appendix~\ref{app:inv_eq}).

\paragraph{Expected residual bounds (ID vs OOD).}
Let $q_t$ denote the time-$t$ marginal distribution induced by the ID training distribution $q(\mathbf{x}_0)$, and let $p_t$ denote the time-$t$ marginal distribution induced by any test distribution.

\begin{proposition}[Expected GEPC residual bounds]
\label{prop-bounds}
For any marginal $p_t$, define
\[
\Delta_E(p_t,t)\;\coloneqq\;
\E_{\mathbf{x}\sim p_t,\;g\sim\nu_\mathcal{G}}
\Big[\|\mathbf{e}_{p_t}(\mathcal{P}_g\mathbf{x},t)-\mathbf{e}_{p_t}(\mathbf{x},t)\|_2^2\Big].
\]
With the shorthand $\E_{p_t,g}[\cdot]\coloneqq \E_{\mathbf{x}\sim p_t,\;g\sim\nu_\mathcal{G}}[\cdot]$, we have
\begin{align}
\E_{p_t,g}\big[R_t(\mathbf{x},g)\big] &\le  2\,\mathcal{B}^{(\mathcal{G})}(p_t)
+4\,\E_{\mathbf{x}\sim p_t}\left[\|\mathbf{e}_{p_t}(\mathbf{x},t)\|_2^2 \right]\nonumber \\
&+4\,\E_{p_t,g}\left[\|\mathbf{e}_{p_t}(\mathcal{P}_g\mathbf{x},t)\|_2^2\right]\, :=  u_b(p_t) \,,\nonumber \\
\E_{p_t,g}\big[R_t(\mathbf{x},g)\big]  &\ge  \mathcal{B}^{(\mathcal{G})}(p_t)+\Delta_E(p_t,t)\\
&- 2\sqrt{\mathcal{B}^{(\mathcal{G})}(p_t)\,\Delta_E(p_t,t)} := l_b(p_t)\, . \nonumber 
\end{align}
\end{proposition}
The proof is provided in Appendix~\ref{app:gepc-bounds}.

\paragraph{Backbone trained on ID.}
In the ideal detection regime, the ID expected residual is small while the OOD expected residual is large:
$u_b(q_t)\ll l_b(p_t)$ for relevant OOD marginals $p_t$.
When the backbone is well trained on $q_t$, the score error
$\mathbb{E}_{\mathbf{x}\sim q_t}\|\mathbf{e}_{q_t}(\mathbf{x},t)\|_2^2$ is small.
Moreover, $\mathbb{E}_{\mathbf{x}\sim q_t,\;g}\|\mathbf{e}_{q_t}(\mathcal{P}_g\mathbf{x},t)\|_2^2$ remains small if the backbone preserves score consistency under $\mathcal{G}$ transformations,
often observed for convolutional architectures on approximately invariant data (Section~\ref{sec-score_eq}).
Finally, when $q_t$ is approximately $\mathcal{G}$-invariant, $\mathcal{B}^{(\mathcal{G})}(q_t)$ is also small, so $u_b(q_t)$ is small.
For an OOD marginal $p_t$ that violates the assumed invariances, $\mathcal{B}^{(\mathcal{G})}(p_t)$ and/or the error terms increase,
pushing $l_b(p_t)$ upward, which formalizes how GEPC separates ID from OOD via non-invariance and score mismatch.

\paragraph{Cross-backbone case.}
In cross-backbone detection, the backbone is trained on a \emph{source} distribution $r(\mathbf{x}_0)$ while detection is performed on another ID distribution $q(\mathbf{x}_0)$ (and OODs).
Score accuracy is then expected only near high-density regions under the source marginal $r_t$.
We model this by an effective source manifold of $r_t$ , $\mathcal{M}_t$, and the ambient space $\mathcal{N}_t$ of a distribution $p_t$ such that $\mathcal{N}_t\supset\mathcal{M}_t$. We denote
the projection $\pi_t:\mathcal{N}_t\to\mathcal{M}_t$ commuting with the group action. Let define $d_t(\mathbf{x})\coloneqq \|\mathbf{x}-\pi_t(\mathbf{x})\|_2$ and assume $\mathbf{s}_\theta(\cdot,t)$ is $L_t$-Lipschitz on $\mathcal{N}_t$:
\begin{equation}
\|\mathbf{s}_\theta(\mathbf{x},t)-\mathbf{s}_\theta(\mathbf{y},t)\|_2
\le  L_t\|\mathbf{x}-\mathbf{y}\|_2,
\,\,\, \forall \mathbf{x},\mathbf{y}\in\mathcal N_t.
\label{eq-lip}
\end{equation}

\begin{proposition}[Cross-backbone pointwise bounds]
\label{prop-cross-bounds}
Assume \eqref{eq-lip} and $\pi_t(\mathcal{P}_g\mathbf{x})=\mathcal{P}_g\pi_t(\mathbf{x})$ for all $\mathbf{x}\in\mathcal{N}_t$, $g\in\mathcal{G}$.
Then, for any $\mathbf{x}\in\mathcal N_t$,
\begin{equation}
\mathbb{E}_{g\sim\nu_\mathcal{G}} \left[R_t(\mathbf{x},g)\right]
\ \le\
2\,\mathbb{E}_{g\sim\nu_\mathcal{G}} \left[R_t(\pi_t(\mathbf{x}),g)\right]
\;+\;
8L_t^2\,d_t(\mathbf{x})^2.
\label{eq:cross-upper}
\end{equation}
If moreover there exist $m_t>0$ and $d_{0,t}\ge 0$ such that for all $\mathbf x\in\mathcal N_t$ with $d_t(\mathbf{x})\ge d_{0,t}$,
\begin{equation}
\Big\langle \mathbf s_\theta(\mathbf{x},t)- \mathbf s_\theta(\pi_t(\mathbf{x}),t),
\frac{\mathbf{x}-\pi_t(\mathbf{x})}{\|\mathbf{x}-\pi_t(\mathbf{x})\|_2}\Big\rangle
 \le - m_t\,d_t(\mathbf{x}),
\label{eq-dir}
\end{equation}
then, writing $\rho_t(\mathbf{x})\coloneqq \sqrt{\mathbb{E}_{g\sim\nu_\mathcal{G}} \left[R_t(\pi_t(\mathbf{x}),g)\right]}$, we have
\begin{equation}
\mathbb{E}_{g\sim\nu_\mathcal{G}} \left[R_t(\mathbf{x},g)\right]
\ \ge\
\Big(\,(m_t-L_t)\,d_t(\mathbf{x}) - \rho_t(\mathbf{x})\,\Big)^2, %_+
\label{eq-cross-lower}
\end{equation}

%where $(\cdot)_+=\max(\cdot,0)$.
\end{proposition}

The proof is provided in Appendix~\ref{app:cross-backbone-proof}, and the derivation of regularity hypothesis are discussed in Appendix~\ref{app:dir}.

\paragraph{Implications for detection.}
If the backbone is accurate and approximately equivariant on the high-density region of the source distribution, we may assume that
$\mathbb{E}_{g}\left[R_t(\mathbf{z},g)\right]$ is small for $\mathbf{z}\in\mathcal{M}_t$.
In this regime, the residual terms in Proposition~\ref{prop-cross-bounds} become negligible and the bounds are dominated by the distance-to-manifold terms (quadratic in $d_t(\mathbf{x})$), implying that the GEPC score increases as samples move away from the source manifold.

Taking expectations over $\mathbf{x}\sim p_t$ yields a comparison between in-distribution and out-of-distribution residuals: in-distribution samples satisfy an upper bound of order
$8L_t^2\,\mathbb{E}_{q_t}[d_t(\mathbf{x})^2]$,
whereas out-of-distribution samples exceed
$(m_t-L_t)^2\,\mathbb{E}_{p_t}[d_t(\mathbf{x})^2]$.
This separation suggests good detection performance when
$ \frac{\mathbb{E}_{q_t}[d_t(\mathbf{x})^2]}{\mathbb{E}_{p_t}[d_t(\mathbf{x})^2]}
\ll
\left(\frac{m_t}{L_t} - 1 \right)^2.
$

%%%%%%%%%%%%%%%%%%%%%%%%%%%%%%%%%%%%%%%%%%%%%%%%%%%%%%%%%%%%
\section{Practical GEPC for DDPM}
%%%%%%%%%%%%%%%%%%%%%%%%%%%%%%%%%%%%%%%%%%%%%%%%%%%%%%%%%%%%
We now describe how GEPC is computed in practice for discrete-time DDPM or improved-diffusion backbones.

\subsection{Per-sample GEPC, pooling, and normalisation}
Let $\mathcal{G}$ be a set of invertible image transformations with known inverses. Throughout, unless stated otherwise,
$\mathcal{G}=\{\mathrm{id},\mathrm{flip}_x,\mathrm{flip}_y,\mathrm{rot}_{90},\mathrm{rot}_{180},\mathrm{shift}_x,\mathrm{shift}_y\}$
with 1-pixel circular shifts, so $|\mathcal{G}|=7$ on $32\times32$ square images. %We denote by $\mathcal{P}_g$ the action of $g\in\mathcal{G}$ and by $\mathcal{P}_g^{-1}$ its inverse.

\paragraph{Pooling convention.}
Given a field $A\in\mathbb{R}^{C\times h\times w}$, $\mathrm{pool}(A)$ denotes a standard spatial pooling that first averages across channels and then aggregates over spatial locations by either mean-pooling or top-$k$ pooling (top-$k$ averages the $k$ largest spatial responses). With a slight abuse of notation, $\mathrm{pool}(\|\mathbf{.}\|_2^2)$ denotes pooling applied to the pointwise squared $\ell_2$-norm over channels.

Given an input $\mathbf{x}_0$ and timestep $t$, we sample $\mathbf{x}_t$ via
$\mathbf{x}_t=\sqrt{\bar\alpha_t}\,\mathbf{x}_0+\sqrt{1-\bar\alpha_t}\,\boldsymbol{\epsilon}$,
 $\boldsymbol{\epsilon}\sim\mathcal N(\mathbf{0},\mathbf{I}).$
Define the transported score residual field
\begin{equation}
\mathbf r_t(\mathbf{x}_t,g)\coloneqq\mathcal{P}_g^{-1}\,\mathbf{s}_\theta(\mathcal{P}_g\mathbf{x}_t,t)-\mathbf{s}_\theta(\mathbf{x}_t,t)\in\mathbb{R}^{C\times h\times w}.
\end{equation}
We also define the pooled score-energy normaliser
\begin{equation}
b_t(\mathbf{x}_0)\;\coloneqq\;\mathrm{pool}\big(\|\mathbf{s}_\theta(\mathbf{x}_t,t)\|_2^2\big).
\end{equation}
Our default per-timestep GEPC scalar (denoted GEPC$_s$ in the code) is the base-normalised residual energy
\begin{equation}
z^{(s)}_t(\mathbf{x}_0)\coloneqq\mathbb{E}_{g\sim\mathrm{Unif}(\mathcal{G})}\left[b_t^{-1}(\mathbf{x}_0)\mathrm{pool}\big(\|\mathbf r_t(\mathbf{x}_t,g)\|_2^2\big)\right].
\label{eq-prac-gepc-t}
\end{equation}
We optionally average \eqref{eq-prac-gepc-t} over $m$ Monte Carlo noise draws $\boldsymbol{\epsilon}$ (\texttt{mc\_samples}). Using the same transported scores $\{\mathcal{P}_g^{-1}\mathbf{s}_\theta(\mathcal{P}_g \mathbf{x}_t,t)\}_{g\in\mathcal{G}}$, we also compute alternative GEPC features, including cosine consistency, pairwise dispersion, $\mathbf{x}_0$-consistency, and cycle consistency; see
Appendix~\ref{app:gepc-features}. All quadratic (L2-type) features are reported in base-normalised form (with a feature-specific normaliser when appropriate), while the cosine feature is scale-invariant and therefore left unnormalised.

Finally, we aggregate across a small set of selected timesteps $\mathcal{T}$ using \texttt{agg\_t} (default: weighted mean)
\begin{equation}
\widehat{\mathrm{GEPC}}(\mathbf{x}_0)\;\coloneqq\;\sum_{t\in\mathcal{T}} w_t\, z^{(s)}_t(\mathbf{x}_0),\qquad \sum_{t\in\mathcal{T}}w_t=1.
\label{eq-prac-gepc}
\end{equation}

\subsection{ID-only timestep selection and calibration}
To avoid OOD-labelled tuning, we select timesteps, per-timestep weights, and calibration using ID samples only. We first form a candidate set $\mathcal{T}_{\mathrm{cand}}$ by mapping a fixed list of target schedule levels \texttt{snr\_levels} to discrete indices (for DDPM schedules this is implemented by nearest-neighbour matching on $\sqrt{\bar\alpha_t}$).

On ID-train, for each $t\in\mathcal{T}_{\mathrm{cand}}$ we compute a stability score via the coefficient of variation,  
$\mathrm{CV}(t)=\frac{\mathrm{std}(u_t(\mathbf{x}))}{|\mathrm{mean}(u_t(\mathbf{x}))|},$

where $u_t(\mathbf{x})$ is a base GEPC statistic at timestep $t$ (default: $z^{(s)}_t$). We keep the $K$ most stable steps (lowest CV), yielding $\mathcal{T}$ with $|\mathcal{T}|=K$ (\texttt{keep\_k}). Optionally, we set weights $w_t\propto 1/\mathrm{CV}(t)$ and normalise them (\texttt{weight\_t=inv\_cv}); otherwise $w_t$ is uniform (\texttt{weight\_t=none}).

\paragraph{Calibration modes (ID-only).}
Let $z_{t,f}(\mathbf{x})$ denote the enabled feature scalars (each OOD-high by construction). We support three ID-only calibration modes:
(i) \textbf{KDE} (\texttt{density\_mode=kde}): fit a 1D KDE $p_{t,f}$ per $(t,f)$ (Silverman rule-of-thumb with robust IQR bandwidth) and aggregate log-densities;
(ii) \textbf{z-score} (\texttt{density\_mode=zscore}): fit $(\mu_{t,f},\sigma_{t,f})$ and use the Gaussian log-score $-\tfrac12((z-\mu)/\sigma)^2$;
(iii) \textbf{raw} (\texttt{density\_mode=none}): no density model is fit and we directly aggregate raw OOD-high feature values.
Alternatively, \textbf{vector MVN} (\texttt{vector\_mode=mvn}) fits a single Gaussian/Mahalanobis model on the concatenated multi-$(t,f)$ feature vector.
For all density-based modes, the final anomaly score is the negative ID score (OOD-high), matching the implementation.

\subsection{Metrics and compute (F+J)}
We report AUROC and forward-equivalent compute as  $F+J$, where $F$ is one score-network forward evaluation and $J$ is one Jacobian--vector product, each counted as a forward-equivalent operation. 

GEPC is fully test-time and uses only score-network evaluations.
For GEPC, at each timestep $t$, we compute one reference score $\mathbf{s}_\theta(\mathbf{x}_t,t)$ and one batched evaluation over $\{\mathcal{P}_g\mathbf{x}_t\}_{g\in\mathcal{G}}$, hence $F=(1+|\mathcal{G}|)\,|\mathcal{T}|\,m$ and $J=0$. All GEPC feature variants reuse the same score evaluations at each $(t,g)$, so enabling additional features or feature fusion does not change $F+J$.

For methods that require a reverse trajectory of $T$ steps, we count $F=T$ score evaluations (and the corresponding $J$ terms when applicable).

%%%%%%%%%%%%%%%%%%%%%%%%%%%%%%%%%%%%%%%%%%%%%%%%%%%%%%%%%%%%
\section{Experiments}
\label{sec:experiments}
%%%%%%%%%%%%%%%%%%%%%%%%%%%%%%%%%%%%%%%%%%%%%%%%%%%%%%%%%%%%

We evaluate GEPC as a diffusion-based OOD detector under two regimes: (i) CIFAR-scale benchmarks at $32\times32$, using a single CelebA-trained improved-diffusion backbone; and (ii) a cross-domain, high-resolution setting, where a $256\times256$ LSUN-trained backbone is evaluated on radar SAR imagery, with OOD samples corresponding to targets or anomalies embedded in clutter.
We address two questions: (i) whether GEPC is competitive with state-of-the-art diffusion-based OOD scores under a strictly comparable backbone; and (ii) whether GEPC provides robust and interpretable OOD signals when a high-resolution LSUN-trained backbone is applied cross-domain to SAR imagery.

\subsection{Setup}
\label{sec:setup}

\paragraph{Backbones and evaluation regime.}
Unless stated otherwise, all diffusion-based scores are computed from a \emph{single} unconditional improved-diffusion backbone trained on CelebA at $32\times32$ using the public \texttt{improved-diffusion} codebase~\cite{ho2020ddpm}.
This checkpoint is never fine-tuned; methods differ only by their test-time statistic.
For high-resolution cross-domain evaluation, we further probe an unconditional LSUN-$256$ improved-diffusion backbone on $256\times256$ SAR patches.

\vspace{-2em}
\paragraph{Baselines.}

We compare GEPC against the two classes of OOD detection methods. First,  we consider \textbf{ID-trained} discriminative and generative baselines, including energy-based models such as IGEBM~\cite{igebm}, VAEBM~\cite{vaebm}, and Improved Contrastive Divergence (CD) ~\cite{improvedcd}, as well as Input Complexity (IC)~\cite{inputcomp}, Density of States (DOS)~\cite{dos}, Watanabe--Akaike Information Criterion (WAIC)~\cite{waic}, the Typicality Test (TT)~\cite{tt}, and the Likelihood Ratio (LR)~\cite{likelihood}.
Second, we compare to \textbf{training-free diffusion-based} scores computed from the \emph{same} CelebA-$32$ backbone, including NLL and DiffPath~\cite{diffpath}, MSMA~\cite{msma}, DDPM-OOD~\cite{graham2023ddpmood}, LMD~\cite{liu2023lmd}, and SCOPED~\cite{scoped2025}.

\subsection{CIFAR-10 / SVHN / CelebA at $32\times32$}
\label{sec:celeba-main}

We evaluate GEPC on the low-resolution regime with three ID datasets: CIFAR-10 (C10), SVHN, and CelebA (downsampled to $32\times32$).
To enable direct comparison with recent diffusion-OOD benchmarks under the same backbone, we report the 9 canonical ID/OOD pairs used in SCOPED \cite{scoped2025} and DiffPath \cite{diffpath}.

Table~\ref{tab:main_32} reports AUROC for all 9 ID/OOD pairs.
The upper block groups \emph{ID-trained} likelihood/energy-based model (EBM)-style baselines from prior work (trained per ID dataset).
The lower block groups \emph{training free} methods that operate on a single pretrained CelebA improved-diffusion backbone and differ only by their test-time scoring rule, including DiffPath, SCOPED, and our GEPC.

\begin{table*}[t]
  \centering
  \caption{AUROC for in-distribution vs.\ out-of-distribution tasks at $32\times32$ (9 standard ID/OOD pairs at 32×32).
  Higher is better. We report compute as $F+J$ (forward passes + JVPs).
  Baseline numbers for non-GEPC methods follow prior diffusion-OOD benchmarks under a CelebA backbone.}
  \label{tab:main_32}
  \small
  \setlength{\tabcolsep}{4pt}
  \begin{tabular}{lccc|ccc|ccc|cc}
    \toprule
    & \multicolumn{3}{c|}{CIFAR-10 (ID)} & \multicolumn{3}{c|}{SVHN (ID)} & \multicolumn{3}{c|}{CelebA (ID)} & \multirow{2}{*}{Avg.} & \multirow{2}{*}{$F+J$} \\
    \cmidrule(lr){2-4}\cmidrule(lr){5-7}\cmidrule(lr){8-10}
    Method
    & SVHN & CelebA & C100
    & C10 & CelebA & C100
    & C10 & SVHN & C100
    & & \\
    \midrule
    \multicolumn{12}{c}{\emph{ID-trained baselines (trained per ID)}} \\
    IC                & 0.950 & 0.863 & \second{0.736} & --    & --    & --    & --    & --    & --    & --    & -- \\
    IGEBM             & 0.630 & 0.700 & 0.500 & --    & --    & --    & --    & --    & --    & --    & -- \\
    VAEBM             & 0.830 & 0.770 & 0.620 & --    & --    & --    & --    & --    & --    & --    & -- \\
    Improved CD       & 0.910 & --    & \best{0.830} & --    & --    & --    & --    & --    & --    & --    & -- \\
    DoS               & 0.955 & 0.995 & 0.571 & 0.962 & \best{1.00}  & 0.965 & 0.949 & \second{0.997} & 0.956 & \best{0.928} & -- \\
    WAIC$^{\dagger}$        & 0.143 & 0.928 & 0.532 & 0.802 & 0.991 & 0.831 & 0.507 & 0.139 & 0.535 & 0.601 & -- \\
    TT$^{\dagger}$          & 0.870 & 0.848 & 0.548 & 0.970 & \best{1.00}  & 0.965 & 0.634 & 0.982 & 0.671 & 0.832 & -- \\
    LR$^{\dagger}$          & 0.064 & 0.914 & 0.520 & 0.819 & 0.912 & 0.779 & 0.323 & 0.028 & 0.357 & 0.524 & -- \\
    \midrule
    \multicolumn{12}{c}{\emph{Training-free diffusion methods (single CelebA backbone)}} \\
    NLL               & 0.091 & 0.574 & 0.521 & \best{0.990} & \second{0.999} & \best{0.992} & 0.814 & 0.105 & 0.786 & 0.652 & $1000F+0J$ \\
    IC (diffusion)    & 0.921 & 0.516 & 0.519 & 0.080 & 0.028 & 0.100 & 0.485 & 0.972 & 0.510 & 0.459 & $1000F+0J$ \\
    MSMA              & \second{0.957} & \best{1.00}  & 0.615 & \second{0.976} & 0.995 & \second{0.980} & 0.910 & 0.996 & 0.927 & \best{0.928} & $10F+0J$ \\
    DDPM-OOD          & 0.390 & 0.659 & 0.536 & 0.951 & 0.986 & 0.945 & 0.795 & 0.636 & 0.778 & 0.742 & $350F+0J$ \\
    LMD               & \best{0.992} & 0.557 & 0.604 & 0.919 & 0.890 & 0.881 & 0.989 & \best{1.00}  & 0.979 & 0.868 & $104F+0J$ \\
    DiffPath          & 0.910 & 0.897 & 0.590 & 0.939 & 0.979 & 0.953 & \second{0.998} & \best{1.00}  & \second{0.998} & \second{0.918} & $10F+0J$ \\
    SCOPED            & 0.814 & 0.940 & 0.477 & 0.971 & 0.996 & 0.959 & 0.925 & 0.994 & 0.962 & 0.892 & $2F+2J$ \\
    \textbf{GEPC (ours)} & 0.842 & \second{0.999} & 0.554 & 0.880 & \best{1.00} & 0.897 & \best{1.00} & \best{1.00} & \best{1.00} & 0.908 &
    $16F+0J$ \\
    \bottomrule
  \end{tabular}\\
  {\footnotesize $\dagger$ Results obtained from~\cite{dos}.}
\end{table*}
% $(1+|\mathcal{G}|)\,|\mathcal{T}|\,m\,F+0J$

\subsection{Radar SAR OOD detection and localisation}
\label{sec:sar}

We evaluate GEPC for ship/wake localisation on high-resolution SAR imagery (HRSID).
We construct an OOD task where \emph{sea-clutter-only} patches are in-distribution (ID) and patches containing at least one ship or wake are out-of-distribution (OOD).
We apply a pretrained LSUN-$256$ diffusion backbone \emph{as-is} to $256\times256$ SAR patches (no SAR fine-tuning) and compute GEPC patch-wise.
GEPC residual maps remain low on homogeneous sea clutter while concentrating on ships and wakes, yielding interpretable symmetry-breaking localisation (Figure~\ref{fig:sar_gepc_main}).
Additional datasets (SSDD), quantitative results, and further qualitative examples are provided in Appendix~\ref{app:radar-details} and Figure \ref{fig:sar_qual_appendix}.

\begin{figure*}[t]
  \centering
  \begin{subfigure}[t]{0.22\linewidth}
    \centering
    \includegraphics[width=\linewidth]{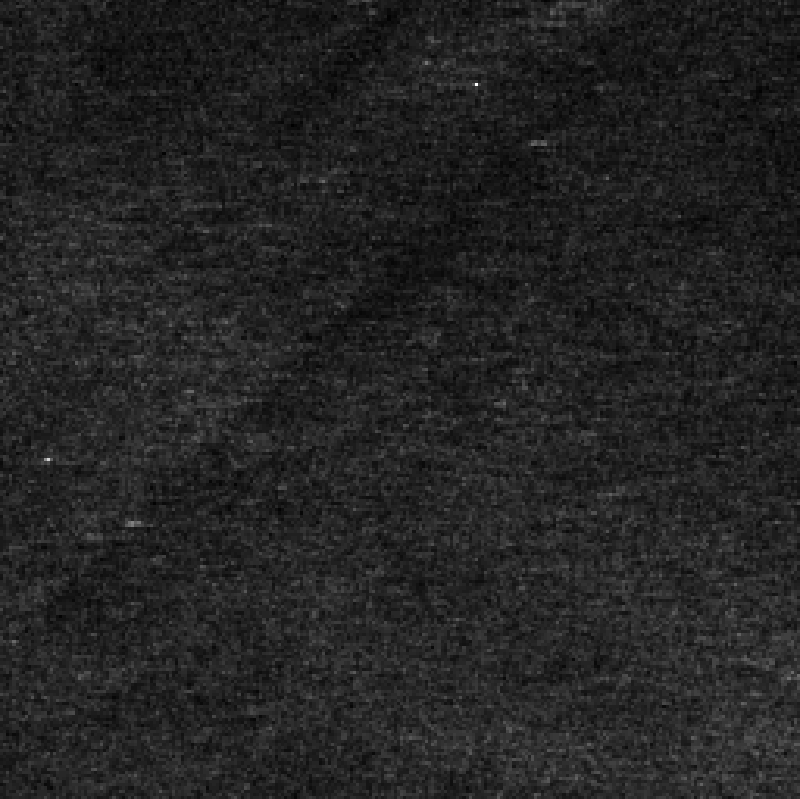}
    \caption{ID: log-mag}
  \end{subfigure}
  \hfill
  \begin{subfigure}[t]{0.26\linewidth}
    \centering
    \includegraphics[width=\linewidth]{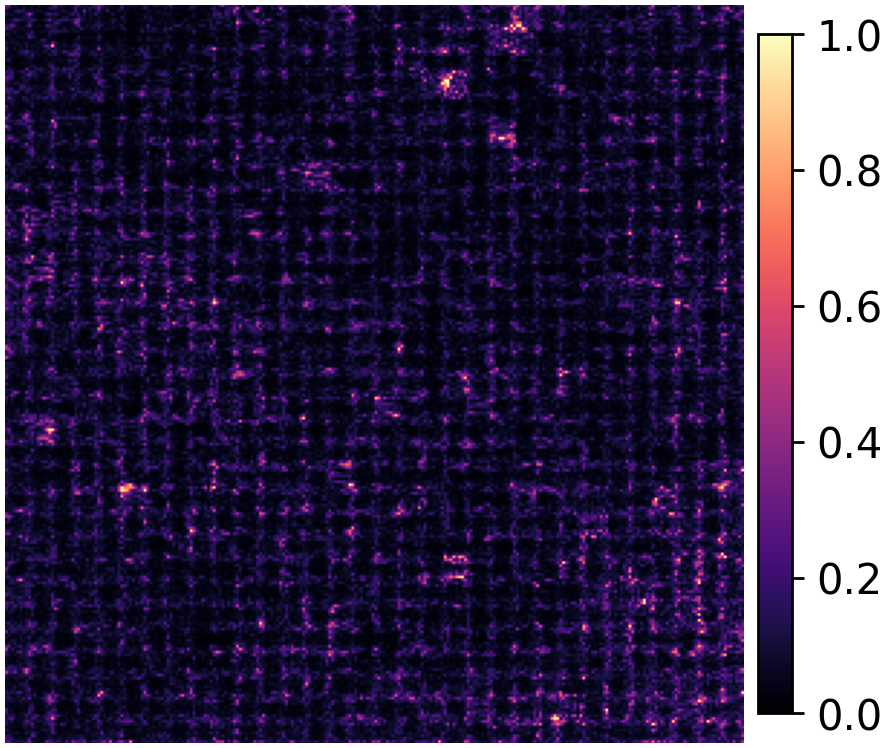}
    \caption{ID: GEPC}
  \end{subfigure}
  \hfill
  \begin{subfigure}[t]{0.22\linewidth}
    \centering
    \includegraphics[width=\linewidth]{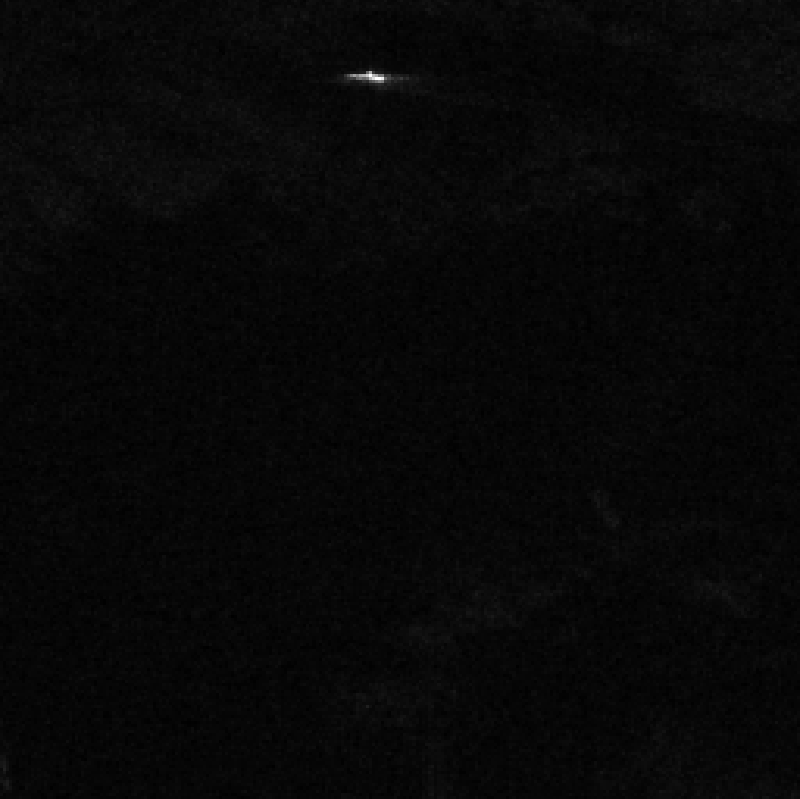}
    \caption{OOD: log-mag}
  \end{subfigure}
  \hfill
  \begin{subfigure}[t]{0.26\linewidth}
    \centering
    \includegraphics[width=\linewidth]{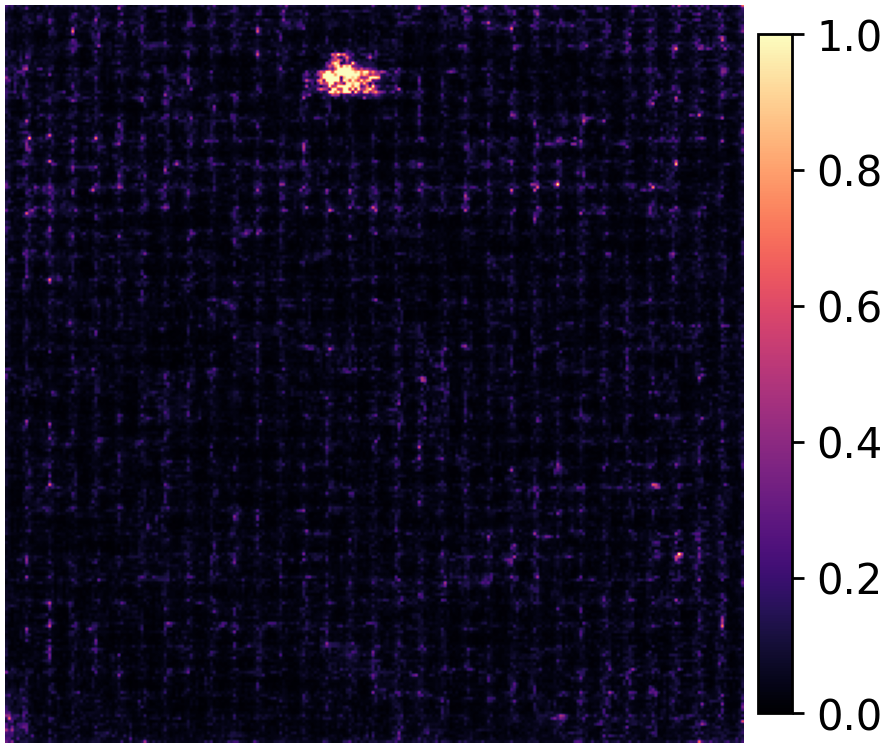}
    \caption{OOD: GEPC}
  \end{subfigure}

  \caption{
  GEPC on HRSID SAR imagery (LSUN-$256$ backbone, no SAR fine-tuning).
  We visualise the pre-pooling residual magnitude map using a \emph{global} normalisation (shared scale) to enable direct comparison between ID and OOD (Appendix~\ref{app:radar-details}, Figure~\ref{fig:sar_qual_appendix}).
  }
  \label{fig:sar_gepc_main}
\end{figure*}

\subsection{Ablations and runtime}
\label{sec:ablations}

We conduct ablations to assess robustness, sensitivity to design choices, and computational cost.
Detailed ablation tables across the 9 ID/OOD pairs are reported in Appendix~\ref{app:timestep-ablation}, along with representative plots and score histograms.

\textbf{Group elements.}
Using our default transport set (identity, flips, rotations, and 1-pixel circular shifts), we report a diagnostic AUROC obtained by isolating each transform contribution on the \emph{raw transported-gap component} (no KDE/z-score calibration), and compare it to the same component averaged over $\mathcal{G}$.
Across pairs, performance is not dominated by a single element, supporting that GEPC captures a stable symmetry-breaking effect rather than an isolated artifact (Appendix~\ref{app:timestep-ablation}, Table~\ref{tab:per_g_allpairs} and Figure~\ref{fig:ablation_c10_svhn}).

\textbf{Timestep selection and weighting.}
Single-timestep AUROC-vs-$t$ curves are shown for the \emph{raw transported-gap component} to localise where symmetry-breaking arises.
Our ID-only coefficient-of-variation (CV) rule then selects a small retained set $\mathcal{T}$ (fixed $K$ across datasets for comparable compute) and achieves performance close to the best single-timestep choices without any OOD labels (Appendix~\ref{app:timestep-ablation}, Table~\ref{tab:tselect_sweep_allpairs} and Figure~\ref{fig:ablation_c10_svhn}).

\textbf{Calibration and feature fusion.}
We compare 1D KDE calibration to z-score normalisation, the raw uncalibrated score, and a Gaussian/Mahalanobis model on multi-$t$ GEPC feature vectors.
We also ablate feature fusion via mean (Table~\ref{tab:feat_allpairs}).

\textbf{Runtime and NFEs.}
GEPC requires no backpropagation, Jacobian-vector products, nor fine-tuning.
For each retained timestep $t$, we evaluate one reference score field $\mathbf{s}_\theta(\mathbf{x}_t,t)$ and one batched evaluation over transported inputs $\{\mathcal{P}_g\mathbf{x}_t\}_{g\in\mathcal{G}}$, i.e.\ $(1+|\mathcal{G}|)$ forward passes per timestep.
With $m$ Monte-Carlo noise samples and $K=|\mathcal{T}|$ retained timesteps, the total cost is
$\mathrm{NFE}=(1+|\mathcal{G}|)Km$ forward passes per input, parallelisable over $g$ (and, memory permitting, over $t$).
We report the accuracy-compute trade-off via a sweep over $K$ with the implied NFE in Appendix~\ref{app:tselect} (Table~\ref{tab:tselect_sweep_allpairs}),
and provide measured wall-clock timing in Appendix~\ref{app:runtime}.

\paragraph{Representative plots.}
For readability, we visualise per-transform and per-timestep behaviours on a representative pair (SVHN as ID, CIFAR-100 as OOD) in Figure~\ref{fig:ablation_c10_svhn}, and show score histograms in Figure~\ref{fig:hists_svhn_c100}.
Complete 9-pair ablation tables are provided in Appendix~\ref{app:timestep-ablation}.

%%%%%%%%%%%%%%%%%%%%%%%%%%%%%%%%%%%%%%%%%%%%%%%%%%%%%%%%%%%%
\section{Conclusion and discussions}
%%%%%%%%%%%%%%%%%%%%%%%%%%%%%%%%%%%%%%%%%%%%%%%%%%%%%%%%%%%%
GEPC enables OOD detection with diffusion models by leveraging symmetry properties. It achieves competitive performance both with an ID-trained backbone and in a training-free ID setting, and provides equivariance maps that facilitate detection on complex images such as SAR imagery.

\textbf{Computational cost.}
GEPC requires multiple score evaluations per input; stochastic subsampling of group elements and timestep reduces cost but remains higher than scalar diagnostics such as score norm.
However, GEPC avoids Jacobian/Hessian evaluations and remains competitive in NFE with many diffusion-based baselines.

\textbf{Symmetry assumptions.}
GEPC relies on approximate invariances under a chosen group $\mathcal{G}$.
For modalities lacking such symmetries (e.g., strongly oriented or structured data), performance may degrade or require adapting $\mathcal{G}$ (e.g., using learned or domain-specific transformations).

\textbf{Backbone reliance.}
GEPC requires a pretrained diffusion backbone, which may not be available for all domains.
Our cross-backbone experiments, however, suggest that even mismatched backbones can be informative, consistent with recent "foundation" diffusion models reused across tasks.

\textbf{Relation to dynamic covariance calibration and neural-collapse-based OOD.}
Feature-space approaches that adapt covariance geometry or exploit neural collapse structure~\cite{guo2025dcc,ammar2024neco,harun2025ncoodg} are complementary to GEPC: they refine matrix-induced distances on classifier features, whereas GEPC probes equivariance breaking directly in the diffusion score field.
% Our experiments suggest that these signals are largely orthogonal, and combining GEPC with dynamic covariance or neural-collapse-aware backbones is a promising direction for future work.

\textbf{Extensions.}
Future work includes continuous groups and steerable operators, learned group actions, combining GEPC with curvature and path-based diagnostics, and applying GEPC multi-modal diffusion models.

\section*{Impact Statement}
This paper advances out-of-distribution detection for diffusion models, with potential applications in safety-critical sensing scenarios such as anomaly detection in radar imaging; we do not anticipate specific negative societal impacts beyond standard considerations in machine learning.  %SAR imagery;


\bibliography{example_paper}

@article{yang2023diffusionSurvey,
  title   = {Diffusion Models: A Comprehensive Survey of Methods and Applications},
  author  = {Ling Yang and Zhilong Zhang and Yang Song and Shenda Hong and Runsheng Xu and Yue Zhao and Wentao Zhang and Bin Cui and Ming-Hsuan Yang},
  journal = {ACM Computing Surveys},
  volume  = {56},
  number  = {4},
  pages   = {1--39},
  year    = {2023},
  doi     = {10.1145/3626235},
}

@inproceedings{diffpath,
title={Out-of-Distribution Detection with a Single Unconditional Diffusion Model},
author={Alvin Heng and Alexandre H. Thiery and Harold Soh},
booktitle={The Thirty-eighth Annual Conference on Neural Information Processing Systems},
year={2024},
}

@inproceedings{ho2020ddpm,
  title     = {Denoising Diffusion Probabilistic Models},
  author    = {Jonathan Ho and Ajay Jain and Pieter Abbeel},
  booktitle = {NeurIPS},
  year      = {2020},
}

@inproceedings{song2021sde,
  title     = {Score-Based Generative Modeling through Stochastic Differential Equations},
  author    = {Yang Song and Jascha Sohl-Dickstein and Diederik P. Kingma and Abhishek Kumar and Stefano Ermon and Ben Poole},
  booktitle = {ICLR},
  year      = {2021},
}

@inproceedings{karras2022edm,
  title     = {Elucidating the Design Space of Diffusion-Based Generative Models},
  author    = {Tero Karras and Miika Aittala and Timo Aila and Samuli Laine},
  booktitle = {NeurIPS},
  year      = {2022},
}

@inproceedings{sun2022knnOod,
  title     = {Out-of-Distribution Detection with Deep Nearest Neighbors},
  author    = {Yiyou Sun and Yifei Ming and Xiaojin Zhu and Yixuan Li},
  booktitle = {ICML},
  year      = {2022},
}

@inproceedings{graham2023ddpmood,
  title     = {Denoising Diffusion Models for Out-of-Distribution Detection},
  author    = {Graham, Mark S. and Pinaya, Walter H. L. and Tudosiu, Petru-Daniel and Nachev, Parashkev and Ourselin, Sebastien and Cardoso, M. Jorge},
  booktitle = {Proceedings of the IEEE/CVF Conference on Computer Vision and Pattern Recognition Workshops (CVPRW)},
  year      = {2023}
}

@inproceedings{liu2023lmd,
  title     = {Unsupervised Out-of-Distribution Detection with Diffusion Inpainting},
  author    = {Zhenzhen Liu and Jin Peng Zhou and Yufan Wang and Kilian Q. Weinberger},
  booktitle = {ICML},
  year      = {2023},
}

@inproceedings{shin2023spr,
  title     = {Anomaly Detection using Score-based Perturbation Resilience},
  author    = {Yechan Shin and Jaeseok Jang and Jinwoo Choi},
  booktitle = {ICCV},
  year      = {2023},
  pages     = {21245--21254},
}

@article{scoped2025,
  title   = {SCOPED: Score–Curvature Out-of-distribution Proximity Evaluator for Diffusion},
  author  = {Brett Barkley and Preston Culbertson and David Fridovich-Keil},
  journal = {arXiv:2510.01456},
  year    = {2025},
}

@article{eigenscore2025,
  title   = {EigenScore: OOD Detection using Covariance in Diffusion Models},
  author  = {Shirin Shoushtari and Yi Wang and Xiao Shi and M. Salman Asif and Ulugbek S. Kamilov},
  journal = {arXiv:2510.07206},
  year    = {2025}
}

@inproceedings{kaur2022idecode,
  title     = {In-Distribution Equivariance for Conformal Out-of-Distribution Detection},
  author    = {Kaur, Ramneet and Jha, Susmit and Roy, Anirban and Park, Sangdon and Dobriban, Edgar and Sokolsky, Oleg and Lee, Insup},
  booktitle = {Proceedings of the AAAI Conference on Artificial Intelligence},
  year      = {2022},
  volume    = {36},
  number    = {7},
  pages     = {7104--7114},
}

@inproceedings{zhang2019shiftinvariant,
  title     = {Making Convolutional Networks Shift-Invariant Again},
  author    = {Zhang, Richard},
  booktitle = {International Conference on Machine Learning (ICML)},
  year      = {2019},
  pages     = {7324--7334},
}

@inproceedings{bruintjes2023equivariancefactors,
  title     = {What Affects Learned Equivariance in Deep Image Recognition Models?},
  author    = {Bruintjes, Ruben and Vreuls, Vincent and Koniusz, Piotr and Georgoulis, Stamatios and Gavves, Efstratios},
  booktitle = {Proceedings of the IEEE/CVF Conference on Computer Vision and Pattern Recognition Workshops (CVPRW)},
  year      = {2023}
}

@inproceedings{cohen2016gcnn,
  title     = {Group Equivariant Convolutional Networks},
  author    = {Cohen, Taco S. and Welling, Max},
  booktitle = {International Conference on Machine Learning (ICML)},
  year      = {2016},
  pages     = {2990--2999},
}

@inproceedings{cohen2017steerable,
  title     = {Steerable {CNN}s},
  author    = {Cohen, Taco S. and Welling, Max},
  booktitle = {International Conference on Learning Representations (ICLR)},
  year      = {2017},
}

@inproceedings{ling2025cadref,
  title={CADRef: Robust Out-of-Distribution Detection via Class-Aware Decoupled Relative Feature Leveraging},
  author={Ling, Zhiwei and Chang, Yachen and Zhao, Hailiang and Zhao, Xinkui and Chow, Kingsum and Deng, Shuiguang},
  booktitle={Proceedings of the IEEE/CVF Conference on Computer Vision and Pattern Recognition (CVPR)},
  year={2025},
}

@inproceedings{guan2023pca,
  title={Revisit PCA-based Technique for Out-of-Distribution Detection},
  author={Guan, Xiaoyuan and Liu, Zhouwu and Zheng, Wei-Shi and Zhou, Yuren and Wang, Ruixuan},
  booktitle={Proceedings of the IEEE/CVF International Conference on Computer Vision (ICCV)},
  year={2023},
}

@inproceedings{fang2024kernelpca,
  title={Kernel PCA for Out-of-Distribution Detection},
  author={Fang, Kun and Tao, Qinghua and Lv, Kexin and He, Mingzhen and Huang, Xiaolin and Yang, Jie},
  booktitle={Advances in Neural Information Processing Systems},
  year={2024}
}

@inproceedings{behpour2023gradorth,
  title={GradOrth: A Simple yet Efficient Out-of-Distribution Detection with Orthogonal Projection of Gradients},
  author={Behpour, Sima and Doan, Thang and Li, Xin and He, Wenbin and Gou, Liang and Ren, Liu},
  booktitle={Advances in Neural Information Processing Systems},
  year={2023}
}

@inproceedings{niu2020permutation,
  author    = {Chenhao Niu and Yang Song and Jiaming Song and Shengjia Zhao
               and Aditya Grover and Stefano Ermon},
  title     = {Permutation Invariant Graph Generation via Score-Based Generative Modeling},
  booktitle = {Proceedings of the 23rd International Conference on Artificial Intelligence and Statistics (AISTATS)},
  year      = {2020},
  series    = {Proceedings of Machine Learning Research},
  volume    = {108},
  pages     = {4474--4484}
}

@article{hendrycks2017baseline,
  author    = {Dan Hendrycks and Kevin Gimpel},
  title     = {A Baseline for Detecting Misclassified and Out-of-Distribution Examples in Neural Networks},
  journal = {Proceedings of International Conference on Learning Representations},
  year = {2017},
}

@inproceedings{liang2018odin,
  author={Shiyu Liang and Yixuan Li and R. Srikant},
  title={Enhancing The Reliability of Out-of-distribution Image Detection in Neural Networks},
  year={2018},
  cdate={1514764800000},
}

@article{liu2020energy,
  title={Energy-based Out-of-distribution Detection},
  author={Liu, Weitang and Wang, Xiaoyun and Owens, John and Li, Yixuan},
  journal={Advances in Neural Information Processing Systems},
  year={2020}
 }

@InProceedings{pmlr-v162-hoogeboom22a,
  title = 	 {Equivariant Diffusion for Molecule Generation in 3{D}},
  author =       {Hoogeboom, Emiel and Satorras, V\'{\i}ctor Garcia and Vignac, Cl{\'e}ment and Welling, Max},
  booktitle = 	 {Proceedings of the 39th International Conference on Machine Learning},
  pages = 	 {8867--8887},
  year = 	 {2022},
  editor = 	 {Chaudhuri, Kamalika and Jegelka, Stefanie and Song, Le and Szepesvari, Csaba and Niu, Gang and Sabato, Sivan},
  volume = 	 {162},
  series = 	 {Proceedings of Machine Learning Research},
  month = 	 {17--23 Jul},
  publisher =    {PMLR},
}

@inproceedings{NEURIPS2024_587b3f36,
 author = {Cornet, Fran\c{c}ois and Bartosh, Grigory and Schmidt, Mikkel N. and Naesseth, Christian A.},
 booktitle = {Advances in Neural Information Processing Systems},
 doi = {10.52202/079017-1564},
 editor = {A. Globerson and L. Mackey and D. Belgrave and A. Fan and U. Paquet and J. Tomczak and C. Zhang},
 pages = {49429--49460},
 publisher = {Curran Associates, Inc.},
 title = {Equivariant Neural Diffusion for Molecule Generation},
 volume = {37},
 year = {2024}
}

@article{DBLP:journals/bmcbi/ZhangLLWG24,
  author={Hao Zhang and Yang Liu and Xiaoyan Liu and Cheng Wang and Maozu Guo},
  title={Equivariant score-based generative diffusion framework for 3D molecules},
  year={2024},
  month={December},
  cdate={1733011200000},
  journal={BMC Bioinform.},
  volume={25},
  number={1},
  pages={203},
}

@misc{chen2024equivariantscorebasedgenerativemodels,
      title={Equivariant score-based generative models provably learn distributions with symmetries efficiently}, 
      author={Ziyu Chen and Markos A. Katsoulakis and Benjamin J. Zhang},
      year={2024},
      eprint={2410.01244},
      archivePrefix={arXiv},
      primaryClass={stat.ML},
}

@inproceedings{
tahmasebi2024sample,
title={Sample Complexity Bounds for Estimating Probability Divergences under Invariances},
author={Behrooz Tahmasebi and Stefanie Jegelka},
booktitle={Forty-first International Conference on Machine Learning},
year={2024},
}

@inproceedings{guo2025dcc,
title={Improving Out-of-Distribution Detection via Dynamic Covariance Calibration},
author={Kaiyu Guo and Zijian Wang and Tan Pan and Brian C. Lovell and Mahsa Baktashmotlagh},
booktitle={Forty-second International Conference on Machine Learning},
year={2025},
}

@INPROCEEDINGS{wang2022vim,
  author={Wang, Haoqi and Li, Zhizhong and Feng, Litong and Zhang, Wayne},
  booktitle={2022 IEEE/CVF Conference on Computer Vision and Pattern Recognition (CVPR)}, 
  title={ViM: Out-Of-Distribution with Virtual-logit Matching}, 
  year={2022},
  volume={},
  number={},
  pages={4911-4920},
  doi={10.1109/CVPR52688.2022.00487}}

@inproceedings{ammar2024neco,
title={{NECO}: {NE}ural Collapse Based Out-of-distribution detection},
author={Mou{\"\i}n Ben Ammar and Nacim Belkhir and Sebastian Popescu and Antoine Manzanera and Gianni Franchi},
booktitle={The Twelfth International Conference on Learning Representations},
year={2024},
}

@inproceedings{harun2025ncoodg,
title={Controlling Neural Collapse Enhances Out-of-Distribution Detection and Transfer Learning},
author={Md Yousuf Harun and Jhair Gallardo and Christopher Kanan},
booktitle={Forty-second International Conference on Machine Learning},
year={2025},
}

@misc{lee2018mahalanobis,
  title        = {A Simple Unified Framework for Detecting Out-of-Distribution Samples and Adversarial Attacks},
  author       = {Lee, Kimin and Lee, Kibok and Lee, Honglak and Shin, Jinwoo},
  year         = {2018},
  eprint       = {1807.03888},
  archivePrefix= {arXiv},
  primaryClass = {stat.ML},
  note         = {Accepted in NeurIPS (NIPS) 2018},
  doi          = {10.48550/arXiv.1807.03888}
}

@article{vincent2011connection,
  title   = {A Connection Between Score Matching and Denoising Autoencoders},
  author  = {Pascal Vincent},
  journal = {Neural Computation},
  volume  = {23},
  number  = {7},
  pages   = {1661--1674},
  year    = {2011}
}

@inproceedings{nichol2021improved,
  title     = {Improved Denoising Diffusion Probabilistic Models},
  author    = {Alex Nichol and Prafulla Dhariwal},
  booktitle = {Proceedings of the 38th International Conference on Machine Learning},
  year      = {2021},
  series    = {Proceedings of Machine Learning Research},
  volume    = {139},
  pages     = {8162--8171},
}

@ARTICLE{11029244,
  author={Huang, Kuiyun and Chen, Menglong and Zheng, Hong and Lin, Baihong and Fan, Shicai},
  journal={IEEE Transactions on Circuits and Systems for Video Technology}, 
  title={Soft Cluster-Aware Equivariant Contrastive Learning for Unsupervised Out-of-Distribution Detection}, 
  year={2025},
  volume={35},
  number={11},
  pages={11309-11322},
  doi={10.1109/TCSVT.2025.3578365}}

@inproceedings{robbins1956empiricalbayes,
  title     = {An Empirical Bayes Approach to Statistics},
  author    = {Robbins, Herbert},
  booktitle = {Proceedings of the Third Berkeley Symposium on Mathematical Statistics and Probability, Volume 1: Contributions to the Theory of Statistics},
  pages     = {157--163},
  year      = {1956},
  publisher = {University of California Press}
}

@article{efron2011tweedie,
  title   = {Tweedie's Formula and Selection Bias},
  author  = {Efron, Bradley},
  journal = {Journal of the American Statistical Association},
  volume  = {106},
  number  = {496},
  pages   = {1602--1614},
  year    = {2011},
  doi     = {10.1198/jasa.2011.tm11181}
}

@article{saremi2019neuralEB,
  title   = {Neural Empirical Bayes},
  author  = {Saremi, Saeed and Hyv{\"a}rinen, Aapo},
  journal = {Journal of Machine Learning Research},
  volume  = {20},
  number  = {181},
  pages   = {1--23},
  year    = {2019}
}

@article{guo2005immse,
  title   = {Mutual Information and Minimum Mean-Square Error in Gaussian Channels},
  author  = {Guo, Dongning and Shamai, Shlomo and Verd{\'u}, Sergio},
  journal = {IEEE Transactions on Information Theory},
  volume  = {51},
  number  = {4},
  pages   = {1261--1282},
  year    = {2005},
  doi     = {10.1109/TIT.2005.844072}
}

@article{brascamp1976bl,
  title   = {On Extensions of the Brunn--Minkowski and Pr{\'e}kopa--Leindler Theorems, Including Inequalities for Log Concave Functions, and with an Application to the Diffusion Equation},
  author  = {Brascamp, Herm Jan and Lieb, Elliott H.},
  journal = {Journal of Functional Analysis},
  volume  = {22},
  number  = {4},
  pages   = {366--389},
  year    = {1976}
}

@article{dalalyan2017sampling,
  title   = {Theoretical Guarantees for Sampling from Smooth and Log-Concave Densities},
  author  = {Dalalyan, Arnak S.},
  journal = {Journal of the Royal Statistical Society: Series B},
  year    = {2017}
}

@article{durmus2019ula,
  title   = {High-Dimensional Bayesian Inference via the Unadjusted Langevin Algorithm},
  author  = {Durmus, Alain and Moulines, {\'E}ric},
  journal = {Bernoulli},
  year    = {2019}
}

@article{alain2014regularizedAEs,
  title   = {What Regularized Auto-Encoders Learn from the Data-Generating Distribution},
  author  = {Alain, Guillaume and Bengio, Yoshua},
  journal = {Journal of Machine Learning Research},
  volume  = {15},
  pages   = {3563--3593},
  year    = {2014}
}

@book{boyd2004convex,
  title     = {Convex Optimization},
  author    = {Boyd, Stephen and Vandenberghe, Lieven},
  year      = {2004},
  publisher = {Cambridge University Press}
}

@inproceedings{igebm,
 author = {Du, Yilun and Mordatch, Igor},
 booktitle = {Advances in Neural Information Processing Systems},
 editor = {H. Wallach and H. Larochelle and A. Beygelzimer and F. d\textquotesingle Alch\'{e}-Buc and E. Fox and R. Garnett},
 pages = {},
 publisher = {Curran Associates, Inc.},
 title = {Implicit Generation and Modeling with Energy Based Models},
 volume = {32},
 year = {2019}
}

@inproceedings{vaebm,
title={{\{}VAEBM{\}}: A Symbiosis between Variational Autoencoders and Energy-based Models},
author={Zhisheng Xiao and Karsten Kreis and Jan Kautz and Arash Vahdat},
booktitle={International Conference on Learning Representations},
year={2021}
}

@misc{improvedcd,
title={Improved Contrastive Divergence Training of Energy Based Models},
author={Yilun Du and Shuang Li and Joshua B. Tenenbaum and Igor Mordatch},
year={2021},
}

@misc{dos,
      title={Density of States Estimation for Out-of-Distribution Detection}, 
      author={Warren R. Morningstar and Cusuh Ham and Andrew G. Gallagher and Balaji Lakshminarayanan and Alexander A. Alemi and Joshua V. Dillon},
      year={2020},
      eprint={2006.09273},
      archivePrefix={arXiv},
      primaryClass={cs.LG},
}

@misc{waic,
      title={WAIC, but Why? Generative Ensembles for Robust Anomaly Detection}, 
      author={Hyunsun Choi and Eric Jang and Alexander A. Alemi},
      year={2019},
      eprint={1810.01392},
      archivePrefix={arXiv},
      primaryClass={stat.ML},
}

@misc{tt,
      title={Detecting Out-of-Distribution Inputs to Deep Generative Models Using Typicality}, 
      author={Eric Nalisnick and Akihiro Matsukawa and Yee Whye Teh and Balaji Lakshminarayanan},
      year={2019},
      eprint={1906.02994},
      archivePrefix={arXiv},
      primaryClass={stat.ML},
}

@misc{likelihood,
      title={Likelihood Ratios for Out-of-Distribution Detection}, 
      author={Jie Ren and Peter J. Liu and Emily Fertig and Jasper Snoek and Ryan Poplin and Mark A. DePristo and Joshua V. Dillon and Balaji Lakshminarayanan},
      year={2019},
      eprint={1906.02845},
      archivePrefix={arXiv},
      primaryClass={stat.ML},
}

@inproceedings{inputcomp,
title={Input Complexity and Out-of-distribution Detection with Likelihood-based Generative Models},
author={Joan Serrà and David Álvarez and Vicenç Gómez and Olga Slizovskaia and José F. Núñez and Jordi Luque},
booktitle={International Conference on Learning Representations},
year={2020},
}

@inproceedings{msma,
title={Multiscale Score Matching for Out-of-Distribution Detection},
author={Ahsan Mahmood and Junier Oliva and Martin Andreas Styner},
booktitle={International Conference on Learning Representations},
year={2021},
}
\bibliographystyle{icml2026}

%%%%%%%%%%%%%%%%%%%%%%%%%%%%%%%%%%%%%%%%%%%%%%%%%%%%%%%%%%%%%%%%%%%%%%%%%%%%%%%
%%%%%%%%%%%%%%%%%%%%%%%%%%%%%%%%%%%%%%%%%%%%%%%%%%%%%%%%%%%%%%%%%%%%%%%%%%%%%%%
% APPENDIX
%%%%%%%%%%%%%%%%%%%%%%%%%%%%%%%%%%%%%%%%%%%%%%%%%%%%%%%%%%%%%%%%%%%%%%%%%%%%%%%
%%%%%%%%%%%%%%%%%%%%%%%%%%%%%%%%%%%%%%%%%%%%%%%%%%%%%%%%%%%%%%%%%%%%%%%%%%%%%%%
\newpage
\appendix
\onecolumn

% You can have as much text here as you want. The main body must be at most $8$
% pages long. For the final version, one more page can be added. If you want, you
% can use an appendix like this one.

% The $\mathtt{\backslash onecolumn}$ command above can be kept in place if you
% prefer a one-column appendix, or can be removed if you prefer a two-column
% appendix.  Apart from this possible change, the style (font size, spacing,
% margins, page numbering, etc.) should be kept the same as the main body.

%%%%%%%%%%%%%%%%%%%%%%%%%%%%%%%%%%%%%%%%%%%%%%%%%%%%%%%%%%%%%%%%%%%%%%
\section{Diffusion and score-matching identities (detailed)}
\label{app:diffusion-details}
%%%%%%%%%%%%%%%%%%%%%%%%%%%%%%%%%%%%%%%%%%%%%%%%%%%%%%%%%%%%%%%%%%%%%%

We collect detailed derivations for the identities used in the main text:
(i) denoising-score identities linking $\boldsymbol{\epsilon}$-prediction to scores,
(ii) Tweedie's formula under DDPM scaling,
(iii) posterior covariance identities and their relation to the Jacobian,
(iv) Lipschitz / contractivity properties derived from posterior covariance bounds.

\subsection{Forward noising closed-form (DDPM)}
Recall that the forward diffusion process is defined by
\begin{equation}
q(\mathbf{x}_t\mid \mathbf{x}_{t-1})=\mathcal N\left(\mathbf{x}_t;\sqrt{\alpha_t}\,\mathbf{x}_{t-1},\beta_t \, \mathbf{I}\right),
\qquad \alpha_t=1-\beta_t,
\end{equation}
where we denote $\bar\alpha_t=\displaystyle\prod_{s=1}^t\alpha_s$. It follows that the marginal distribution admits the closed form:
\begin{equation}
q(\mathbf{x}_t\mid \mathbf{x}_0)=\mathcal N\left(\mathbf{x}_t;\sqrt{\bar\alpha_t}\,\mathbf{x}_0,(1-\bar\alpha_t)\,\mathbf{I}\right),
\end{equation}
and equivalently that
\begin{equation}
\mathbf{x}_t=\sqrt{\bar\alpha_t}\,\mathbf{x}_0+\sigma_t\,\boldsymbol{\epsilon},
\ \sigma_t^2\, \coloneqq 1-\bar\alpha_t.
\label{eq:app-forward}
\end{equation}

\subsection{Denoising-score identity: $\E[\boldsymbol{\epsilon}\mid \mathbf{x}_t]$ and applying the forward process in $\mathbf{s}_t(\mathbf{x}_t)$}
\label{app:s_theta}
Let $p_t$ denote the marginal density of $\mathbf{x}_t$ induced by $\mathbf{x}_0\sim p_0$ and \eqref{eq:app-forward}.
We define the ideal score as $\mathbf{s}_t(\mathbf{x})\coloneqq \nabla_\mathbf x\log p_t(\mathbf{x})$~\cite{vincent2011connection,saremi2019neuralEB}.

\begin{lemma}[Conditional-noise / score identity]
\label{lem:app-eps-score}
For each fixed $t$,
\begin{equation}
\mathbf{s}_t(\mathbf{x}_t)= -\frac{1}{\sigma_t}\,\E[\boldsymbol{\epsilon}\mid \mathbf{x}_t].
\label{eq:app-score-eps}
\end{equation}
\end{lemma}

\begin{proof}
Let $K_t(\mathbf{x}_t\mid \mathbf{x}_0)=\mathcal N\left(\mathbf{x}_t;\sqrt{\bar\alpha_t}\,\mathbf{x}_0,\sigma_t^2 \,\mathbf{I}\right)$ denote the Gaussian transition kernel of the forward process. The marginal density of $\mathbf{x}_t$ can then be written as
\[
p_t(\mathbf{x}_t)=\int p_0(\mathbf{x}_0)\,K_t(\mathbf{x}_t\mid \mathbf{x}_0)\,d\mathbf{x}_0.
\]

Differentiating under the integral yields:
\[
\nabla_{\mathbf{x}_t}p_t(\mathbf{x}_t)=\int p_0(\mathbf{x}_0)\,K_t(\mathbf{x}_t\mid \mathbf{x}_0)\,\nabla_{\mathbf{x}_t}\log K_t(\mathbf{x}_t\mid \mathbf{x}_0)\,d\mathbf{x}_0.
\]
Since
\[
\nabla_{\mathbf{x}_t}\log K_t(\mathbf{x}_t\mid \mathbf{x}_0)= -\frac{\mathbf{x}_t-\sqrt{\bar\alpha_t}\,\mathbf{x}_0}{\sigma_t^2}.
\]
We obtain:
\[
\nabla_{\mathbf{x}_t}\log p_t(\mathbf{x}_t)=
\frac{\nabla p_t(\mathbf{x}_t)}{p_t(\mathbf{x}_t)}
=
-\E\!\left[\frac{\mathbf{x}_t-\sqrt{\bar\alpha_t}\,\mathbf{x}_0}{\sigma_t^2}\,\Big|\,\mathbf{x}_t\right].
\]
Using the identity, $\boldsymbol{\epsilon}=(\mathbf{x}_t-\sqrt{\bar\alpha_t}\,\mathbf{x}_0)/\sigma_t$, this simplifies to
\[
\nabla_{\mathbf{x}_t}\log p_t(\mathbf{x}_t)= -\frac{1}{\sigma_t}\E[\boldsymbol{\epsilon}\mid \mathbf{x}_t]\,,
\]
which establishes \eqref{eq:app-score-eps}.
\end{proof}

\paragraph{Implication for $\boldsymbol{\epsilon}$-prediction.}
By definition of the mean squared error objective, $\boldsymbol{\epsilon}_\theta(\mathbf{x}_t,t)$ is an estimator of $ \E[\boldsymbol{\epsilon}\mid \mathbf{x}_t]$. Combining this observation with \eqref{eq:app-score-eps} yields
\begin{equation}
\mathbf{s}_\theta(\mathbf{x}_t,t)=-\frac{1}{\sigma_t}\,\boldsymbol{\epsilon}_\theta(\mathbf{x}_t,t).
\label{eq:app-score-from-eps}
\end{equation}

\subsection{Tweedie formula under DDPM scaling (posterior mean of $\mathbf{x}_0$)}
\label{app:dir}
The classical Tweedie formula is for additive noise model  $\mathbf{y}=\mathbf{x}+\sigma\,\boldsymbol{\epsilon}$~\cite{robbins1956empiricalbayes,efron2011tweedie}.
whereas DDPM involves an additional scaling factor $\sqrt{\bar{\alpha}_t}$. We therefore reduce to the additive setting by introducing a rescaled variable:
\begin{equation}
\mathbf{y}_t \coloneqq \frac{\mathbf{x}_t}{\sqrt{\bar\alpha_t}}
= \mathbf{x}_0 + \tilde\sigma_t \,\boldsymbol{\epsilon},
\qquad \tilde\sigma_t\coloneqq \frac{\sigma_t}{\sqrt{\bar\alpha_t}}.
\label{eq:app-rescale}
\end{equation}
Let $\tilde p_t$ denote the marginal density of $\mathbf{y}_t$ and define its score by $\tilde{\mathbf{s}}_t(\mathbf{y})\coloneqq \nabla_\mathbf{y}\log \tilde p_t(\mathbf{y})$.
For clarity, in the next two subsections we work with a generic additive Gaussian model
$\mathbf{y}=\mathbf{x}_0+\tilde\sigma\,\boldsymbol{\epsilon}$ and omit the time index $t$,
writing $\tilde p$ and $\tilde{\mathbf{s}}$ for the corresponding marginal and score.

\begin{lemma}[Tweedie (additive form)]
\label{lem:app-tweedie-add}
For the additive noise model $\mathbf{y}=\mathbf{x}_0+\tilde\sigma\,\boldsymbol{\epsilon}$, we have
\begin{equation}
\E[\mathbf{x}_0\mid \mathbf{y}] = \mathbf{y} + \tilde\sigma^2\,\tilde{\mathbf{s}}(\mathbf{y}).
\label{eq:app-tweedie-add}
\end{equation}
\end{lemma}

\begin{proof}
Applying Lemma~\ref{lem:app-eps-score} in the additive model gives
$\tilde{\mathbf{s}}(\mathbf{y})=-(1/\tilde\sigma)\, \E[\boldsymbol{\epsilon}\mid \mathbf{y}]$ and
$\mathbf{x}_0=\mathbf{y}-\tilde\sigma\,\boldsymbol{\epsilon}$. Taking the conditional expectation yields
$\E[\mathbf{x}_0\mid \mathbf{y}]=\mathbf{y}-\tilde\sigma\,\E[\boldsymbol{\epsilon}\mid \mathbf{y}]=\mathbf{y}+\tilde\sigma^2 \,\tilde{\mathbf{s}}(\mathbf{y})$, which establishes \eqref{eq:app-tweedie-add}.
\end{proof}

We now translate the Tweedie formula back to $\mathbf{x}_t$. Since $\mathbf{y}_t=\mathbf{x}_t/\sqrt{\bar\alpha_t}$, the
score transforms by the chain rule:
\begin{equation}
\tilde{\mathbf{s}}_t(\mathbf{y}_t)
=\nabla_{\mathbf{y}_t}\log \tilde p_t(\mathbf{y}_t)
=\sqrt{\bar\alpha_t}\,\nabla_{\mathbf{x}_t}\log p_t(\mathbf{x}_t)
=\sqrt{\bar\alpha_t}\,\mathbf{s}_t(\mathbf{x}_t).
\label{eq:app-score-rescale}
\end{equation}
Combining \eqref{eq:app-tweedie-add} and \eqref{eq:app-score-rescale}  gives the DDPM-scaled Tweedie formula:

\begin{lemma}[Tweedie for DDPM]
\label{lem:app-tweedie-ddpm}
Let $m(\mathbf{x}_t)\coloneqq \E[\mathbf{x}_0\mid \mathbf{x}_t]$ denote the Bayes denoiser (posterior mean). Then
\begin{equation}
m(\mathbf{x}_t)
=
\frac{1}{\sqrt{\bar\alpha_t}}\Big(\mathbf{x}_t+\sigma_t^2\,\mathbf{s}_t(\mathbf{x}_t)\Big).
\label{eq:app-tweedie-ddpm}
\end{equation}
\end{lemma}

\subsection{Posterior covariance and Jacobian: $\mathrm{Cov}(\mathbf{x}_0\mid \mathbf{x}_t)$}
This subsection makes explicit the identity “Jacobian = posterior covariance” that underlies Lipschitz and contractivity arguments~\cite{saremi2019neuralEB,guo2005immse}. We work in the additive form $\mathbf{y}=\mathbf{x}_0+\tilde\sigma\,\boldsymbol{\epsilon}$ for clarity. Let $m(\mathbf{y})\coloneqq \E[\mathbf{x}_0\mid \mathbf{y}]$ and $C(\mathbf{y})\coloneqq \mathrm{Cov}(\mathbf{x}_0\mid \mathbf{y})$.

\begin{lemma}[Posterior covariance identity]
\label{lem:app-postcov-hess}
For additive Gaussian noise,
\begin{equation}
C(\mathbf{y})=\tilde\sigma^2 \, \mathbf{I} + \tilde\sigma^4 \,\nabla_\mathbf{y}^2 \log \tilde p(\mathbf{y}),
\label{eq:app-cov-hess}
\end{equation}
and equivalently, using $m(\mathbf{y})=\mathbf{y}+\tilde\sigma^2 \, \nabla_\mathbf{y}\log\tilde p(\mathbf{y})$,
\begin{equation}
\nabla_\mathbf{y} m(\mathbf{y}) = \mathbf{I} + \tilde\sigma^2 \,\nabla_\mathbf{y}^2 \log \tilde p(\mathbf{y}),
\qquad
C(\mathbf{y})=\tilde\sigma^2\,\nabla_\mathbf{y} m(\mathbf{y}).
\label{eq:app-cov-jac}
\end{equation}
\end{lemma}

\begin{proof}
We start from the posterior mean expressed as $$m(\mathbf{y})=\displaystyle\frac{1}{\tilde p(\mathbf{y})}\int \mathbf{x}\,p_0(\mathbf{x})\,\phi_{\tilde\sigma}(\mathbf{y}-\mathbf{x})\,d\mathbf{x}\, ,$$
where $\phi_{\tilde\sigma}(.)$ is the Gaussian density with variance $\tilde\sigma^2\, \mathbf{I}$.
Differentiating componentwise with respect to $\mathbf{y}$ and using
$\nabla_\mathbf{y} \phi_{\tilde\sigma}(\mathbf{y}-\mathbf{x})=-(\mathbf{y}-\mathbf{x})\phi_{\tilde\sigma}(\mathbf{y}-\mathbf{x})/\tilde\sigma^2$, a standard quotient-rule calculation gives
\[
\nabla_\mathbf{y} m(\mathbf{y})=\frac{1}{\tilde\sigma^2}\Big(\E[\mathbf{x}_0\mathbf{x}_0^\top\mid \mathbf{y}]-\E[\mathbf{x}_0\mid \mathbf{y}]\, \E[\mathbf{x}_0\mid \mathbf{y}]^\top\Big)
=\frac{1}{\tilde\sigma^2}\,C(\mathbf{y})\, .
\]
This immediately yields $C(\mathbf{y})=\tilde\sigma^2\,\nabla_\mathbf{y} m(\mathbf{y})$.

To obtain \eqref{eq:app-cov-hess}, differentiate the Tweedie formula  $m(\mathbf{y})=\mathbf{y}+\tilde\sigma^2\, \nabla_\mathbf{y}\log\tilde p(\mathbf{y})$ to get
$\nabla_\mathbf{y} m(\mathbf{y})=\mathbf{I}+\tilde\sigma^2\,\nabla_\mathbf{y}^2\log\tilde p(\mathbf{y})$ and multiply both sides by $\tilde\sigma^2$.
\end{proof}

\paragraph{DDPM scaling.}
For the forward sample  $\mathbf{x}_t=\sqrt{\bar\alpha_t}\,\mathbf{x}_0+\sigma_t\,\boldsymbol{\epsilon}$, define the rescaled variable  $\mathbf{y}_t=\mathbf{x}_t/\sqrt{\bar\alpha_t}$.
Then $\tilde\sigma_t=\sigma_t/\sqrt{\bar\alpha_t}$ and the same identities hold for the posterior
of $\mathbf{x}_0\mid \mathbf{x}_t$ after change of variables.

\subsection{Lipschitzness and contractivity of the Bayes denoiser}

The identity $C(\mathbf{y})=\tilde\sigma^2\, \nabla_\mathbf{y} m(\mathbf{y})$ immediately provides Lipschitz control of the posterior mean. Such covariance bounds hold, for example, under (strong) log-concavity of the prior via Brascamp–Lieb inequalities~\cite{brascamp1976bl}.

\begin{lemma}[Covariance bound implies Lipschitz denoiser]
\label{lem:app-lip-denoiser}
Let $\Omega \subset \mathbb{R}^d$ be a region where the posterior covariance satisfies $\|C(\mathbf{y})\|_{\mathrm{op}}\le \rho\,\tilde\sigma^2$ for all $\mathbf{y} \in \Omega$, then the posterior 
mean $m(.)$ satisfies $\|\nabla_\mathbf{y} m(\mathbf{y})\|_{\mathrm{op}}\le \rho$ for all $\mathbf{y}\in\Omega$ and is therefore $\rho$-Lipschitz on $\Omega$.
\end{lemma}

\begin{proof}
Using the identity $C(\mathbf{y})=\tilde\sigma^2\,\nabla_\mathbf{y} m(\mathbf{y})$, and taking operator norms, we have
$\|\nabla_\mathbf{y} m(\mathbf{y})\|_{\mathrm{op}}=\|C(\mathbf{y})\|_{\mathrm{op}}/\tilde\sigma^2\le \rho$, which establishes the Lipschitz bound.
\end{proof}

For directional contraction—used in the cross-backbone “normal-to-manifold” argument—we isolate a normal direction $\mathbf{n}$ and assume contraction along that direction.

\begin{assumption}[Directional contraction of the denoiser]~\cite{dalalyan2017sampling,durmus2019ula}
\label{ass:app-dir-contract}
There exists $\kappa\in(0,1]$ such that, for all $\mathbf{y},\mathbf{y}'$ in the tube,
\begin{equation}
\langle m(\mathbf{y})-m(\mathbf{y}'),\,\mathbf{y}-\mathbf{y}'\rangle \le (1-\kappa)\,\|\mathbf{y}-\mathbf{y}'\|_2^2
\qquad\text{whenever }(\mathbf{y}-\mathbf{y}')\parallel \mathbf{n}.
\label{eq:app-dir-contract}
\end{equation}
\end{assumption}

A sufficient condition is (locally, a.e.) a bound on the directional derivative along $\mathbf{n}$ in the tube: 
$\langle \mathbf{n},(\nabla_\mathbf{y} m(\mathbf{y})) \,\mathbf{n}\rangle\le 1-\kappa$. Using the covariance–Jacobian identity \eqref{eq:app-cov-jac}, this is equivalent to
$\langle \mathbf{n},C(\mathbf{y})\mathbf{n}\rangle\le (1-\kappa)\, \tilde\sigma^2$.

\subsection{From denoiser contraction to directional growth of the score}
This is the key step used to justify the main-text condition \eqref{eq-dir}. We work in additive coordinates, $\mathbf{y}=\mathbf{x}_0+\tilde\sigma\,\boldsymbol{\epsilon}$. From the Tweedie formula, $\tilde{\mathbf{s}}(\mathbf{y}) \coloneqq \nabla_\mathbf{y}\log\tilde p(\mathbf{y})$, 
Let $\mathbf{y}'$ denote a projection point~\cite{alain2014regularizedAEs} (e.g., $\mathbf{y}' = \pi(\mathbf{y}))$ and define $\mathbf{v} = \mathbf{y} - \mathbf{y}'$.

\begin{lemma}[Directional growth of the ideal score]
\label{lem:app-dir-score-star}
Assume \eqref{eq:app-dir-contract} holds for $\mathbf{y},\mathbf{y}'$ with $\mathbf{v}\parallel \mathbf{n}$.
Then
\begin{equation}
\Big\langle \tilde{\mathbf{s}}(\mathbf{y})-\tilde{\mathbf{s}}(\mathbf{y}'),\,\frac{\mathbf{v}}{\|\mathbf{v}\|}\Big\rangle
\le -\frac{\kappa}{\tilde\sigma^2}\,\|\mathbf{v}\|.
\label{eq:app-dir-score-star}
\end{equation}
\end{lemma}

\begin{proof}
Using $\tilde{\mathbf{s}}(\mathbf{y})=(m(\mathbf{y})-\mathbf{y})/\tilde\sigma^2$, we have
\[
\tilde{\mathbf{s}}(\mathbf{y})-\tilde{\mathbf{s}}(\mathbf{y}')
=\frac{(m(\mathbf{y})-m(\mathbf{y}'))-(\mathbf{y}-\mathbf{y}')}{\tilde\sigma^2}\, .
\]
Taking the inner product with $\mathbf{v} = \mathbf{y}-\mathbf{y}'$ gives:
\[
\langle \tilde{\mathbf{s}}(\mathbf{y})-\tilde{\mathbf{s}}(\mathbf{y}'),\,\mathbf{v}\rangle
=\frac{\langle m(\mathbf{y})-m(\mathbf{y}'),\mathbf{v}\rangle-\|\mathbf{v}\|^2}{\tilde\sigma^2}
\le -\frac{\kappa}{\tilde\sigma^2}\,\|\mathbf{v}\|^2\, ,
\]
by Assumption~\ref{ass:app-dir-contract}. Dividing both sides by $\|\mathbf{v}\|$ yields \eqref{eq:app-dir-score-star}.
\end{proof}

\subsection{From $s$ to $\mathbf{s}_\theta$ (approximation on a tube)}
Let $\mathbf{s}_\theta$ be a learned score that approximates the source score on a tube:
\[
\sup_{\mathbf{y}\in\Omega}\|\mathbf{s}_\theta(\mathbf{y})-\tilde{\mathbf{s}}(\mathbf{y})\|\le \delta\, .
\]
Then the directional inequality transfers with a slack.

\begin{lemma}[Directional growth for $\mathbf{s}_\theta$]
\label{lem:app-dir-score-theta}
Under the above uniform approximation, for $\mathbf{v}=\mathbf{y}-\mathbf{y}'$,
\begin{equation}
\Big\langle \mathbf{s}_\theta(\mathbf{y})-\mathbf{s}_\theta(\mathbf{y}'),\,\frac{\mathbf{v}}{\|\mathbf{v}\|}\Big\rangle
\le -\frac{\kappa}{\tilde\sigma^2}\|\mathbf{v}\| + 2\delta\, .
\label{eq:app-dir-score-theta-raw}
\end{equation}
In particular, if $\|\mathbf{v}\|\ge \displaystyle\frac{4\tilde\sigma^2}{\kappa}\delta$, then
\begin{equation}
\Big\langle \mathbf{s}_\theta(\mathbf{y})-\mathbf{s}_\theta(\mathbf{y}'),\,\frac{\mathbf{v}}{\|\mathbf{v}\|}\Big\rangle
\le -\underline m\,\|\mathbf{v}\|,
\qquad \underline m\coloneqq \frac{\kappa}{2\tilde\sigma^2}\, .
\label{eq:app-dir-score-theta}
\end{equation}
\end{lemma}

\begin{proof}
Decompose $\mathbf{s}_\theta(\mathbf{y})=\tilde{\mathbf{s}}(\mathbf{y})+\boldsymbol{\xi}(\mathbf{y})$ with $\|\boldsymbol{\xi}(\mathbf{y})\|\le\delta$:
\[
\langle \mathbf{s}_\theta(\mathbf{y})-\mathbf{s}_\theta(\mathbf{y}'),\mathbf{v}/\|\mathbf{v}\|\rangle
=
\langle \tilde{\mathbf{s}}(\mathbf{y})-\tilde{\mathbf{s}}(\mathbf{y}'),\mathbf{v}/\|\mathbf{v}\|\rangle
+\langle \boldsymbol{\xi}(\mathbf{y})-\boldsymbol{\xi}(\mathbf{y}'),\mathbf{v}/\|\mathbf{v}\|\rangle\, .
\]
By Lemma~\ref{lem:app-dir-score-star} and $|\langle \xi(\mathbf{y})-\xi(\mathbf{y}'),\cdot\rangle|\le \|\xi(\mathbf{y})\|+\|\xi(\mathbf{y}')\|\le 2\delta$,
we get \eqref{eq:app-dir-score-theta-raw}. If $\|\mathbf{v}\|\ge 4\tilde\sigma^2\delta/\kappa$ then $2\delta\le (\kappa/(2\tilde\sigma^2))\,\|\mathbf{v}\|$
and \eqref{eq:app-dir-score-theta} follows.
\end{proof}

\paragraph{Connection to the main-text condition \eqref{eq-dir}.}
In the main text, a projection $\pi_t$ onto a source manifold $\mathcal{M}_t$ is defined in $\mathbf{x}_t$-space.
Applying the previous derivation in the rescaled additive coordinates  $\mathbf{y}_t=\mathbf{x}_t/\sqrt{\bar\alpha_t}$ yields \eqref{eq-dir} with explicit definitions of $m_t$ and $d_{0,t}$ up to the scaling $\tilde\sigma_t=\sigma_t/\sqrt{\bar\alpha_t}$.

%%%%%%%%%%%%%%%%%%%%%%%%%%%%%%%%%%%%%%%%%%%%%%%%%%%%%%%%%%%%%%%%%%%%%%
\section{GEPC theory: detailed proofs and cross-backbone geometry}
\label{app:gepc-theory}
%%%%%%%%%%%%%%%%%%%%%%%%%%%%%%%%%%%%%%%%%%%%%%%%%%%%%%%%%%%%%%%%%%%%%%

\subsection{Invariance of a distribution and score equivariance}
\label{app:inv_eq}

We work with a finite group $\mathcal{G}$ acting on $\mathbb{R}^d$ via orthogonal matrices $\mathcal{P}_g$,
so $\mathcal{P}_g^{-1}=\mathcal{P}_g^\top$ and $|\det \mathcal{P}_g|=1$.

\begin{lemma}[Invariance $\Leftrightarrow$ score equivariance]
\label{lem:inv_eq}
Let $p$ be a positive $C^1$ density on $\mathbb{R}^d$ with score $\mathbf{s}_p(\mathbf{x})=\nabla_{\mathbf{x}}\log p(\mathbf{x})$.
Then the following are equivalent:
\begin{enumerate}
\item[(i)] $p(\mathcal{P}_g\mathbf{x})=p(\mathbf{x})$ for all $g\in\mathcal{G}$ and all $\mathbf{x}\in\mathbb{R}^d$;
\item[(ii)] $\mathbf{s}_p(\mathcal{P}_g\mathbf{x})=\mathcal{P}_g\mathbf{s}_p(\mathbf{x})$ for all $g\in\mathcal{G}$ and all $\mathbf{x}\in\mathbb{R}^d$.
\end{enumerate}
\end{lemma}

\begin{proof}
(i)$\Rightarrow$(ii). If $p(\mathcal{P}_g\mathbf{x})=p(\mathbf{x})$, then $\log p(\mathcal{P}_g\mathbf{x})=\log p(\mathbf{x})$.
Differentiating w.r.t.\ $\mathbf{x}$ and using the chain rule gives
\[
\nabla_{\mathbf{x}}\log p(\mathcal{P}_g\mathbf{x})
=\mathcal{P}_g^\top \nabla_{\mathbf{y}}\log p(\mathbf{y})\big|_{\mathbf{y}=\mathcal{P}_g\mathbf{x}}
=\mathcal{P}_g^\top \mathbf{s}_p(\mathcal{P}_g\mathbf{x}).
\]
The left-hand side equals $\nabla_{\mathbf{x}}\log p(\mathbf{x})=\mathbf{s}_p(\mathbf{x})$, hence
$\mathbf{s}_p(\mathbf{x})=\mathcal{P}_g^\top \mathbf{s}_p(\mathcal{P}_g\mathbf{x})$, i.e.\ $\mathbf{s}_p(\mathcal{P}_g\mathbf{x})=\mathcal{P}_g\mathbf{s}_p(\mathbf{x})$.

(ii)$\Rightarrow$(i). Assume $\mathbf{s}_p(\mathcal{P}_g\mathbf{x})=\mathcal{P}_g\mathbf{s}_p(\mathbf{x})$.
Define $h_g(\mathbf{x})\coloneqq \log p(\mathcal{P}_g\mathbf{x})-\log p(\mathbf{x})$.
Then
\[
\nabla_{\mathbf{x}} h_g(\mathbf{x})
=\mathcal{P}_g^\top \mathbf{s}_p(\mathcal{P}_g\mathbf{x})-\mathbf{s}_p(\mathbf{x})
=\mathcal{P}_g^\top \mathcal{P}_g\mathbf{s}_p(\mathbf{x})-\mathbf{s}_p(\mathbf{x})
=\mathbf{0},
\]
so $h_g(\mathbf{x})$ is constant in $\mathbf{x}$: $h_g(\mathbf{x})=c_g$. Hence
$p(\mathcal{P}_g\mathbf{x})=e^{c_g}p(\mathbf{x})$.
Integrating both sides over $\mathbb{R}^d$ and using $|\det \mathcal{P}_g|=1$ yields
$1=\displaystyle\int p(\mathcal{P}_g\mathbf{x})d\mathbf{x}=e^{c_g}\int p(\mathbf{x})d\mathbf{x}=e^{c_g}$, so $c_g=0$ and $p(\mathcal{P}_g\mathbf{x})=p(\mathbf{x})$.
\end{proof}

\subsection{Residual decomposition and expectation bounds}
\label{app:gepc-bounds}

Recall the residual operator from \eqref{eq:delta-def}:
\[
\Delta_g f(\mathbf{x},t)\coloneqq \mathcal{P}_g^{-1} f(\mathcal{P}_g\mathbf{x},t)-f(\mathbf{x},t).
\]
For orthogonal transforms, $\mathcal{P}_g^{-1}=\mathcal{P}_g^\top$ and $\|\mathcal{P}_g^{-1}\mathbf{v}\|_2=\|\mathbf{v}\|_2$.
For a backbone score $\mathbf{s}_\theta(\cdot,t)$ we define
\begin{equation}
R_t(\mathbf{x},g)\coloneqq \|\Delta_g \mathbf{s}_\theta(\mathbf{x},t)\|_2^2.
\label{eq:def-Rt}
\end{equation}

Fix any absolutely continuous test marginal $p_t$ with score
$\mathbf{s}_{p_t}(\mathbf{x})\coloneqq \nabla_{\mathbf{x}}\log p_t(\mathbf{x})$
and define the score error
$\mathbf{e}_{p_t}(\mathbf{x},t)\coloneqq \mathbf{s}_\theta(\mathbf{x},t)-\mathbf{s}_{p_t}(\mathbf{x})$.
Then for all $\mathbf{x},g$,
\begin{equation}
\Delta_g \mathbf{s}_\theta(\mathbf{x},t)
=
\Delta_g \mathbf{s}_{p_t}(\mathbf{x},t)
+
\Delta_g \mathbf{e}_{p_t}(\mathbf{x},t),
\qquad
R_t(\mathbf{x},g)
=
\|\Delta_g \mathbf{s}_{p_t}(\mathbf{x},t)+\Delta_g \mathbf{e}_{p_t}(\mathbf{x},t)\|_2^2.
\label{eq:app-decomp}
\end{equation}

We also recall
\begin{equation}
\mathcal{B}^{(\mathcal{G})}(p_t)\coloneqq
\mathbb{E}_{\mathbf{x}\sim p_t,\,g\sim\nu_\mathcal{G}}
\big[\|\Delta_g \mathbf{s}_{p_t}(\mathbf{x},t)\|_2^2\big],
\label{eq:app-symbreak}
\end{equation}
and define
\begin{equation}
\Delta_E(p_t,t)
\;\coloneqq\;
\mathbb{E}_{\mathbf{x}\sim p_t,\,g\sim\nu_\mathcal{G}}
\Big[
\|\mathbf{e}_{p_t}(\mathcal{P}_g\mathbf{x},t) - \mathbf{e}_{p_t} (\mathbf{x},t)\|_2^2
\Big].
\label{eq:app-DeltaE}
\end{equation}

\paragraph{Proof of Proposition~\ref{prop-bounds}.}
Expanding the squared norm in \eqref{eq:app-decomp} gives
\begin{align}
R_t(\mathbf{x},g)
&=
\|\Delta_g \mathbf{s}_{p_t}(\mathbf{x},t)\|_2^2
+
\|\Delta_g \mathbf{e}_{p_t}(\mathbf{x},t)\|_2^2
+
2\big\langle \Delta_g \mathbf{s}_{p_t}(\mathbf{x},t),\,\Delta_g \mathbf{e}_{p_t}(\mathbf{x},t)\big\rangle.
\label{eq:app-expand}
\end{align}

\emph{Upper bound.}
Using Cauchy--Schwarz and inequality $2\langle \mathbf{a},\mathbf{b}\rangle\le \|\mathbf{a}\|_2^2+\|\mathbf{b}\|_2^2$, for any vectors $\mathbf{a}$ and $\mathbf{b}$,  in \eqref{eq:app-expand} leads to
\[
R_t(\mathbf{x},g)
\le
2\|\Delta_g \mathbf{s}_{p_t}(\mathbf{x},t)\|_2^2
+
2\|\Delta_g \mathbf{e}_{p_t}(\mathbf{x},t)\|_2^2.
\]
Taking expectation over $\mathbf{x}\sim p_t$ and $g\sim\nu_\mathcal{G}$ yields
\[
\mathbb{E}[R_t(\mathbf{x},g)]
\le
2\mathcal{B}^{(\mathcal{G})}(p_t)
+
2\,\mathbb{E}\|\Delta_g \mathbf{e}_{p_t}(\mathbf{x},t)\|_2^2.
\]
Finally, since
$\Delta_g \mathbf{e}_{p_t}(\mathbf{x},t)=\mathcal{P}_g^{-1}\mathbf{e}_{p_t}(\mathcal{P}_g\mathbf{x},t)-\mathbf{e}_{p_t}(\mathbf{x},t)$
and $\|\mathbf{u}-\mathbf{v}\|_2^2\le 2\|\mathbf{u}\|_2^2+2\|\mathbf{v}\|_2^2$,
\[
\|\Delta_g \mathbf{e}_{p_t}(\mathbf{x},t)\|_2^2
\le
2\|\mathbf{e}_{p_t}(\mathcal{P}_g\mathbf{x},t)\|_2^2
+
2\|\mathbf{e}_{p_t}(\mathbf{x},t)\|_2^2,
\]
which gives the stated $u_b(p_t)$ in Proposition~\ref{prop-bounds}.

\emph{Lower bound.}
From \eqref{eq:app-expand} and Cauchy--Schwarz,
\[
R_t(\mathbf{x},g)
\ge
\|\Delta_g \mathbf{s}_{p_t}(\mathbf{x},t)\|_2^2
+
\|\Delta_g \mathbf{e}_{p_t}(\mathbf{x},t)\|_2^2
-
2\|\Delta_g \mathbf{s}_{p_t}(\mathbf{x},t)\|_2\,\|\Delta_g \mathbf{e}_{p_t}(\mathbf{x},t)\|_2.
\]
Taking expectation and applying Cauchy--Schwarz to the cross term yields
\[
\mathbb{E}[R_t(\mathbf{x}, g)]
\ge
\mathcal{B}^{(\mathcal{G})}(p_t)
+
\mathbb{E}\|\Delta_g \mathbf{e}_{p_t}(\mathbf{x},t)\|_2^2
-
2\sqrt{\mathcal{B}^{(\mathcal{G})}(p_t)}\,
\sqrt{\mathbb{E}\|\Delta_g \mathbf{e}_{p_t}(\mathbf{x},t)\|_2^2}.
\]
Noting that $\Delta_g \mathbf{e}_{p_t}(\mathbf{x},t)=\mathcal{P}_g^{-1}\mathbf{e}_{p_t}(\mathcal{P}_g\mathbf{x},t)-\mathbf{e}_{p_t}(\mathbf{x},t)$
and $\|\mathcal{P}_g^{-1}\mathbf{v}\|_2=\|\mathbf{v}\|_2$, we obtain
$\mathbb{E}\|\Delta_g \mathbf{e}_{p_t}(\mathbf{x},t)\|_2^2=\Delta_E(p_t,t)$, which yields the lower bound of proposition ~\ref{prop-bounds}.
\qed

\subsection{Cross-backbone bounds: proof of Proposition~\ref{prop-cross-bounds}}
\label{app:cross-backbone-proof}

Fix $t$ and consider $\mathbf{x}\in\mathcal{N}_t$. Let $\mathbf{z}=\pi_t(\mathbf{x})\in\mathcal{M}_t$, so that
$d_t(\mathbf{x})=\|\mathbf{x}-\mathbf{z}\|_2$.
Assume $\pi_t$ commutes with the group action: $\pi_t(\mathcal{P}_g\mathbf{x})=\mathcal{P}_g\mathbf{z}$, and $\mathcal{P}_g$ is orthogonal,
so $\|\mathcal{P}_g\mathbf{x}-\mathcal{P}_g\mathbf{z}\|_2=\|\mathbf{x}-\mathbf{z}\|_2=d_t(\mathbf{x})$.

Define the off-manifold deviation $\boldsymbol{\delta}(\mathbf{x})\coloneqq \mathbf{s}_\theta(\mathbf{x},t)-\mathbf{s}_\theta(\mathbf{z},t)$.
By Lipschitzness \eqref{eq-lip}, $\|\boldsymbol{\delta}(\mathbf{x})\|_2\le L_t d_t(\mathbf{x})$ and
$\|\boldsymbol{\delta}(\mathcal{P}_g\mathbf{x})\|_2\le L_t d_t(\mathbf{x})$.

\paragraph{Upper bound \eqref{eq:cross-upper}.}
Using add-and-subtract around $\mathbf{z}$ and $\mathcal{P}_g\mathbf{z}$:
\begin{align*}
\Delta_g \mathbf{s}_\theta(\mathbf{x},t)
&=\mathcal{P}_g^\top \mathbf{s}_\theta(\mathcal{P}_g\mathbf{x},t)-\mathbf{s}_\theta(\mathbf{x},t) \\
&=\underbrace{\big(\mathcal{P}_g^\top \mathbf{s}_\theta(\mathcal{P}_g\mathbf{z},t)-\mathbf{s}_\theta(\mathbf{z},t)\big)}_{\Delta_g \mathbf{s}_\theta(\mathbf{z},t)}
+\underbrace{\mathcal{P}_g^\top\big(\mathbf{s}_\theta(\mathcal{P}_g\mathbf{x},t)-\mathbf{s}_\theta(\mathcal{P}_g\mathbf{z},t)\big)
-\big(\mathbf{s}_\theta(\mathbf{x},t)-\mathbf{s}_\theta(\mathbf{z},t)\big)}_{\mathbf{b}_g(\mathbf{x})}.
\end{align*}
Thus,
$
R_t(\mathbf{x},g)=\|\Delta_g \mathbf{s}_\theta(\mathbf{z},t)+\mathbf{b}_g(\mathbf{x})\|_2^2.
$
Using the inequality $\|\mathbf{a}+\mathbf{b}\|_2^2\le 2\|\mathbf{a}\|_2^2+2\|\mathbf{b}\|_2^2$ for any vectors $\mathbf{a}$ and $\mathbf{b}$ gives
\[
R_t(\mathbf{x},g)\le 2R_t(\mathbf{z},g)+2\|\mathbf{b}_g(\mathbf{x})\|_2^2.
\]
Moreover, by the triangle inequality and orthogonality of $\mathcal{P}_g$,
\[
\|\mathbf{b}_g(\mathbf{x})\|_2
\le
\|\mathbf{s}_\theta(\mathcal{P}_g\mathbf{x},t)-\mathbf{s}_\theta(\mathcal{P}_g\mathbf{z},t)\|_2
+
\|\mathbf{s}_\theta(\mathbf{x},t)-\mathbf{s}_\theta(\mathbf{z},t)\|_2
\le
2L_t d_t(\mathbf{x}),
\]
so $\|\mathbf{b}_g(\mathbf{x})\|_2^2\le 4L_t^2 d_t(\mathbf{x})^2$ and therefore
\[
R_t(\mathbf{x},g)\le 2R_t(\mathbf{z},g)+8L_t^2 d_t(\mathbf{x})^2.
\]
Taking expectation over $g\sim\nu_{\mathcal{G}}$ yields \eqref{eq:cross-upper}.

\paragraph{Lower bound \eqref{eq-cross-lower}.}
Let $\mathbf{z}=\pi_t(\mathbf{x})$. By the reverse triangle inequality,
\[
\|\Delta_g \mathbf{s}_\theta(\mathbf{x},t)\|_2
=
\|\mathcal{P}_g^\top \mathbf{s}_\theta(\mathcal{P}_g\mathbf{x},t)-\mathbf{s}_\theta(\mathbf{x},t)\|_2
\ge
\|\mathbf{s}_\theta(\mathbf{x},t)-\mathbf{s}_\theta(\mathbf{z},t)\|_2
-
\|\mathcal{P}_g^\top \mathbf{s}_\theta(\mathcal{P}_g\mathbf{x},t)-\mathbf{s}_\theta(\mathbf{z},t)\|_2.
\]
Using $\|\mathcal{P}_g^\top \mathbf{u}-\mathbf{v}\|_2=\|\mathbf{u}-\mathcal{P}_g\mathbf{v}\|_2$ and adding/subtracting $\mathbf{s}_\theta(\mathcal{P}_g\mathbf{z},t)$,
\[
\|\mathcal{P}_g^\top \mathbf{s}_\theta(\mathcal{P}_g\mathbf{x},t)-\mathbf{s}_\theta(\mathbf{z},t)\|_2
=
\|\mathbf{s}_\theta(\mathcal{P}_g\mathbf{x},t)-\mathcal{P}_g\mathbf{s}_\theta(\mathbf{z},t)\|_2
\le
\|\mathbf{s}_\theta(\mathcal{P}_g\mathbf{x},t)-\mathbf{s}_\theta(\mathcal{P}_g\mathbf{z},t)\|_2
+
\|\Delta_g \mathbf{s}_\theta(\mathbf{z},t)\|_2
\le
L_t d_t(\mathbf{x})+\|\Delta_g \mathbf{s}_\theta(\mathbf{z},t)\|_2.
\]
Hence,
\[
\|\Delta_g \mathbf{s}_\theta(\mathbf{x},t)\|_2
\ge
\|\mathbf{s}_\theta(\mathbf{x},t)-\mathbf{s}_\theta(\mathbf{z},t)\|_2
-
L_t d_t(\mathbf{x})
-
\|\Delta_g \mathbf{s}_\theta(\mathbf{z},t)\|_2.
\]
If $d_t(\mathbf{x})\ge d_{0,t}$ and \eqref{eq-dir} holds, then
$\|\mathbf{s}_\theta(\mathbf{x},t)-\mathbf{s}_\theta(\mathbf{z},t)\|_2\ge m_t d_t(\mathbf{x})$ (since $\|v\|\ge |\langle v,u\rangle|$ for unit $u$).
Thus,
\[
\|\Delta_g \mathbf{s}_\theta(\mathbf{x},t)\|_2
\ge
(m_t-L_t)\,d_t(\mathbf{x}) - \|\Delta_g \mathbf{s}_\theta(\mathbf{z},t)\|_2.
\]
Let $a\coloneqq (m_t-L_t)\,d_t(\mathbf{x})$. %and $B_g\coloneqq \|\Delta_g \mathbf{s}_\theta(\mathbf{z},t)\|_2$.
Then $R_t(\mathbf{x},g)\ge \left(a-\|\Delta_g \mathbf{s}_\theta(\mathbf{z},t)\|_2\right)^2$, %where $(\cdot)_+=\max(\cdot,0)$.
define $\varphi(y)\coloneqq (a-y)^2$, which is convex; by Jensen's inequality \cite{boyd2004convex},
\[
\mathbb{E}_g R_t(\mathbf{x},g)
\ge
\mathbb{E}_g \varphi(\|\Delta_g \mathbf{s}_\theta(\mathbf{z},t)\|_2)
\ge
\varphi(\mathbb{E}_g \|\Delta_g \mathbf{s}_\theta(\mathbf{z},t)\|_2)
=
\big(a-\mathbb{E}_g \|\Delta_g \mathbf{s}_\theta(\mathbf{z},t)\|_2\big)^2.
\]
Finally, $\mathbb{E}_g \|\Delta_g \mathbf{s}_\theta(\mathbf{z},t)\|_2\le \sqrt{\mathbb{E}_g \|\Delta_g \mathbf{s}_\theta(\mathbf{z},t)\|_2^2}=\sqrt{\mathbb{E}_g R_t(\mathbf{z},g)}$ by Cauchy--Schwarz, yielding
\[
\mathbb{E}_g R_t(\mathbf{x},g)
\ge
\Big(a-\sqrt{\mathbb{E}_g R_t(\mathbf{z},g)}\Big)^2
=
\Big(\,(m_t-L_t)\,d_t(\mathbf{x}) - \sqrt{\mathbb{E}_g R_t(\pi_t(\mathbf{x}),g)}\,\Big)^2,
\]
% which is \eqref{eq-cross-lower}.
which proves the claimed lower bound in Proposition~\ref{prop-cross-bounds}.
\qed

%%%%%%%%%%%%%%%%%%%%%%%%%%%%%%%%%%%%%%%%%%%%%%%%%%%%%%%%%%%%%%%%%%%%%%
\section{Gaussian sanity checks (mean shift and $90^\circ$ rotation)}
\label{app:gaussian-sanity}
%%%%%%%%%%%%%%%%%%%%%%%%%%%%%%%%%%%%%%%%%%%%%%%%%%%%%%%%%%%%%%%%%%%%%%

We provide closed-form computations of the \emph{ideal} GEPC residual for a simple Gaussian, illustrating that GEPC captures equivariance-breaking information even when the score magnitude remains insensitive.

\subsection{Mean shift with $\mathcal{G}=\{\mathcal{I}d,-\mathcal{I}d\}$}

Let $p=\mathcal N(\boldsymbol{\mu},\sigma^2 \mathbf{I})$. Then
$s(\mathbf{x})=\nabla_{\mathbf{x}}\log p(\mathbf{x})=-(\mathbf{x}-\boldsymbol{\mu})/\sigma^2$.
For $\mathcal{G}=\{\mathcal{I}d,-\mathcal{I}d\}$, take $\mathcal{P}_{-\mathcal{I}d}=-\mathcal{I}d$.
For $g=-\mathcal{I}d$ (orthogonal, hence $\mathcal{P}_g^{-1}=\mathcal{P}_g^\top$):
\[
\Delta_g s(\mathbf{x})
=\mathcal{P}_g^{-1}s(\mathcal{P}_g\mathbf{x})-s(\mathbf{x})
=(-\mathcal{I})\,s(-\mathbf{x})-s(\mathbf{x})
= -\frac{2}{\sigma^2}\,\boldsymbol{\mu}.
\]
Hence
\[
R(\mathbf{x},g)\;=\;\left\|-\frac{2}{\sigma^2}\, \boldsymbol{\mu}\right\|_2^2
\;=\;\frac{4}{\sigma^4}\,\|\boldsymbol{\mu}\|_2^2,
\qquad
\mathbb{E}_{g\sim\nu_\mathcal{G}}\,R(\mathbf{x},g)\;=\;\frac{2}{\sigma^4}\,\|\boldsymbol{\mu}\|_2^2,
\]

since the $g=\mathcal{I}d$ term is $0$ and $\nu_\mathcal{G}$ is uniform.
Meanwhile,
$\E_{\mathbf{x}\sim u}\|s(\mathbf{x})\|_2^2=d/\sigma^2$ is independent of $\boldsymbol{\mu}$.
Thus, GEPC separates mean-shifts invisible to the score magnitude.

\subsection{Anisotropic covariance with $90^\circ$ rotations ($\mathcal{G}=C_4$)}

For $d=2$, let $\mathbf{p}=\mathcal N(\mathbf{0},\boldsymbol{\Sigma})$ and $\boldsymbol{\Sigma}=\mathrm{diag}(\sigma_1^2,\sigma_2^2)$. Then $s(\mathbf{x})=-\boldsymbol{\Sigma}^{-1} \, \mathbf{x}$.

Let $\mathcal{G}=C_4=\{\mathcal{I}d,\mathcal{R},\mathcal{R}^2,\mathcal{R}^3\}$ where
\[
\mathcal{R}=\begin{pmatrix}0&-1\\1&0\end{pmatrix},
\quad \mathcal{R}^2=-\mathcal{I}d,\quad \mathcal{R}^3=\mathcal{R}^\top.
\]
For any $g\in \mathcal{G}$,
\[
\Delta_g s(\mathbf{x})
=\mathcal{P}_g^{-1} s(\mathcal{P}_g \mathbf{x})-s(\mathbf{x})
= -\big(\mathcal{P}_g^\top\boldsymbol{\Sigma}^{-1}\mathcal{P}_g-\boldsymbol{\Sigma}^{-1}\big)\mathbf{x}.
\]
Since $C_4$ is orthogonal, $\mathcal{P}_g^{-1}=\mathcal{P}_g^\top$.

\paragraph{Compute $\mathcal{P}_g^\top\boldsymbol{\Sigma}^{-1}\mathcal{P}_g$.}
For $g=\mathcal{R}$,
\[
\mathcal{R}^\top\boldsymbol{\Sigma}^{-1}\mathcal{R}
=
\begin{pmatrix}0&1\\-1&0\end{pmatrix}
\begin{pmatrix}\sigma_1^{-2}&0\\0&\sigma_2^{-2}\end{pmatrix}
\begin{pmatrix}0&-1\\1&0\end{pmatrix}
=
\begin{pmatrix}\sigma_2^{-2}&0\\0&\sigma_1^{-2}\end{pmatrix}
\]
Therefore
\[
\mathcal{R}^\top\boldsymbol{\Sigma}^{-1}\mathcal{R}-\boldsymbol{\Sigma}^{-1}
=
\begin{pmatrix}\sigma_2^{-2}-\sigma_1^{-2}&0\\0&\sigma_1^{-2}-\sigma_2^{-2}\end{pmatrix}
=
\left(\sigma_2^{-2}-\sigma_1^{-2}\right)\,\begin{pmatrix}1&0\\0&-1\end{pmatrix}\,.
\]
Hence
\[
\Delta_\mathcal{R} s(\mathbf{x})
=-\left(\sigma_2^{-2}-\sigma_1^{-2}\right)\,\begin{pmatrix}1&0\\0&-1\end{pmatrix}\,\mathbf{x}\, ,
\]
and since $\left\|\begin{pmatrix}1&0\\0&-1\end{pmatrix}\, \mathbf{x}\right\|^2=\mathbf{x}_1^2+\mathbf{x}_2^2$,
\begin{equation}
R(\mathbf{x},\mathcal{R})
\;\coloneqq\;\|\Delta_{\mathcal{R}} s(\mathbf{x})\|_2^2
\;=\;\left(\sigma_2^{-2}-\sigma_1^{-2}\right)^2\left(\mathbf{x}_1^2+\mathbf{x}_2^2\right)\, .
\label{eq:rot-residual-pointwise}
\end{equation}

For $g=\mathcal{R}^3$, the same computation gives the same residual.
For $g=\mathcal{R}^2=-\mathcal{I}d$, we have $(-\mathcal{I}d)^\top \boldsymbol{\Sigma}^{-1} (-\mathcal{I}d)=\boldsymbol{\Sigma}^{-1}$,
hence $R(\mathbf{x},\mathcal{R}^2)=0$.
Also $R(\mathbf{x},\mathcal{I}d)=0$.

\paragraph{Expectation under $\mathbf{x}\sim\mathcal N(\mathbf{0},\boldsymbol{\Sigma})$.}
We have $\E\left[\mathbf{x}_1^2+\mathbf{x}_2^2\right]=\mathrm{tr}(\boldsymbol{\Sigma})=\sigma_1^2+\sigma_2^2$.
Thus from \eqref{eq:rot-residual-pointwise},
\[
\mathbb{E}_{\mathbf{x}\sim p}\, R(\mathbf{x},\mathcal{R})
=
\left(\sigma_2^{-2}-\sigma_1^{-2}\right)^2\left(\sigma_1^2+\sigma_2^2\right)\,,
\]
and averaging over $g\sim \nu_\mathcal{G}$ (uniform over four elements) yields
\begin{equation}
\E_{\mathbf{x}\sim p,\,g\sim\nu_\mathcal{G}} \, R(\mathbf{x},g)
=
\frac{1}{2}\,\left(\sigma_2^{-2}-\sigma_1^{-2}\right)^2\left(\sigma_1^2+\sigma_2^2\right)\, ,
\label{eq:rot-residual-exp}
\end{equation}
since only $\mathcal{R}$ and $\mathcal{R}^3$ contribute.
This quantity is zero iff $\sigma_1=\sigma_2$ (isotropy), i.e.\ iff the Gaussian is rotation-invariant.
Hence, GEPC detects anisotropy relative to the $90^\circ$ rotation group.

%%%%%%%%%%%%%%%%%%%%%%%%%%%%%%%%%%%%%%%%%%%%%%%%%%%%%%%%%%%%
\section{Experimental details and reproducibility}
\label{app:exp-details}
%%%%%%%%%%%%%%%%%%%%%%%%%%%%%%%%%%%%%%%%%%%%%%%%%%%%%%%%%%%%

\paragraph{Implementation.}
All methods are evaluated using the same pretrained diffusion checkpoints (CelebA-$32$ and LSUN-$256$) with no fine-tuning.
For GEPC, we follow the ID-only protocol of Section~\ref{sec:experiments}: ID-train is used for timestep selection, weighting, and density calibration; ID-test and OOD-test are used only for evaluation.

\paragraph{Hardware and software.}
Unless stated otherwise, experiments are run on a single GPU (NVIDIA GeForce RTX 4060 Laptop GPU) with PyTorch on Linux.

\paragraph{Determinism.}
We fix seeds for Python, NumPy, and PyTorch, disable TF32, and optionally enable PyTorch deterministic algorithms.
DataLoaders use an explicit \texttt{torch.Generator} with a fixed seed and \texttt{worker\_init\_fn} to ensure stable shuffling across workers.
We report exact command lines and YAML configs in the released code.

\paragraph{Compute accounting.}
We report compute as F+J, where $F$ is a forward evaluation of $\mathbf{s}_\theta(\cdot,t)$ and $J$ is a Jacobian--vector product counted as a forward-equivalent operation.
For methods using $T$ reverse diffusion steps, we report the corresponding number of sequential score evaluations.

%%%%%%%%%%%%%%%%%%%%%%%%%%%%%%%%%%%%%%%%%%%%%%%%%%%%%%%%%%%%
\section{GEPC feature variants and fusion}
\label{app:gepc-features}
%%%%%%%%%%%%%%%%%%%%%%%%%%%%%%%%%%%%%%%%%%%%%%%%%%%%%%%%%%%%
Let $\mathbf{x}_t\sim q(\mathbf{x}_t\mid \mathbf{x}_0)$. Define the transported score residual field
\begin{equation}
\mathbf r_t(\mathbf{x}_t,g)\coloneqq \mathcal{P}_g^{-1}\mathbf{s}_\theta(\mathcal{P}_g\mathbf{x}_t,t)-\mathbf{s}_\theta(\mathbf{x}_t,t)\in\mathbb{R}^{C\times h\times w},
\end{equation}
and the transported score in the canonical frame
\begin{equation}
\tilde{\mathbf{s}}_\theta(\mathbf{x}_t,t;g)\coloneqq \mathcal{P}_g^{-1}\mathbf{s}_\theta(\mathcal{P}_g\mathbf{x}_t,t)\in\mathbb{R}^{C\times h\times w},
\end{equation}
so that $\mathbf r_t(\mathbf{x}_t,g)=\tilde{\mathbf{s}}_\theta(\mathbf{x}_t,t;g)-\mathbf{s}_\theta(\mathbf{x}_t,t)$.
% We use $\mathrm{pool}(\|\mathbf{u}\|_2^2)\coloneqq \left\langle \tfrac{1}{C}\sum_c u_{c,h,w}^2\right\rangle_{h,w}$ with mean or top-$k$ pooling.
Throughout, $\mathrm{pool}(\cdot)$ denotes the following convention: for $A\in\mathbb{R}^{C\times h\times w}$ we first average over channels and then pool over spatial locations by either mean-pooling or top-$k$ pooling (top-$k$ averages the $k$ largest spatial responses). We apply this to pointwise energies, e.g. $\mathrm{pool}(\|\mathbf{u}\|_2^2)$.

\paragraph{Baseline normaliser.}
We use the pooled score energy
\begin{equation}
b_t(\mathbf{x}_0)\;\coloneqq\;\mathrm{pool}\!\Big(\|\mathbf{s}_\theta(\mathbf{x}_t,t)\|_2^2\Big).
\end{equation}

\paragraph{GEPC$_s$ (base-normalised residual energy).}
\begin{equation}
z^{(s)}_t(\mathbf{x}_0)\;\coloneqq\;\mathbb{E}_{g\sim\mathrm{Unif}(\mathcal{G})}\left[\frac{\mathrm{pool}\!\Big(\|\mathbf r_t(\mathbf{x}_t,g)\|_2^2\Big)}{b_t(\mathbf{x}_0)}\right].
\end{equation}

\paragraph{GEPC$_{\cos}$ (global cosine inconsistency).}
Let $\langle a,b\rangle$ denote the dot product after vectorising over $(c,h,w)$, and $\|a\|$ the corresponding Euclidean norm. We use
\begin{equation}
z^{(\cos)}_t(\mathbf{x}_0)\;\coloneqq\;\mathbb{E}_{g\sim\mathrm{Unif}(\mathcal{G})}\left[1-\frac{\langle \tilde{\mathbf{s}}_\theta(\mathbf{x}_t,t;g),\mathbf{s}_\theta(\mathbf{x}_t,t)\rangle}{\|\tilde{\mathbf{s}}_\theta(\mathbf{x}_t,t;g)\|\,\|\mathbf{s}_\theta(\mathbf{x}_t,t)\|}\right],
\end{equation}
which is scale-invariant and thus requires no additional base normalisation.

\paragraph{GEPC$_{\mathrm{pair}}$ (pairwise dispersion, base-normalised).}
We also use explicit pair enumeration:
\begin{equation}
z^{(\mathrm{pair})}_t(\mathbf{x}_0)\;\coloneqq\;\mathbb{E}_{g<g'}\left[\frac{\mathrm{pool}\!\Big(\|\tilde{\mathbf{s}}_\theta(\mathbf{x}_t,t;g)-\tilde{\mathbf{s}}_\theta(\mathbf{x}_t,t;g')\|_2^2\Big)}{b_t(\mathbf{x}_0)}\right].
\end{equation}

\paragraph{ID-only calibration and fusion.}
Let $\mathcal{F}=\{s,\cos,\mathrm{pair}\}$ denote the enabled feature set.
In the default scalar-density mode (\texttt{vector\_mode=none}), we fit an ID-only model per $(t,f)$ on ID-train:
(i) KDE (\texttt{density\_mode=kde}) provides $\log p_{t,f}(z)$,
(ii) z-score (\texttt{density\_mode=zscore}) provides $\ell_{t,f}(z)=-\tfrac12((z-\mu_{t,f})/\sigma_{t,f})^2$,
or (iii) raw (\texttt{density\_mode=none}) uses $z$ directly.
Within a timestep, we aggregate per-feature scores using \texttt{agg\_feat} (sum/mean), then aggregate across timesteps using \texttt{agg\_t} (default: inverse-CV weighted mean).
For KDE/z-score, the ID score is
\begin{equation}
L(\mathbf{x}_0)\;\coloneqq\;\sum_{t\in\mathcal{T}} w_t\,\mathrm{AggFeat}\big(\{\ell_{t,f}(z^{(f)}_t(\mathbf{x}_0))\}_{f\in\mathcal{F}}\big)\, ,
\end{equation}
and the final anomaly score is $S(\mathbf{x}_0)\coloneqq -L(\mathbf{x}_0)$ (OOD-high). In raw mode, we directly set $S(\mathbf{x}_0)$ to the corresponding aggregated one-sided statistic.

\paragraph{Vector MVN (optional).}
In \texttt{vector\_mode=mvn}, we fit a single Gaussian on the concatenated feature vector over all kept $(t,f)$ on ID-train and score with the corresponding Mahalanobis distance (OOD-high).
Importantly, all three features reuse the same score-network evaluations, so enabling multiple features does not change the NFE.

\paragraph{Timestep Selection Algorithm.} %Stability-Based Timestep Selection
\begin{algorithm}[t]
\caption{Stability-Based Timestep Selection (ID-only)}
\label{alg:tselect}
\begin{algorithmic}[1]
\STATE \textbf{Input:} ID-train set $\mathcal{X}$, candidate timesteps $\mathcal{T}_{\mathrm{cand}}$, integer $K$, small $\epsilon>0$
\STATE \textbf{Output:} selected timesteps $\mathcal{T}$ and weights $\{w_t\}$
\FOR{$t\in\mathcal{T}_{\mathrm{cand}}$}
  \STATE Compute scores $\{z_t(\mathbf{x})\}_{\mathbf{x}\in\mathcal{X}}$ (default: $z^{(s)}_t(\mathbf{x})$)
  \STATE $\mathrm{CV}(t)\leftarrow \mathrm{std}(z_t)/(|\mathrm{mean}(z_t)|)$
\ENDFOR
\STATE $\mathcal{T}\leftarrow$ the $K$ timesteps with smallest $\mathrm{CV}(t)$
\STATE $w_t \propto 1/(\mathrm{CV}(t))$ for $t\in\mathcal{T}$ and normalise $\displaystyle\sum_{t\in\mathcal{T}}w_t=1$
\STATE \textbf{return} $\mathcal{T}$ and $\{w_t\}$
\end{algorithmic}
\end{algorithm}

%%%%%%%%%%%%%%%%%%%%%%%%%%%%%%%%%%%%%%%%%%%%%%%%%%%%%%%%%%%%
\section{Additional ablations and runtime}
\label{app:timestep-ablation}
%%%%%%%%%%%%%%%%%%%%%%%%%%%%%%%%%%%%%%%%%%%%%%%%%%%%%%%%%%%%

This appendix reports comprehensive ablations for GEPC on the $32\times32$ setting.
Unless stated otherwise, ablations follow the default configuration in Section~\ref{sec:setup} and are reported for \textbf{all 9 ID/OOD pairs}.
For readability, we additionally provide representative plots for one pair (SVHN as ID, CIFAR-100 as OOD) in Figs.~\ref{fig:ablation_c10_svhn}--\ref{fig:hists_svhn_c100}.

\subsection{SNR-to-timestep mapping}
\label{app:snr-map}
For DDPM-style schedules, we use $\mathrm{SNR}(t)\coloneqq \bar\alpha_t/(1-\bar\alpha_t)$ and map each target SNR level
(\texttt{snr\_levels}) to the closest discrete index $t$ by nearest-neighbour matching on the precomputed schedule.
This yields a small candidate set $\mathcal{T}_{\mathrm{cand}}$.

\subsection{ID-only timestep selection and weighting}
\label{app:tselect}
For each $t\in\mathcal{T}_{\mathrm{cand}}$, we compute an ID-only stability score using the coefficient of variation
\[
\mathrm{CV}(t)=\frac{\mathrm{std}\!\left(z_t(\mathbf{x})\right)}{|\mathrm{mean}\!\left(z_t(\mathbf{x})\right)|}\, ,
\]
over ID-train samples (default: $z^{(s)}_t$).
We keep the $K$ most stable timesteps (lowest CV), yielding $\mathcal{T}$, and set
\[
w_t \propto \frac{1}{\mathrm{CV}(t)}\, ,
\]
(\texttt{weight\_t=inv\_cv}), normalised to sum to one.
We use \texttt{agg\_t=wmean} unless stated otherwise, and fix $K$ across datasets in the main table to keep compute comparable.

\noindent\textbf{Two per-$t$ diagnostics.}
We distinguish (i) a component-level diagnostic that reports AUROC of the raw transported gap at each single timestep (Figure~\ref{fig:ablation_c10_svhn}c), and (ii) the AUROC of the final GEPC score when evaluated using a \emph{single} timestep (stored alongside the $K$-sweep in Table~\ref{tab:tselect_sweep_allpairs}).
The former explains \emph{where} symmetry-breaking arises; the latter supports the ID-only selection rule.

\paragraph{Selected timesteps and weights (9 pairs).}
Table~\ref{tab:tselect_kept_allpairs} reports $\mathcal{T}_{\mathrm{cand}}$, the default kept set ($K=2$, \texttt{inv\_cv}), and the corresponding weights (normalised over kept timesteps).

\begin{table*}[t]
  \centering
  \caption{Timestep candidates, selected timesteps, and (kept-only normalised) weights for the default configuration ($K=2$, \texttt{weight\_t=inv\_cv}).}
  \label{tab:tselect_kept_allpairs}
  \scriptsize
  \setlength{\tabcolsep}{3pt}
  \begin{adjustbox}{max width=1.13\textwidth}
  \begin{tabular}{lccc|ccc|ccc}
    \toprule
    & \multicolumn{3}{c|}{CIFAR-10 (ID)} & \multicolumn{3}{c|}{SVHN (ID)} & \multicolumn{3}{c}{CelebA (ID)} \\
    \cmidrule(lr){2-4}\cmidrule(lr){5-7}\cmidrule(lr){8-10}
    & vs SVHN & vs CelebA & vs C100 & vs C10 & vs CelebA & vs C100 & vs C10 & vs SVHN & vs C100 \\
    \midrule
    $\mathcal{T}_{\mathrm{cand}}$
      & $\{5,15,136,172\}$ & $\{5,15,136,172\}$ & $\{5,15,136,172\}$
      & $\{5,15,136,172\}$ & $\{5,15,136,172\}$ & $\{5,15,136,172\}$
      & $\{5,86,172,332\}$ & $\{5,86,172,332\}$ & $\{5,86,172,332\}$ \\
    kept $\mathcal{T}$ ($K=2$)
      & $\{5,136\}$ & $\{5,136\}$ & $\{5,136\}$
      & $\{5,15\}$  & $\{5,15\}$  & $\{5,15\}$
      & $\{86,172\}$ & $\{86,172\}$ & $\{86,172\}$ \\
    % weights on kept
    %   & $(w_5,w_{136})=(0.520,0.480)$ & $(0.520,0.480)$ & $(0.522,0.478)$
    %   & $(w_5,w_{15})=(0.429,0.571)$  & $(0.429,0.571)$  & $(0.428,0.572)$
    %   & $(w_{86},w_{172})=(0.501,0.499)$ & $(0.501,0.499)$ & $(0.502,0.498)$ \\
     weights on kept
      & $(0.520,0.480)$ & $(0.520,0.480)$ & $(0.522,0.478)$
      & $(0.429,0.571)$ & $(0.429,0.571)$ & $(0.428,0.572)$
      & $(0.501,0.499)$ & $(0.501,0.499)$ & $(0.502,0.498)$ \\
    \bottomrule
  \end{tabular}
  \end{adjustbox}
\end{table*}

\paragraph{Sweep over $K$ and weighting (all 9 pairs).}
Table~\ref{tab:tselect_sweep_allpairs} reports a sweep over $K\in\{1,2,3,4\}$ and weighting choices for \emph{all} 9 ID/OOD pairs.
We include the implied NFE per input ($=(1+|\mathcal{G}|)\,K=8K$ with $|\mathcal{G}|=7$).

\begin{table*}[t]
  \centering
  \caption{Timestep selection sweep across 9 ID/OOD pairs. We report AUROC and the implied NFE per input ($8K$).}
  \label{tab:tselect_sweep_allpairs}
  \scriptsize
  \setlength{\tabcolsep}{3pt}
  \begin{tabular}{ccclccc|ccc|ccc}
    \toprule
    $K$ & $w_t$ & NFE/img
      & \multicolumn{3}{c|}{CIFAR-10 (ID)} & \multicolumn{3}{c|}{SVHN (ID)} & \multicolumn{3}{c}{CelebA (ID)} \\
    \cmidrule(lr){4-6}\cmidrule(lr){7-9}\cmidrule(lr){10-12}
    & & & vs SVHN & vs CelebA & vs C100 & vs C10 & vs CelebA & vs C100 & vs C10 & vs SVHN & vs C100 \\
    \midrule
    1 & none   & 8  & 0.871026 & 0.933444 & 0.533699 & 0.756208 & 0.999877 & 0.799473 & 0.999333 & 0.999469 & 0.999215 \\
    2 & none   & 16 & 0.835228 & 0.998620 & 0.554310 & 0.890817 & 0.999886 & 0.902899 & 0.999641 & 0.999771 & 0.999526 \\
    3 & none   & 24 & 0.785096 & 0.998993 & 0.553874 & 0.864898 & 0.999978 & 0.883712 & 0.999583 & 0.999751 & 0.999537 \\
    4 & none   & 32 & 0.757782 & 0.998891 & 0.565282 & 0.842397 & 0.999981 & 0.863905 & 0.999567 & 0.999863 & 0.999361 \\
    \midrule
    1 & inv\_cv & 8  & 0.870275 & 0.933246 & 0.538961 & 0.759835 & 0.999875 & 0.800912 & 0.999540 & 0.999835 & 0.999351 \\
    2 & inv\_cv & 16 & 0.841246 & 0.998688 & 0.557769 & 0.879337 & 0.999903 & 0.893768 & 0.999667 & 0.999792 & 0.999521 \\
    3 & inv\_cv & 24 & 0.791385 & 0.998860 & 0.556345 & 0.863403 & 0.999985 & 0.880403 & 0.999626 & 0.999749 & 0.999441 \\
    4 & inv\_cv & 32 & 0.769219 & 0.998881 & 0.566261 & 0.845128 & 0.999989 & 0.867956 & 0.999607 & 0.999863 & 0.999449 \\
    \bottomrule
  \end{tabular}
\end{table*}

\paragraph{Sweep over $K$ and weighting (representative pair).}
For direct comparison with the plots in Figs.~\ref{fig:ablation_c10_svhn}--\ref{fig:hists_svhn_c100}, Table~\ref{tab:tselect_sweep_rep} reports the same sweep for SVHN (ID) vs CIFAR-100 (OOD).

\begin{table}[t]
  \centering
  \caption{Timestep selection sweep (SVHN as ID, CIFAR-100 as OOD). We report AUROC and predicted NFE per input ($=8K$). Best is \best{bold}, second best is \second{underlined}.}
  \label{tab:tselect_sweep_rep}
  \small
  \setlength{\tabcolsep}{6pt}
  \begin{tabular}{cccc}
    \toprule
    $K$ & weighting $w_t$ & AUROC & NFE/img \\
    \midrule
    1 & none    & 0.799473 & 8  \\
    2 & none    & \best{0.902899} & 16 \\
    3 & none    & 0.883712 & 24 \\
    4 & none    & 0.863905 & 32 \\
    \midrule
    1 & inv\_cv & 0.800912 & 8  \\
    2 & inv\_cv & \second{0.893768} & 16 \\
    3 & inv\_cv & 0.880403 & 24 \\
    4 & inv\_cv & 0.867956 & 32 \\
    \bottomrule
  \end{tabular}
\end{table}

\subsection{Per-transform ablation (group elements)}
\label{app:per-g}
Let $\mathcal{G}$ denote the set of transported inputs used by GEPC.
We compute an AUROC for each $g\in\mathcal{G}$ by isolating the corresponding group-consistency gap, and compare it to the AUROC obtained by averaging over all transforms.
Figure~\ref{fig:ablation_c10_svhn} (middle) shows a representative example.

\noindent\textbf{What is varied in the per-$g$ plot.}
For interpretability, per-transform AUROCs are computed from the \emph{raw} transported-gap component (i.e.\ without KDE/z-score calibration), averaged over the retained timesteps.
The dashed horizontal line corresponds to averaging the same raw gap over all $g\in\mathcal{G}$ ("mean over $g$" in Figure~\ref{fig:ablation_c10_svhn}b).
This diagnostic checks that performance is not driven by a single transform.

\paragraph{9-pair summary table.}
Table~\ref{tab:per_g_allpairs} summarises the AUROC obtained by averaging the raw gap over $g\in\mathcal{G}$.
Since this diagnostic is \emph{unsigned} (the raw gap can be ID-high or OOD-high depending on the pair), we report $\max(\mathrm{AUROC}, 1-\mathrm{AUROC})$ as a sign-invariant separability score.

\begin{table*}[t]
  \centering
  \caption{Per-transform ablation summary (9 ID/OOD pairs). We report sign-invariant AUROC of the group-averaged raw statistic: $\max(\mathrm{AUROC},1-\mathrm{AUROC})$.}
  \label{tab:per_g_allpairs}
  \small
  \setlength{\tabcolsep}{6pt}
  \begin{tabular}{lccc|ccc|ccc}
    \toprule
    & \multicolumn{3}{c|}{CIFAR-10 (ID)} & \multicolumn{3}{c|}{SVHN (ID)} & \multicolumn{3}{c}{CelebA (ID)} \\
    \cmidrule(lr){2-4}\cmidrule(lr){5-7}\cmidrule(lr){8-10}
    Metric & vs SVHN & vs CelebA & vs C100 & vs C10 & vs CelebA & vs C100 & vs C10 & vs SVHN & vs C100 \\
    \midrule
    $\max(\mathrm{AUROC},1-\mathrm{AUROC})$ & 0.857944 & 0.999119 & 0.539358 & 0.915203 & 0.999911 & 0.923161 & 0.999707 & 0.999844 & 0.999618 \\
    \bottomrule
  \end{tabular}
\end{table*}

\subsection{Calibration variants and feature fusion}
\label{app:calib-fusion}
We compare KDE calibration (\texttt{density\_mode=kde}) against z-score normalisation and the uncalibrated score (\texttt{density\_mode=none}).
We also evaluate a Gaussian/Mahalanobis model on multi-$t$ feature vectors (\texttt{vector\_mode=mvn}).

\paragraph{Calibration variants (9 pairs).}
Table~\ref{tab:calib_allpairs} reports AUROC for calibration choices using the single feature GEPC$_s$.

\begin{table*}[t]
  \centering
  \caption{Calibration variants for GEPC$_s$ across 9 ID/OOD pairs. Values are AUROC. Best is \best{bold}, second best is \second{underlined} \emph{within each column}.}
  \label{tab:calib_allpairs}
  \small
  \setlength{\tabcolsep}{5pt}
  \begin{tabular}{lccc|ccc|ccc}
    \toprule
    & \multicolumn{3}{c|}{CIFAR-10 (ID)} & \multicolumn{3}{c|}{SVHN (ID)} & \multicolumn{3}{c}{CelebA (ID)} \\
    \cmidrule(lr){2-4}\cmidrule(lr){5-7}\cmidrule(lr){8-10}
    Calibration & vs SVHN & vs CelebA & vs C100 & vs C10 & vs CelebA & vs C100 & vs C10 & vs SVHN & vs C100 \\
    \midrule
    KDE (ID-only)      & \second{0.839844} & 0.998618 & 0.555951 & 0.878912 & \second{0.999915} & \second{0.894024} & \best{0.999644} & \best{0.999781} & \best{0.999529} \\
    z-score            & \best{0.841080}   & \best{0.998967} & \second{0.556618} & 0.853889 & \best{0.999921} & 0.873391 & \second{0.999635} & \second{0.999760} & \second{0.999528} \\
    none (raw)         & 0.136300          & \second{0.998956} & 0.537980 & \best{0.911203} & 0.999914 & \best{0.917708} & 0.000346 & 0.000225 & 0.000427 \\
    MVN (Mahalanobis)  & 0.837607          & 0.998929 & \best{0.559150} & \second{0.881477} & 0.999913 & 0.890725 & 0.999313 & 0.999514 & 0.998825 \\
    \bottomrule
  \end{tabular}
\end{table*}

\subsection{Feature variants (single-feature ablations)}
\label{app:feat-variants}
We ablate the three GEPC statistics used in the paper (Appendix~\ref{app:gepc-features} for definitions).
For compactness, Table~\ref{tab:feat_allpairs} reports the single-feature AUROC for each statistic across 9 pairs.
Figure~\ref{fig:ablation_c10_svhn} (left) visualises a representative case.

\begin{table*}[t]
  \centering
  \caption{Single-feature ablations across 9 ID/OOD pairs (three GEPC statistics). Values are AUROC under KDE calibration. Best is \best{bold}, second best is \second{underlined} \emph{within each column}.}
  \label{tab:feat_allpairs}
  \small
  \setlength{\tabcolsep}{5pt}
  \begin{tabular}{lccc|ccc|ccc}
    \toprule
    & \multicolumn{3}{c|}{CIFAR-10 (ID)} & \multicolumn{3}{c|}{SVHN (ID)} & \multicolumn{3}{c}{CelebA (ID)} \\
    \cmidrule(lr){2-4}\cmidrule(lr){5-7}\cmidrule(lr){8-10}
    Feature & vs SVHN & vs CelebA & vs C100 & vs C10 & vs CelebA & vs C100 & vs C10 & vs SVHN & vs C100 \\
    \midrule
    GEPC$_s$              & \best{0.838699} & \second{0.998862} & \second{0.556210} & \best{0.879254} & \best{0.999888} & \best{0.896374} & \second{0.999610} & 0.999764 & \best{0.999492} \\
    GEPC$_{\cos}$         & 0.584358        & \best{0.999177}   & 0.546293        & 0.873489 & \second{0.999852} & \second{0.896193} & 0.999596 & \best{0.999820} & \second{0.999476} \\
    GEPC$_{\mathrm{pair}}$& 0.819882 & 0.996907         & 0.549942 & 0.861242        & 0.999298        & 0.877116        & 0.998608        & 0.998695        & 0.998003 \\
    Fusion (mean) & \second{0.831308} & 0.998614 & \best{0.557173} & \second{0.876319} & 0.999842 & 0.893592 & \best{0.999617} & \second{0.999787} & 0.999416 \\
    \bottomrule
  \end{tabular}
\end{table*}

\subsection{Runtime and NFEs}
\label{app:runtime}
For each timestep $t$, GEPC uses one reference evaluation $\mathbf{s}_\theta(\mathbf{x}_t,t)$ and one batched evaluation over transported inputs
$\{\mathcal{P}_g \mathbf{x}_t\}_{g\in\mathcal{G}}$, hence $(1+|\mathcal{G}|)$ forward evaluations and $0$ JVPs per timestep.
With $m$ Monte-Carlo noise samples and $K=|\mathcal{T}|$ retained timesteps, total cost is $(1+|\mathcal{G}|)\,K\,m$ forward passes.
This computation is parallelisable over $g$ and (when memory allows) over $t$.

\begin{table}[t]
  \centering
  \caption{Measured runtime for a representative $32\times 32$ pair (SVHN as ID, CIFAR-100 as OOD) on a single GPU, alongside implied NFE.
  Timing is reported as milliseconds per image (lower is better).
  Hardware: NVIDIA GeForce RTX 4060 Laptop GPU (Linux, PyTorch).}
  \label{tab:runtime_rep}
  \small
  \setlength{\tabcolsep}{6pt}
  \begin{tabular}{lccc}
    \toprule
    Variant & AUROC & ms/img (ID) & ms/img (OOD) \\
    \midrule
    GEPC$_s$ + KDE     & 0.894024 & 69.92 & 69.96 \\
    GEPC$_s$ + z-score & 0.873391 & 70.02 & 69.72 \\
    GEPC$_s$ raw       & 0.917708 & 69.68 & 69.66 \\
    GEPC$_s$ + MVN     & 0.890725 & 69.86 & 69.69 \\
    \bottomrule
  \end{tabular}
\end{table}

\subsection{Representative plots and score distributions}
\label{app:rep-plots}
We provide representative plots for one pair (CIFAR-10 as ID, SVHN as OOD).
Figure~\ref{fig:ablation_c10_svhn} shows feature variants, per-transform AUROC, and single-timestep AUROC vs.\ $t$.
Figure~\ref{fig:hists_svhn_c100} shows the separation of score distributions for baseline energy, transported energy gap, and the final GEPC$_s$ score.

\begin{figure*}[t]
  \centering
  \begin{subfigure}[t]{0.32\linewidth}
    \centering
    \includegraphics[width=\linewidth]{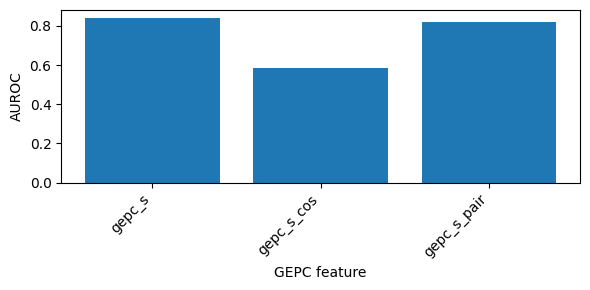}
    \caption{Feature variants (single-feature AUROC).}
  \end{subfigure}
  \hfill
  \begin{subfigure}[t]{0.32\linewidth}
    \centering
    \includegraphics[width=\linewidth]{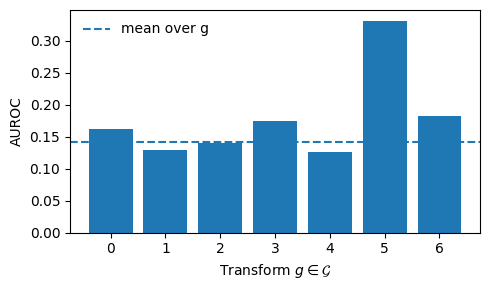}
    \caption{Per-transform AUROC (raw gap component).}
  \end{subfigure}
  \hfill
  \begin{subfigure}[t]{0.32\linewidth}
    \centering
    \includegraphics[width=\linewidth]{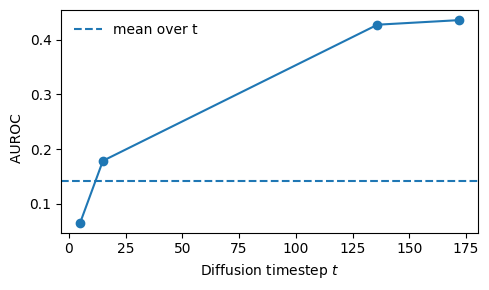}
    \caption{Single-timestep AUROC vs.\ $t$ (raw gap component).}
  \end{subfigure}
  \caption{
  Representative ablations for GEPC (CIFAR10 as ID, SVHN as OOD).
  (a) Single-feature variants under the same ID-only protocol.
  (b) Per-transform AUROC computed from the raw transported-gap component (no calibration); the dashed line averages the same component over $g\in\mathcal{G}$.
  (c) Single-timestep AUROC computed from the raw transported-gap component; the dashed line averages the same component over the retained timesteps.}
  \label{fig:ablation_c10_svhn}
\end{figure*}

\begin{figure*}[t]
  \centering
  \begin{subfigure}[t]{0.32\linewidth}
    \centering
    \includegraphics[width=\linewidth]{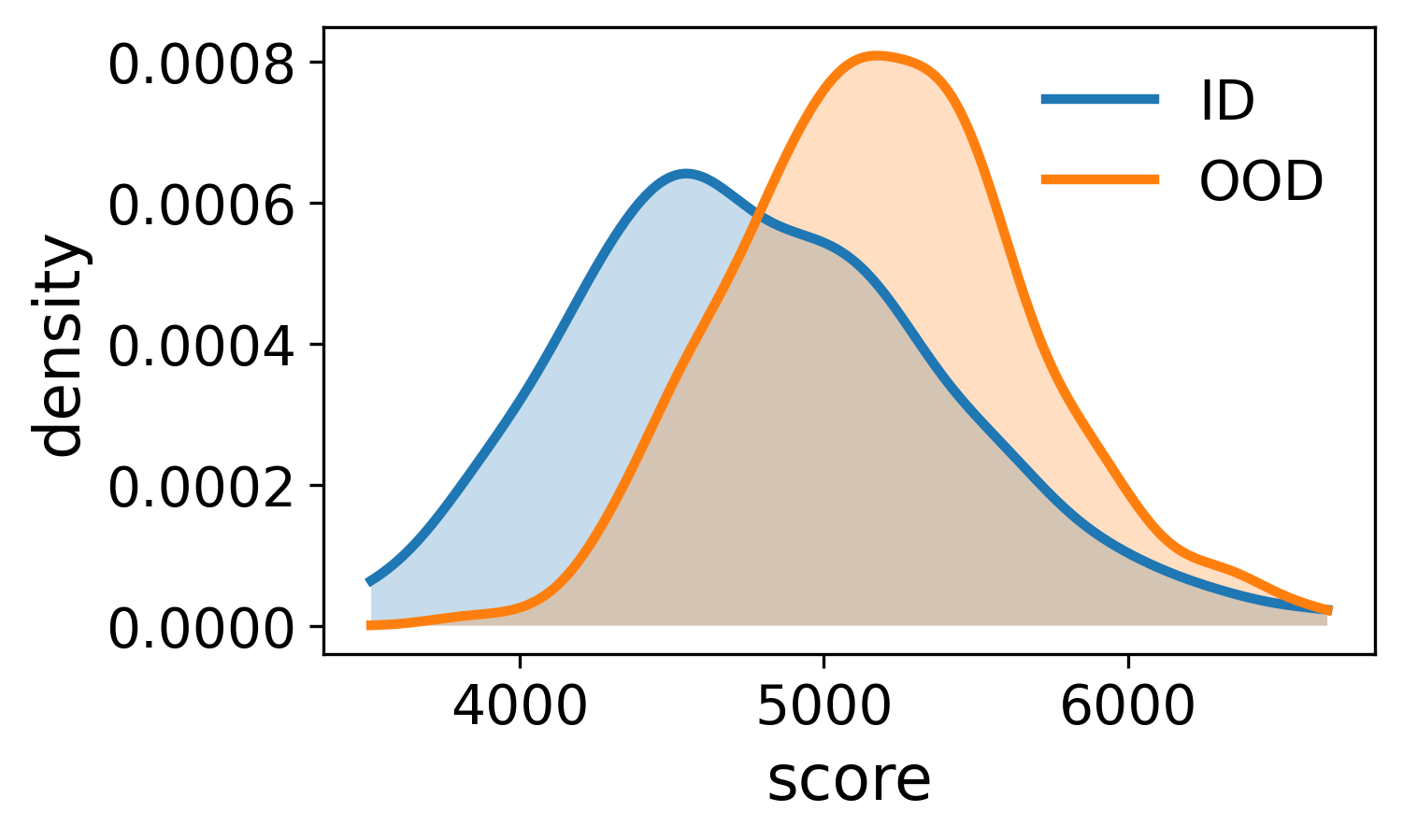}
    \caption{\textbf{Score magnitude (non-GEPC).}
    $E_t(\mathbf{x}_t)\coloneqq \|\mathbf{s}_\theta(\mathbf{x}_t,t)\|_2^2$,
    with $\mathbf{x}_t\sim q(\cdot\mid \mathbf{x}_0)$.}
    \label{fig:hists_svhn_c100_base}
  \end{subfigure}
  \hfill
  \begin{subfigure}[t]{0.32\linewidth}
    \centering
    \includegraphics[width=\linewidth]{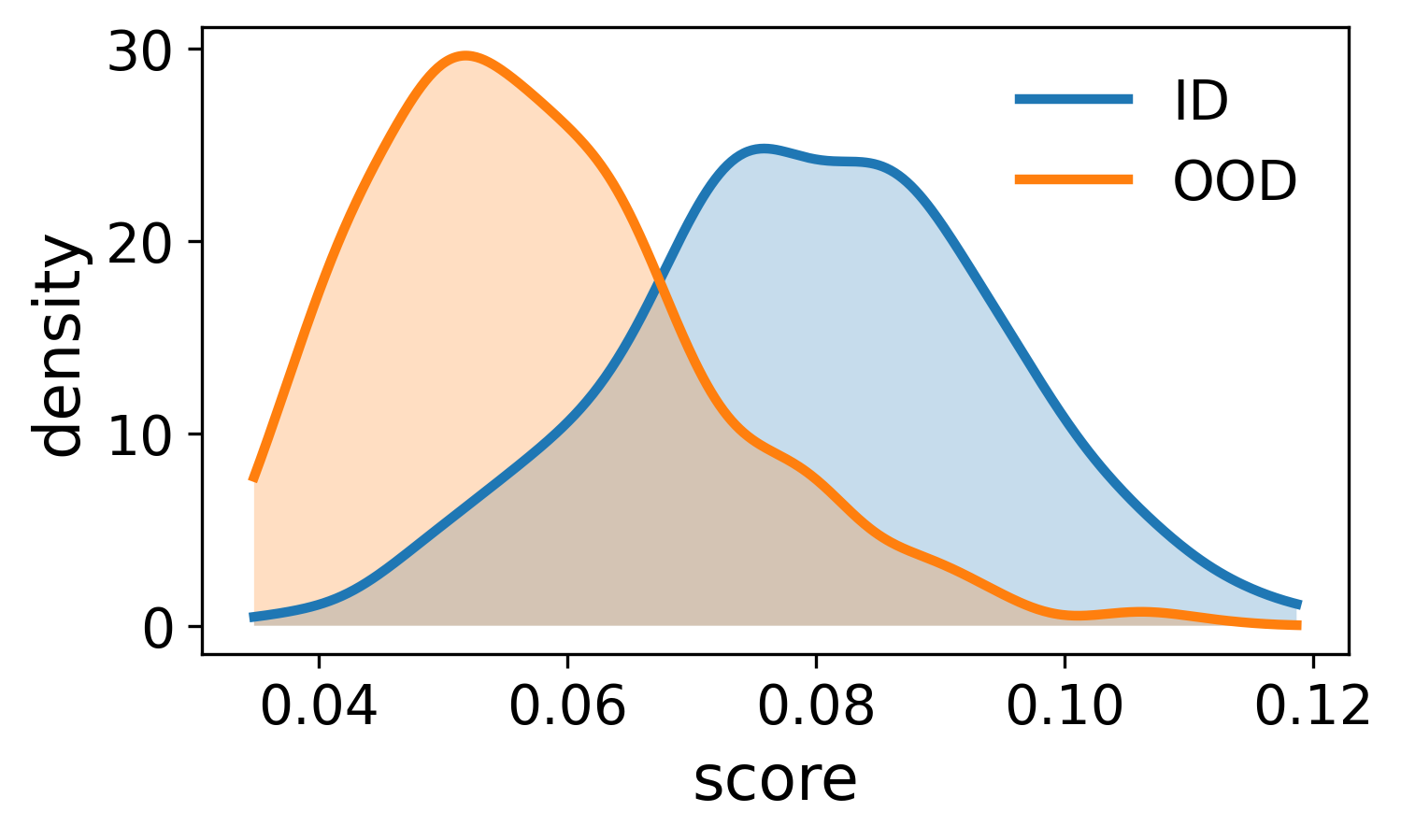}
    \caption{\textbf{Equivariance residual energy (single-step).}
    $R_t(\mathbf{x}_t,g)\coloneqq \|\Delta_g \mathbf{s}_\theta(\mathbf{x}_t,t)\|_2^2$,
    where $\Delta_g f(\mathbf{x},t)\coloneqq \mathcal{P}_g^{-1}f(\mathcal{P}_g\mathbf{x},t)-f(\mathbf{x},t)$.}
    \label{fig:hists_svhn_c100_rt}
  \end{subfigure}
  \hfill
  \begin{subfigure}[t]{0.32\linewidth}
    \centering
    \includegraphics[width=\linewidth]{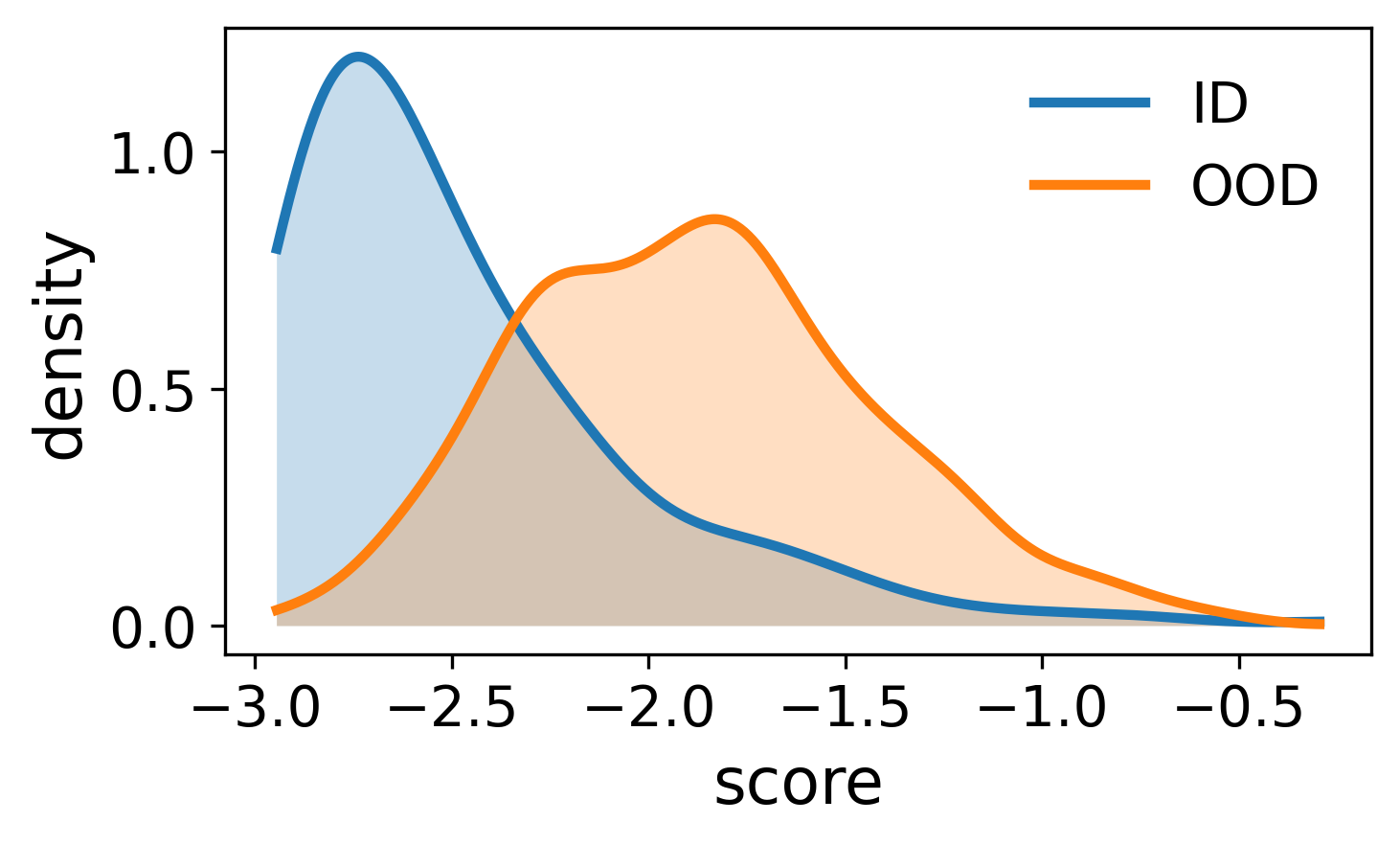}
    \caption{\textbf{Final GEPC score (time-averaged).}
    $\mathrm{GEPC}(\mathbf{x}_0)\coloneqq \sum_{t\in\mathcal T} w_t\,
    \mathbb{E}_{\mathbf{x}_t\sim q(\cdot\mid \mathbf{x}_0),\;g\sim \nu_{\mathcal G}}
    \!\left[R_t(\mathbf{x}_t,g)\right]$.}
    \label{fig:hists_svhn_c100_gepc}
  \end{subfigure}

  \caption{\textbf{Score distributions (ID vs OOD) for a representative pair (SVHN as ID, CIFAR-100 as OOD).}
  Left: score magnitude $E_t(\mathbf{x}_t)$ (a baseline diagnostic, not GEPC).
  Middle: single-step equivariance residual energy $R_t(\mathbf{x}_t,g)$.
  Right: time-averaged GEPC score $\mathrm{GEPC}(\mathbf{x}_0)$ aggregating $R_t$ over $t\in\mathcal T$ with weights $w_t$ and uniform $g\sim\nu_{\mathcal G}$.}
  \label{fig:hists_svhn_c100}
\end{figure*}

%%%%%%%%%%%%%%%%%%%%%%%%%%%%%%%%%%%%%%%%%%%%%%%%%%%%%%%%%%%%
\section{Radar SAR details}
\label{app:radar-details}
%%%%%%%%%%%%%%%%%%%%%%%%%%%%%%%%%%%%%%%%%%%%%%%%%%%%%%%%%%%%

\paragraph{SAR background (context).}
Synthetic Aperture Radar (SAR) is an active microwave imaging modality producing high-resolution reflectivity maps under all-weather and day/night conditions.
SAR images are coherent and typically exhibit speckle and strong intensity dynamics; we therefore visualise and process patches in log-magnitude.

\paragraph{Datasets and OOD task.}
We use HRSID and SSDD, two public SAR datasets commonly used for ship detection.
We form an OOD task where \emph{sea-clutter-only} patches are in-distribution (ID) and patches containing at least one annotated ship (and wake when visible in the patch) are out-of-distribution (OOD).

% \paragraph{Quantitative results.}
% Table~\ref{tab:sar_ood_metrics} reports patch-level AUROC, AUPR (OOD as positive), and FPR@95\%TPR for ID sea-clutter patches from HRSID against target-containing patches from HRSID and SSDD.

\paragraph{Quantitative results.}
Patch-level OOD detection metrics are reported in Table~\ref{tab:sar_ood_metrics} for ID sea-clutter patches from HRSID against target-containing patches from HRSID (intra-dataset) and SSDD (cross-dataset).

\paragraph{Preprocessing and patching.}
For each SAR patch, we convert intensities to log-magnitude, apply per-patch normalisation, and resize/crop to $256\times256$ to match the LSUN-$256$ diffusion backbone input.
If the backbone expects 3 channels, we replicate the single-channel SAR patch across channels.
No SAR-specific fine-tuning is performed.

\paragraph{Equivariance residual maps and normalisation.}
Beyond the scalar GEPC score, we visualise the \emph{pre-pooling} equivariance residual magnitude map $|\Delta(\mathbf{x})|$, highlighting spatial regions where equivariance breaks (typically ships/wakes) while remaining low on homogeneous sea clutter.
For magnitude comparison across examples and datasets, we export globally normalised maps using a fixed
$
v_{\mathrm{global}} = \mathrm{median}_{\mathbf{x}\in\mathcal{P}_{\mathrm{ID}}}\, q_{0.99}(|\Delta(\mathbf{x})|),
$
computed over an ID candidate pool $\mathcal{P}_{\mathrm{ID}}$.
We also export per-image normalised maps and raw residual maps for inspection (see exported files and metadata).

\begin{figure*}[t]
  \centering

  \textbf{ID (HRSID sea clutter)}\\[-1mm]
  \begin{subfigure}[b]{0.29\linewidth}
    \centering
    \includegraphics[width=\linewidth]{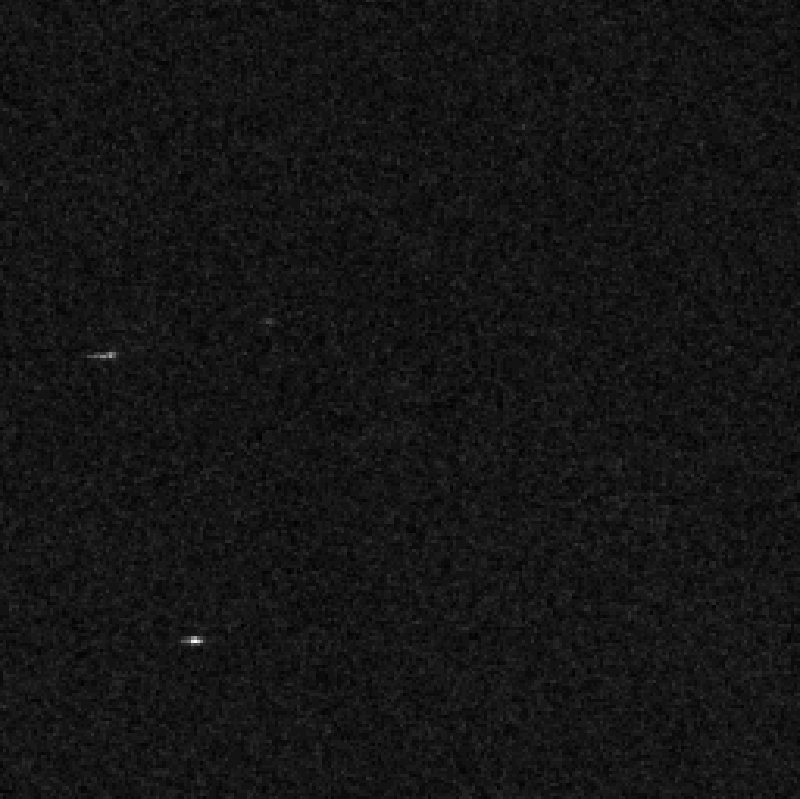}
    \caption{Log-magnitude}
  \end{subfigure}
  \hfill
  \begin{subfigure}[b]{0.34\linewidth}
    \centering
    \includegraphics[width=\linewidth]{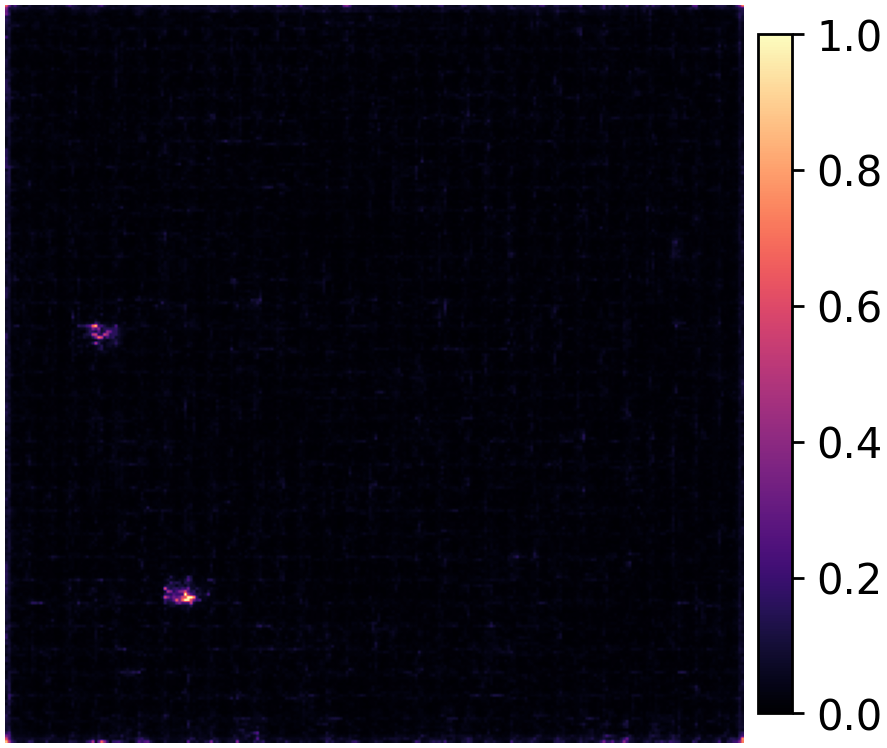}
    \caption{GEPC residual}
  \end{subfigure}
  \hfill
  \begin{subfigure}[b]{0.34\linewidth}
    \centering
    \includegraphics[width=\linewidth]{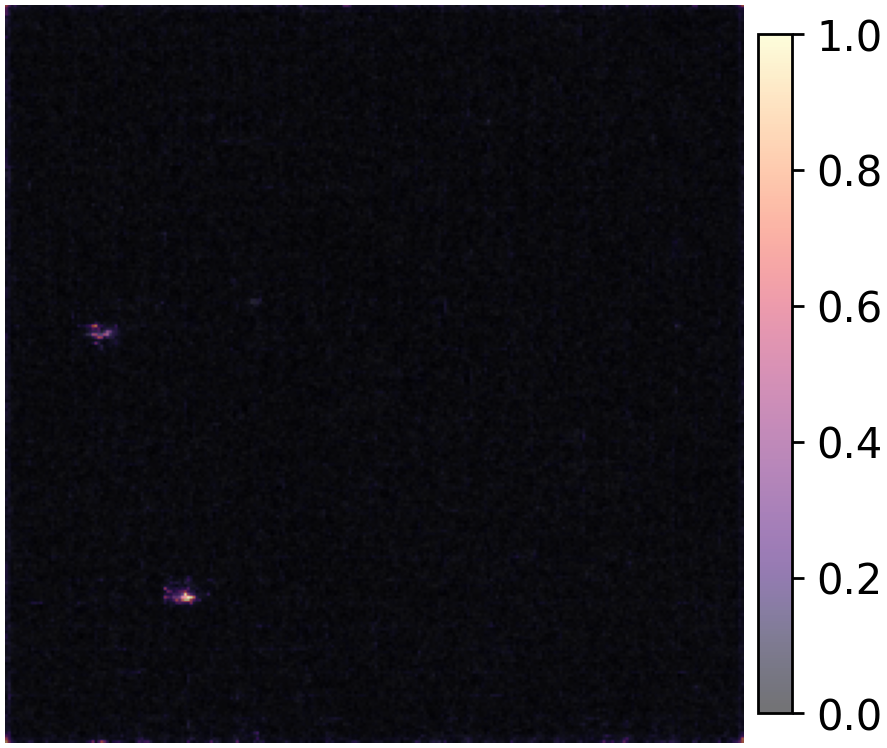}
    \caption{Overlay}
  \end{subfigure}

  \vspace{2mm}
  \textbf{OOD (HRSID targets)}\\[-1mm]
  \begin{subfigure}[b]{0.29\linewidth}
    \centering
    \includegraphics[width=\linewidth]{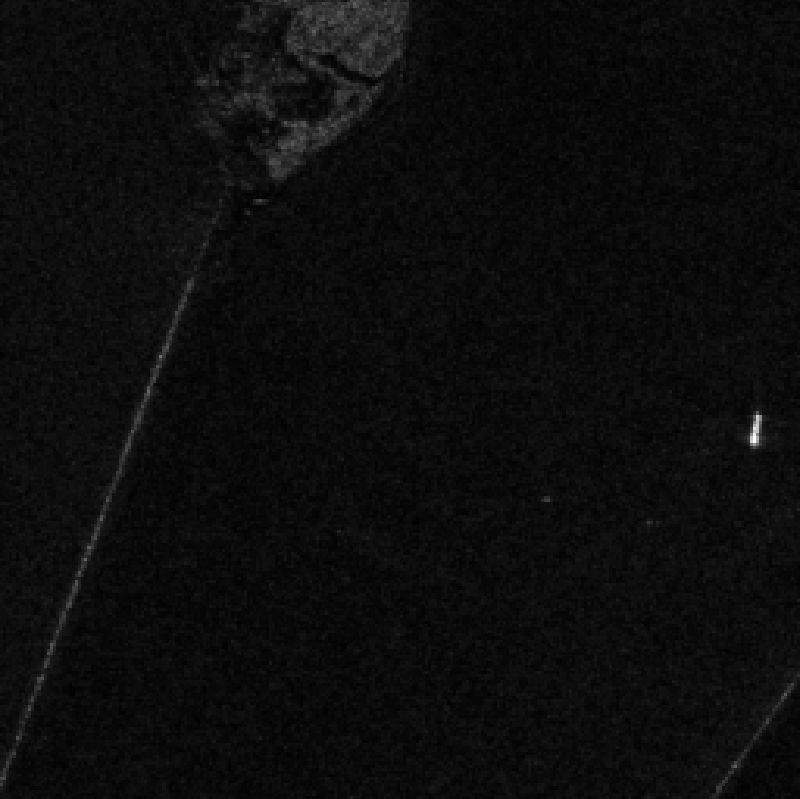}
    \caption{Log-magnitude}
  \end{subfigure}
  \hfill
  \begin{subfigure}[b]{0.34\linewidth}
    \centering
    \includegraphics[width=\linewidth]{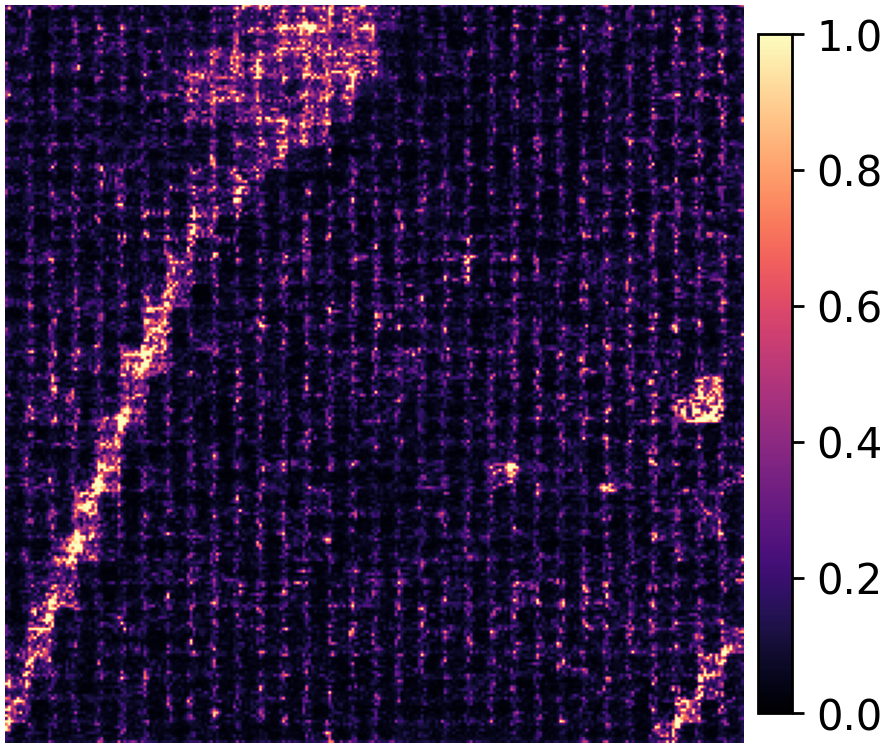}
    \caption{GEPC residual}
  \end{subfigure}
  \hfill
  \begin{subfigure}[b]{0.34\linewidth}
    \centering
    \includegraphics[width=\linewidth]{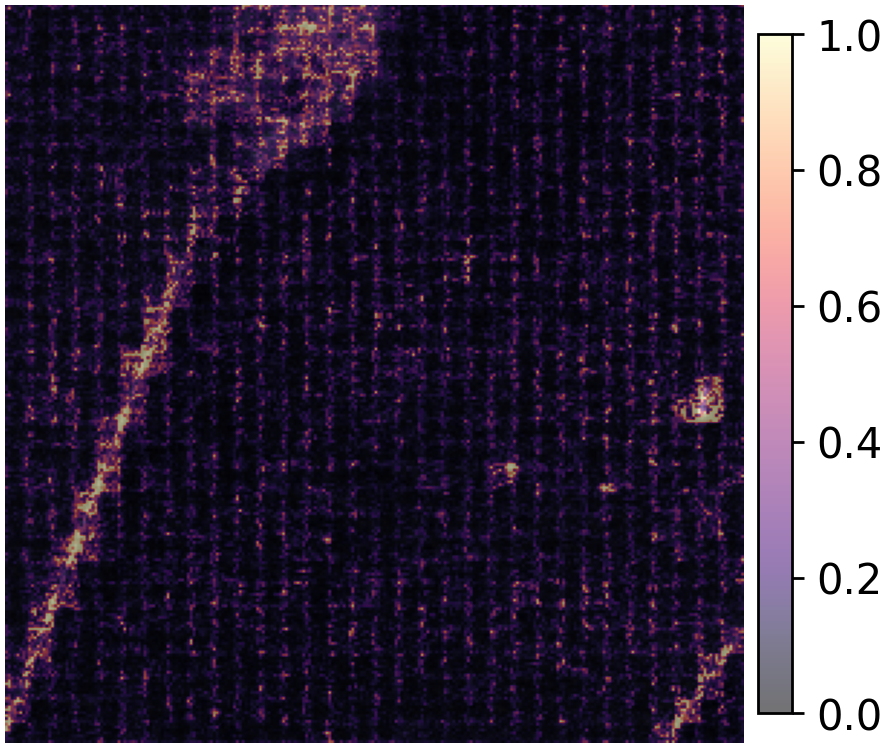}
    \caption{Overlay}
  \end{subfigure}

  \vspace{2mm}
  \textbf{OOD (SSDD targets)}\\[-1mm]
  \begin{subfigure}[b]{0.29\linewidth}
    \centering
    \includegraphics[width=\linewidth]{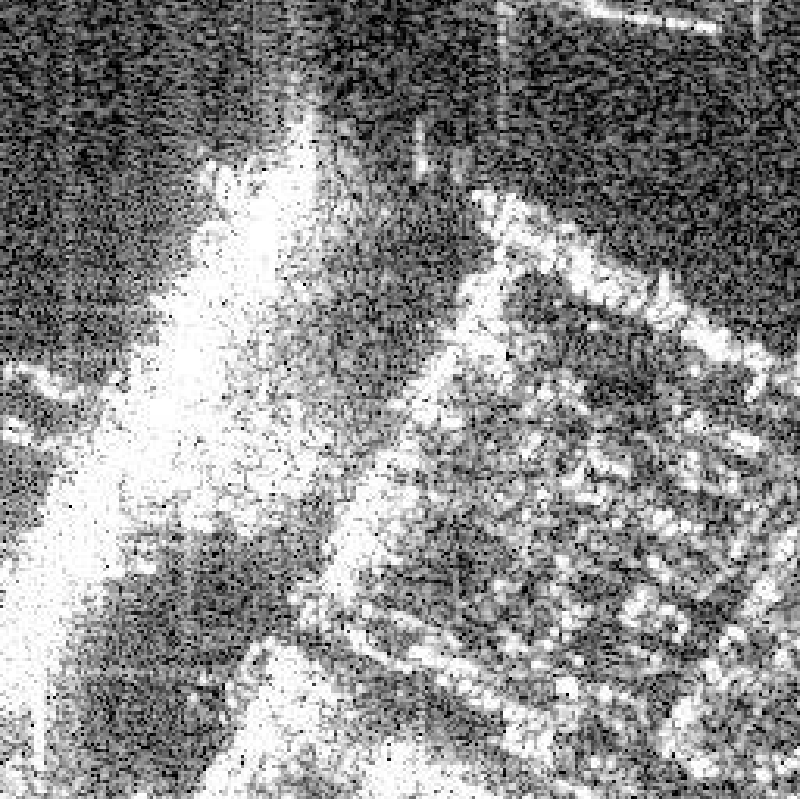}
    \caption{Log-magnitude}
  \end{subfigure}
  \hfill
  \begin{subfigure}[b]{0.34\linewidth}
    \centering
    \includegraphics[width=\linewidth]{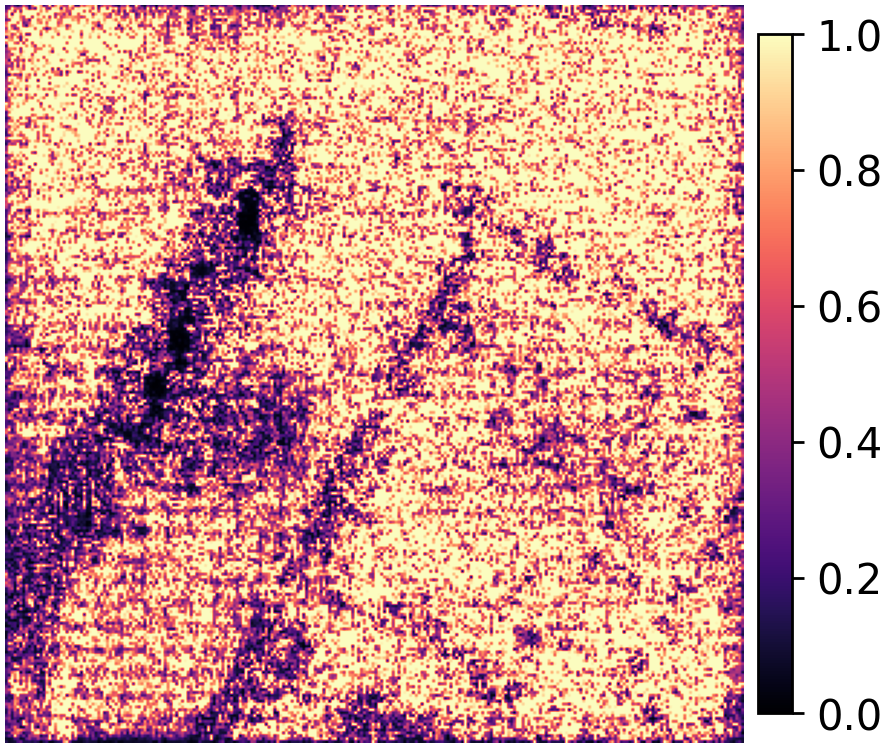}
    \caption{GEPC residual}
  \end{subfigure}
  \hfill
  \begin{subfigure}[b]{0.34\linewidth}
    \centering
    \includegraphics[width=\linewidth]{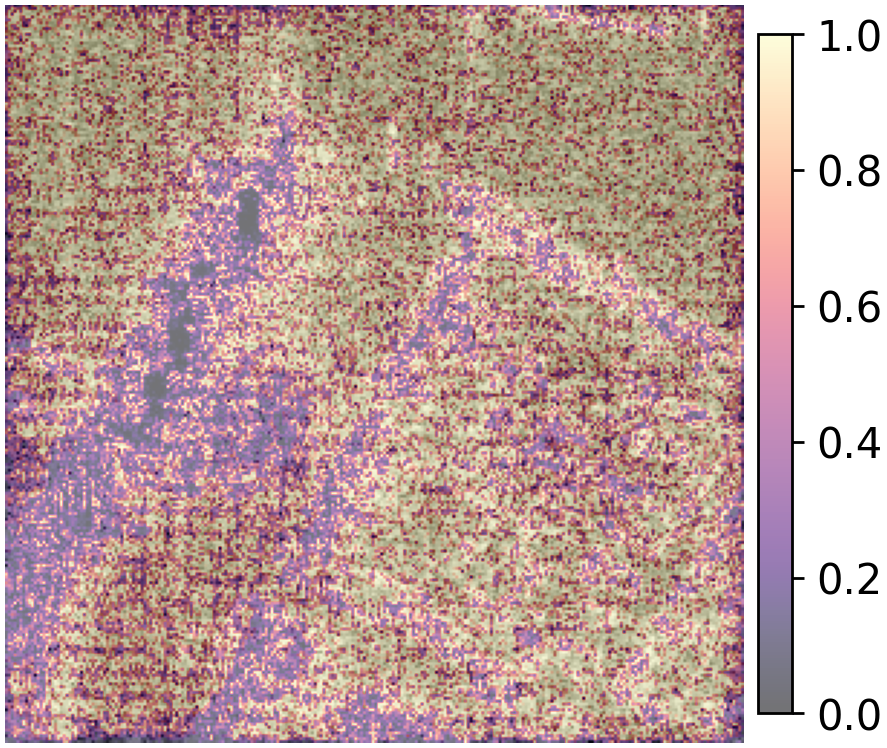}
    \caption{Overlay}
  \end{subfigure}

  \caption{
  Qualitative GEPC localisation on SAR patches (LSUN-$256$, no SAR fine-tuning).
  Residual maps are globally normalised by a shared $v_{\mathrm{global}}$ (computed on an ID pool) to enable comparison across ID/OOD and across datasets.
  }
  \label{fig:sar_qual_appendix}
\end{figure*}

\begin{table}[t]
  \centering
  \caption{Patch-level OOD detection on SAR. ID is sea-clutter patches from HRSID; OOD are target-containing patches from HRSID and SSDD. Higher AUROC/AUPR is better; lower FPR@95\%TPR is better.}
  \label{tab:sar_ood_metrics}
  \small
  \setlength{\tabcolsep}{6pt}
  \begin{tabular}{lccc}
    \toprule
    OOD split (targets) & AUROC $\uparrow$ & FPR@95 $\downarrow$ & AUPR $\uparrow$ \\
    \midrule
    HRSID-ship/wake & 0.853 & 0.000 & 0.619 \\
    SSDD-ship       & 1.000 & 0.000 & 1.000 \\
    \bottomrule
  \end{tabular}
\end{table}

\end{document}